\documentclass[11pt]{article}

\usepackage{fullpage}
\usepackage{multirow}
\usepackage{microtype}
\usepackage[square]{natbib}
\usepackage{hyperref}
\hypersetup{hidelinks}
\hypersetup{
colorlinks=true,
linkcolor=blue,
citecolor=blue
}
\usepackage{graphicx}
\usepackage{subfigure}
\usepackage{booktabs} 
\usepackage{amssymb}
\usepackage{algorithmic}
\usepackage{algorithm}
\usepackage{bbding}
\usepackage{makecell}
\usepackage[normalem]{ulem} 
\usepackage{hyperref}

\usepackage{amsmath}
\usepackage{amssymb}
\usepackage{mathtools}
\usepackage{amsthm}
\allowdisplaybreaks

\usepackage[capitalize,noabbrev]{cleveref}

\theoremstyle{plain}
\newtheorem{theorem}{Theorem}[section]

\newtheorem{lemma}[theorem]{Lemma}

\theoremstyle{definition}

\newtheorem{assumption}[theorem]{Assumption}
\theoremstyle{remark}

\usepackage[textsize=tiny]{todonotes}

\newcommand{\DD}{\mathbb{D}}
\newcommand{\EE}{\mathbb{E}}

\newcommand{\II}{\mathbb{I}}

\newcommand{\LL}{\mathbb{L}}

\newcommand{\NN}{\mathbb{N}}

\newcommand{\PP}{\mathbb{P}}

\newcommand{\RR}{\mathbb{R}}

\newcommand{\cD}{\mathcal{D}}

\newcommand{\cF}{\mathcal{F}}

\newcommand{\cO}{\mathcal{O}}

\newcommand{\cR}{\mathcal{R}}

\newcommand{\cX}{\mathcal{X}}
\newcommand{\cY}{\mathcal{Y}}

\ifx\example\undefined
\newtheorem{example}[theorem]{Example}
\fi

\DeclareMathOperator*{\argmax}{arg\,max}
\DeclareMathOperator*{\argmin}{arg\,min}

\title{Projection Optimization: A General Framework for Multi-Objective and Multi-Group RLHF}
\author{Nuoya Xiong\thanks{Carnegie Mellon University. Email: \texttt{nuoyax@andrew.cmu.edu}.} \qquad Aarti Singh\thanks{Carnegie Mellon University.  Email: \texttt{aarti@cs.cmu.edu}}}
\date{\today}

\begin{document}


\maketitle

\begin{abstract}
Reinforcement Learning with Human Feedback (RLHF) is a widely used fine-tuning approach that aligns machine learning model, particularly Language Model (LM) with human preferences. There are typically multiple objectives driving the preference, hence humans find it easier to express per-objective comparisons rather than a global preference between two choices. 
Multi-Objective RLHF (MORLHF) aims to use per-objective preference feedback and achieve Pareto optimality among these objectives by aggregating them into a single unified objective for optimization. However, nearly all prior works rely on linear aggregation, which rules out policies that favor specific objectives such as the worst one. The only existing approach using non-linear aggregation  is computationally expensive due to its reward-based nature and the need for retraining whenever the aggregation parameters change.
In this work, we address this limitation by transforming the non-linear aggregation maximization problem into a series of sub-problems. Each sub-problem involves only linear aggregation, making it computationally efficient to solve. We further extend our framework to handle multi-group scenarios, where each group has distinct weights for the objectives. Our method enables achieving consensus or maximizing the aggregated objective across all groups.
Theoretically, we demonstrate that our algorithmic framework achieves sublinear regret and can be easily adapted to a reward-free algorithm. Empirically, leveraging our theoretical insights, we propose a nearly training-free algorithm once the optimal policies for individual objectives are obtained.
\end{abstract}

\section{Introduction}

In recent years, there has been considerable effort to fine-tune a machine learning model, particularly Large Language Model (LLM), to perform better on particular tasks. RLHF is a popular fine-tuning approach, which receives the human's preference feedback and aligns the LLM model with human values using fine-tuning.
Standard RLHF exploits human preference feedback between two outputs to maximize the expectation of the implicit or explicit reward function.

However, there are two main challenges for the application of RLHF in the real world. First, standard RLHF only maximizes a single reward function. 
However, people often find it hard to evaluate choices in an overall sense as, in reality, there are often \textit{multiple objectives}. For example, comparing two papers or essays overall is harder than comparing them on specific objectives such as novelty, clarity, correctness etc. Similarly, recommending a city for vacation is harder than comparing cities on food options, nightlife, safety, etc. Each objective has its own implicit or explicit reward function, and the LLM needs to achieve a Pareto optimal trade-off between them by, for example, maximizing an aggregation of these reward function. Second, there are \textit{multiple groups} of users in the real world who may prefer different aggregations of the objectives. For example, groups with different genders, political views, marital status, etc. 
This requires that the LLM either (a) satisfies the requirements of all the groups simultaneously, or (b) optimizes some aggregation across multiple groups. 

\vspace{-0.1in}
\paragraph{Multi-Objective Problem} 
There are some works \citep{rame2024rewarded, yang2024rewards,shi2024decoding} that consider balancing the utilities of multiple objectives to get the Pareto optimal point or maximize the average expectation. Some works \citep{zhong2024provable,park2024rlhf}  consider multi-party problem in which each reward represents a group, which can also be regarded as a multi-objective problem.
We assume that we have $m$ different objectives, and each objective has its own reward function $r_i(x,y)(1\le i\le m)$. Each reward corresponds to an objective of the response $y$ like safety or helpfulness of the LLM.
Nearly all of the previous work consider only linear aggregation, i.e., optimizing $  r(x,y)=\sum_{i=1}^m \alpha_i r_i(x,y),$ where $\alpha=\{\alpha_i\}_{i \in[m]}$ is the weight of all objectives that is assumed to be known. 

However, this kind of aggregation may not lead to an LLM that treats all objectives fairly. For example, the LLM may favor one objective significantly at the expense of another.
In social choice theory, certain natural axioms such as monotonicity, symmetry,
scale invariance, etc. which apply to multi-objective aggregation as well, lead to 
a more general function class \citep{cousins2021axiomatic}
\begin{align}r(x,y) =\left(\sum_{i=1}^m \alpha_i r_i^p(x,y)\right)^{1/p}, p\le 1,\label{eq:generalized_reward_f}\end{align}
The general $p$-norm aggregation with $p\le1$ promotes fairness across multiple objectives, which is particularly useful when aiming for a machine model that achieves a balanced performance among different objectives.
Only one paper \citep{zhong2024provable} addresses the p-norm aggregation setting. In that work, the authors first learn a reward function for each objective, aggregate them into a new reward, and then attempt to optimize this new reward directly. However, this reward-based approach is computationally inefficient compared to the reward-free, DPO-based algorithm \citep{rafailov2024direct}. Moreover, it requires retraining the entire policy whenever the aggregation method changes, which becomes even more time-consuming. 

To reduce the computational cost of the reward-based RLHF algorithm, the paper \citep{shi2024decoding} shows that for $p=1$, once the optimal policy $\pi_{r_i}$
  for each individual objective is obtained, the optimal policy 
$\pi_r$ for the linear averaged sum can be calculated as $\pi_r(y\mid x) \propto \prod_{i=1}^m \pi_{r_i}(y\mid x)^{\alpha_i}.$
However, the derivation heavily depends on the linear structure of the aggregated reward $r(x,y)$. When $p\neq 1,$ this approach breaks and the optimal policy cannot be written as a simple closed-form of the optimal policies of each objective. 
That raises the first question: 
\vspace{0.5em}

\centerline{\textbf{\textit{Question 1: Can we derive a computationally efficient MORLHF algorithm}}}\centerline{\textbf{\textit{with non-linear aggregation?}}} 

\vspace{0.5em}
In our work, we propose a projection-based algorithm both in offline and online preference data settings, which transforms the nonlinear objective maximization problem into a sequence of subproblems, each involving only a linear maximization problem. Theoretically, we provide a thorough analysis for both offline and online setting, showing that it can converge to the optimal policy with a sublinear regret. Empirically, by leveraging the fact that there is a training-free algorithm for linear aggregation maximization, we 
derive a training-free algorithm for the generalized reward aggregation, which saves significant training time. 

Moreover,
previous work typically assumes that the weight for each objective is known.
This assumption simplifies the problem and allows for straightforward optimization. However, in real-world applications,  the importance weights $\{\alpha_i\}$ for each objective are usually unknown. 
In our work, we observe that the weight of an objective reflects its importance, which can be learned by how frequently the objective is reported in the human preferences. We propose a learning paradigm where the LLM learns objective weights from collected data, enabling the estimation of $\{\alpha_i\}$ and incorporating them into our theoretical results.

\paragraph{Multi-Group Problem}  Classical RLHF often assumes a single-group setting, ignoring the heterogeneity in human feedback and assuming that the human feedback relies on one unique reward function. However, real-world scenarios involve multiple groups with distinct preferences. Fine-tuning an LLM for each group is computationally expensive, making it essential to fine-tune the LLM to accommodate all groups' preferences simultaneously.

Since previous papers \citep{zhong2024provable,park2024rlhf} working on multi-group RLHF only consider learning the reward function of each group under a single objective and then aggregating them, we regard them as a special case of the MORLHF. 
Hence, there is a lack of discussion about the multi-group setting where each group may have different importance for different objectives.

Formally, assume that we have $N$ group and $m$ objectives, and each group $n \in [N]$ has their own weight $\alpha^{(n)}\in \Delta_{m-1}$. The reward of the group $n$ is then defined by 
\begin{align*}
    r^{(n)}(x,y) = \left(\sum_{i=1}^m \alpha_i^{(n)}(r_i(x,y))^{p^{(n)}}\right)^{1/p^{(n)}}, p^{(n)}\le 1.
\end{align*}
The reward function of each objective, $\{r_i(x, y)\}_{i \in [m]}$, remains fixed across different groups, while the weight $\alpha$ and the parameter $p$ can vary.
In other words, the reward of each objective is the inherent value, and the importance weight represents the subjective part of each group. Now we pose the last question: 
\vspace{0.5em}

\centerline{\textbf{\textit{Question 2: Can we formulate and tackle the multi-group }}}\centerline{\textbf{\textit{problem under MORLHF setting?}}}

\vspace{0.5em}
In this paper, we consider two final goals for multi-group problem. Motivated by the poll theory, the first objective is called ``consensus", in which LLM needs to meet the requirements of all groups as good as possible simultaneously. Motivated by social choice theory, the second objective is called "aggregation", in which the LLM needs to optimize a general aggregation of the utilities of all groups. 
We will show that our formulation and algorithmic framework naturally solve these two final goals.
In summary, we have the following contributions:
\begin{itemize}

\item We reformulate the reward maximization in MORLHF as minimizing the distance between the current reward vector and a target set. This reframing decomposes the aggregated reward maximization into sub-problems, each focusing on minimizing the distance in a specific direction. These sub-problems reduce to linear aggregation and can be efficiently solved using previous approaches.
Theoretically, we provide converge guarantees for both offline and online setting. Empirically, we provide a training-free algorithm once the optimal policy and the reward function for each objective is given, making it more computationally efficient.

\item  We tackle the multi-group problem in two ways: (1) achieving consensus by defining the target set as the intersection of all groups' target sets, and (2) minimizing the malfare function \citep{cousins2021axiomatic} which aggregates the distance between each group's expected reward vector and its target set. Our framework addresses both problems concisely with theoretical guarantees.

\item  We establish a learning paradigm where the LMs learn the importance weight from data. We integrate weight estimation into the online setting and provide theoretical guarantees.

\end{itemize}
\section{Related Works}

\paragraph{RLHF}
Fine-tuning LLMs with human feedback and RL is known as RLHF. The reward-based RLHF first extracts a reward model with a Bradley-Terry (BT) assumption on human preferences, and then optimizes the reward model \citep{ouyang2022training,bai2022training,touvron2023llama,azar2024general}. On the other hand, the reward-free RLHF avoids explicit reward modeling by directly formulating the preference loss as a function of the policy and then using supervised learning \citep{wang2023beyond,rafailov2024direct}, which is more stable and computation-friendly. 
\paragraph{MORLHF}
Multi-Objective RLHF (MORLHF) aims to align an LLM with human preferences while optimizing for multiple objectives, such as harmlessness, helpfulness, and humor. Most previous works aggregate rewards or models as the weighted sum of individual components. MORLHF \citep{wu2023fine,bai2022training} directly optimizes the aggregated reward using PPO, while MODPO \citep{zhou2023beyond} provides a lightweight reward-free alternative. RS \citep{rame2024rewarded} combines individual models by averaging them. MOD \citep{shi2024decoding} calculates the closed-form solution of the optimal policy for aggregated reward directly and derives a training-free algorithm. Only one work \citep{zhong2024provable} consider non-linear aggregation, and they optimize the aggregated reward function directly. However, this approach is computationally expensive and requires retraining when the aggregation changes. Instead, we propose a theoretical framework that can be easily adapted to a reward-free algorithm,  along with a training-free empirical algorithm built on the same theoretical framework. Detailed comparisons are shown in Table~\ref{table:comparison}.

\paragraph{Pluralistic Alignment and Preference Aggregation}
There is a growing body of work on aligning machine learning models with diverse preferences, accounting for different values and perspectives.  The works \citep{chakrabortymaxmin,ramesh2024group} focus on optimizing the worst-case group loss, ensuring that the model achieves reasonable performance across all groups. \citep{park2024rlhf, sorensenposition, conitzer2024social} explore how to aggregate preferences using social choice and voting theory, outlining a high-level roadmap for pluralistic AI alignment. \citep{ge2024axioms} technically demonstrate that the BTL model fails to satisfy well-known standards in social choice theory and propose a novel rule-based approach for learning reward functions. \citep{chen2024pal} further study the generalization of the BTL model and introduce an ideal point model that better accommodates diverse groups.

\begin{table}\centering\footnotesize

 \caption{Comparison of previous work for MORLHF. The parameter $p$ means the exponent in Eq.\eqref{eq:generalized_reward_f}. Algorithm \ref{alg: vpo-fl-offline} (offline setting) \& \ref{alg: vpo-fl-general} (online setting) have theoretical guarantees, 
 while Algorithm \ref{alg: vpo-fl-prac} is the more practical version.}
\begin{tabular}{ccccc}

\midrule[1.5pt]
 &  \makecell{Aggergation} &  \makecell{Reward\\Free}& \makecell{Traning\\Free}  &  \makecell{Multi-\\Group}\\ \hline
\makecell{MORLHF\vspace{-0.3em}\\\tiny\citep{wu2023fine}}    &$p=1$& \tiny\XSolidBrush & \tiny\XSolidBrush& \tiny\XSolidBrush\\ \hline
\makecell{RS \vspace{-0.3em}\\\tiny\citep{rame2024rewarded}}      & $p=1$ & \tiny\Checkmark  & \tiny\Checkmark  & \tiny\XSolidBrush \\ \hline
\makecell{MOD\vspace{-0.3em}\\\tiny\citep{shi2024decoding}} &   $p=1$   &   \makecell{\tiny\Checkmark}   &   \tiny\Checkmark & \tiny\XSolidBrush  \\ \hline
\makecell{PNB\vspace{-0.3em}\\\tiny \citep{zhong2024provable}} & $p\le1$ &   \tiny\XSolidBrush   &  \tiny\XSolidBrush &\tiny\XSolidBrush    \\ \hline
\makecell{Algorithm  \ref{alg: vpo-fl-offline} \& \ref{alg: vpo-fl-general}} & $p\le1$  & \tiny\Checkmark  & \tiny\XSolidBrush  & \tiny\Checkmark\\ 
\hline
\makecell{Algorithm \ref{alg: vpo-fl-prac}} & $p\le1$  & \tiny\XSolidBrush  & \tiny\Checkmark  & \tiny\Checkmark\\ \bottomrule[1.5pt]
\end{tabular}
\label{table:comparison}
\end{table}


\section{Preliminaries and Notations}

Denote the prompt space of the LLM  as $\cX$ and the response space as $\cY$. The distribution $\rho \in \Delta(\cX)$ represent the distribution of the prompt. A policy $\pi:\cX\to \Delta(\cY)$ represents an LLM that generates a response distribution given prompt $x$. In RLHF, we assume that we can get a pre-trained LLM $\pi_{\mathrm{ref}}$ that is usually trained on supervised data. The goal is to fine-tune the pre-trained model to align the model with the human preference on one particular task. To be more specific, given prompt $x\sim \rho$, LLM can generate two responses $y_1,y_2$ , then the human gives a preference feedback on the response pairs as either $y_1 \prec y_2$ or $y_1 \succ y_2$. The responses $y_1, y_2$ are labeled as $y_w, y_l$ respectively  with probability $\PP(y_1 \succ y_2\mid x)$, and are labeled as $y_l,y_w$ with probability $1-\PP(y_1\succ y_2\mid x)$. It is further assumed that the human preference is modeled by a Bradley-Terry (BT) model with the reward function $r^*(x,y) : \cX \times \cY \mapsto [0,B]$:
\begin{align}
    \PP(y_1 \succ y_2 \mid x) 
    = \sigma(r^*(x,y_1)-r^*(x,y_2)),\nonumber
\end{align}
where $\sigma(z) = \frac{1}{1+\exp(-z)}$ and $B\ge 1$.
Given the reward function $r$, the optimal policy $\pi_r = \argmax_\pi J(\pi)$ maximizes the expected reward function, with an additional KL divergence term that prevents the policy from deviating too much from $\pi_{\mathrm{ref}}$: 
\begin{align}
    \pi_r &= \arg\max_{\pi} J(\pi) = \arg\max_\pi \EE_{x\sim \rho}\EE_{y\sim \pi(\cdot \mid x)}\left[r^*(x,y) - \beta \DD_{\mathrm{KL}}(\pi\parallel \pi_{\mathrm{ref}})\right]. \label{eq:rlhf_optimal_policy}
\end{align}




In this paper, we consider both offline and online RLHF. For the offline RLHF setting, the LLM has access to a pre-collected offline data $\cD$ consisting of prompts and corresponding winning and losing responses, and the expectation in the optimal policy is calculated on the offline data. For the online setting, at each round LLM can generate two responses $y_1, y_2$ following the policy $\pi$, and then receive the preference feedback by human for data collection.

We assume there are $m$ known representations $\{\phi_i(x,y) \in \RR^d\}_{i \in [m]}$ and the corresponding reward function class $\{r_i(x,y) = \theta_i^\top \phi_i(x,y) \in [0,B], \|\phi_i\|_2\le 1, \|\theta_i\|_2 \le B\}$ for each objective $i \in [m]$. The true reward $r^*_i$ for objective $i$ can be written as $r^*_i(x,y) = (\theta_i^*)^\top \phi_i(x,y).$ This assumption is purely theoretical. In practice, the reward can be parameterized as $r^\theta$
  using a neural network, and our practical algorithm \ref{alg: vpo-fl-prac} also does not rely on this assumption. 

Since the preference only contains the information of $r_i(x,y_1)-r_i(x,y_2)$ for each objective $i$, rewards are invariant to constant shifts in feedback. Follow \citep{cen2024value}, we can assume there is a known policy $\pi_{\mathrm{base}}$ and constant $C$, such that for each $i \in [m]$, the reward parameter space $\Theta_i$ is defined as
\begin{align}
    \Theta_i = \left\{\theta \in \RR^d: \EE_{\pi_{\mathrm{base}}}\langle \theta_i, \phi_i(x,y) \rangle =C\right\}.\label{eq:theta base policy}
\end{align}


\subsection{Multi-Objective Learning}\label{sec:moL}

We assume that there are $m$ different objectives, and each objective has reward function $r_i(x,y) \in [0,B]$ for $i \in [m]$. As discussed in the introduction, we apply the definition of social welfare function in social choice theory to multi-objective setting and consider the weighted $p$-norm aggregation across objectives 
$$r(x,y) = \left(\sum_{i=1}^m \alpha_i r_i^p(x,y)\right)^{1/p}, p\le 1,$$ where $\alpha \in \Delta_{m-1}$ are weights of the objectives. 
Note that for positive rewards, aggregation yields Pareto optimality.  

The goal is to find the optimal policy for the aggregated reward function $r$. One natural approach to solving multi-objective RLHF is to first learn a reward model for each individual objective, and then aggregate these models to formulate a new reward. Finally, RL methods like PPO can be applied to optimize this new reward. However, this reward-based approach is significantly more computationally inefficient and unstable compared to reward-free approaches, such as DPO \citep{rafailov2024direct}. Additionally, it requires retraining the entire model for all possible reward aggregations, which becomes time-consuming when the aggregation parameters change.
In this work, we first provide a theoretical algorithmic framework for multi-objective RLHF, which naturally leads to the derivation of a reward-free algorithm. Based on this theoretical framework, we propose a \textit{nearly training-free} practical algorithm that incurs almost zero computational cost once the optimal policy for each objective is obtained.

Previous techniques cannot be easily applied to this setting. In fact, for the linear aggregation when $p=1$, the paper \citep{shi2024decoding} finds that the optimal policy $\pi_r$ can be written as a closed-form of the optimal policy $\pi_{r_i}$ as
$
    \pi_r(\cdot \mid x) \propto \pi_{\mathrm{ref}}(\cdot \mid x)\cdot \exp\left(\frac{1}{\beta}r(x,\cdot)\right),
$
and conduct a decoding algorithm MOD using this derivation.
By the linear aggregation $r(x,y) = \sum_{i=1}^m \alpha_i r_i(x,y)$ and $\sum_{i=1}^m \alpha_i =1,$ it is easy to verify that 
$\pi_r(y\mid x) \propto \prod_{i=1}^m \pi_{r_i}(y\mid x)^{\alpha_i}.$
Hence, one natural reward-free algorithm is to first learn the optimal policy $\pi_{r_i}$ for each objective using DPO, then calculate the optimal policy $\pi_r$. It is also a training-free algorithm once the optimal policy for each objective is known. However, 
when we choose the general aggregation with $p\le 1$, this derivation will fail due to the non-linear structure of the reward, making the problem much more complicated.

To avoid this technical difficulty, we draw inspiration from RL with Blackwell-approachability \citep{yu2021provably}, which focuses on minimizing the distance between the reward vector and a specified target set. This approach makes the problem more tractable since we can incorporate the non-linear aggregation into the definition of the target set. To be more specific, a target set $W \subset \RR^m$ is a convex set that is defined by
\begin{equation*}
    W_{p,c}^\alpha = \left\{z \in \RR_{\ge 0}^m: \left(\sum_{i=1}^m\alpha_i z_i^p\right)^{1/p} \ge c\right\},
\end{equation*}
 where $\alpha$ represents the weights assigned to the objectives by humans, $p$ represents the degree of fairness, and $c$ reflects the requirement of humans. In practice, we can learn $\alpha$ and $p$ from supervised and preference data, and the parameter $c$ can be provided by humans or chosen by parameter tuning. 
 The definition of target set implies that the group can be satisfied if the aggregation of the  reward function is larger than some pre-defined constant. 
 We also define the expected reward vector $S(\pi) \in \RR^m$ as 
$
    (S(\pi))_i= \EE_\pi[r_i^*(x,y)-\beta\DD_{KL}(\pi\| \pi_{\mathrm{ref}})], 
$
which is the expected reward following the policy $\pi$ with a regularized term of KL divergence.
Now assume $c,p,\alpha$ are all given, we can transfer the aggregation maximization problem to minimizing the distance between the expected reward vector (with some regularizer) and the target set $W$. 
The goal changes to minimizing the distance between $S(\pi)$ and $W_{p,c}^\alpha$:
\begin{equation}
   \pi^* =  \arg\min_\pi D(\pi) := \ d(S(\pi), W_{p,c}^\alpha). \label{eq:our formulation}
\end{equation}
Note that if we choose $c$ as the maximum value that there exists a policy $\pi$ that satisfies $d(S(\pi), W_{p,c}^\alpha)) = 0,$ then $\pi$ is one of the optimal policies and 
$$\pi = \arg \max_{\pi \in \Pi} \left(\sum_{i=1}^m \alpha_i \EE_\pi [r_i^*(x,y) - \beta \DD_{\mathrm{KL}}(\pi \| \pi_{\mathrm{ref}})]^p\right)^{1/p}$$
where every $\pi \in \Pi$ satisfies that $\EE_\pi[r_i^*(x,y)]-\beta\DD_{\mathrm{KL}}(\pi\|$ $\pi_{\mathrm{ref}})\ge 0$.  This statement highlights the connection between the original maximization problem Eq.~\eqref{eq:rlhf_optimal_policy} and our formulation Eq.~\eqref{eq:our formulation}. Therefore, our formulation can be viewed as an alternative metric for measuring the performance of LLMs in achieving multi-objective learning tasks. 

Now we demonstrate that more general aggregation methods can enable LLM to accommodate a wider range of objectives by selecting different values of $p$.

\begin{example}[$p=1:$ Linear Aggregation]
If we choose $p = 1$ and $c \ge  \max_{\pi} \sum_{i=1}^m \alpha_i \EE_\pi[r_i^*(x,y)]$, then the goal $D(\pi)$ will become 
\begin{align*}
    D(\pi) &=  d(S(\pi), W_{1,c}^{\alpha})
    = \frac{c-\sum_{i=1}^m \alpha_i \EE_{\pi}[r_i^*(x,y)] + \beta\DD_{\mathrm{KL}}(\pi\parallel \pi_{\mathrm{ref}})}{\sqrt{\sum_{i=1}^m \alpha_i^2}}.
\end{align*}
The last equality is because the selection of $c$.
From this derivation, we know that it is equivalent to the previous classical MORLHF with linear aggregation.

\end{example}

\begin{example}[$p=-\infty:$ worst-case reward]
    When $p= -\infty,$ the target set becomes 
    \begin{equation*}
        W_{-\infty,c}^\alpha =\left\{z \in \RR_{\ge 0}^m: \min_i z_i \ge c \right\}, 
    \end{equation*}
    which represents that the human wants to find an LLM with no obvious drawback for any of the objectives, i.e., requiring $\min_i\EE_{\pi}[r_i^*(x,y)] - \DD_{\mathrm{KL}}(\pi \|\pi_{\mathrm{ref}})$ larger than some threshold.
    Now we establish the connection between $p=-\infty$ and the max-min RLHF in \citep{chakrabortymaxmin}. The proof is provided in Appendix \ref{app:proof maxmin}. 
\begin{theorem}\label{thm:relationship_maximin}
    Define the max-min value as $c^*=\max_\pi [\min_i \EE_\pi [r_i^*] - \beta \DD_{\mathrm{KL}}$ $(\pi \|\pi_{\mathrm{ref}})]$. Then, if we choose the target set 
     $W_{-\infty, c}^\alpha$ such that $c$ is close to $c^*$, the resulting optimal policy also achieves a max-min value that close to $c^*$. To be more specific, we have 
     \begin{small}
     \begin{align*}
         \min_i \EE_\pi[r_i^*(x,y) - \DD_{\mathrm{KL}}(\pi \| \pi_{\mathrm{ref}})] \ge c^*-(\sqrt{m}+1) |c^*-c|.
     \end{align*}
     \end{small}
\end{theorem}
\end{example}



\subsection{Multi-Group Learning}
Beyond the single group setting, we also study the multi-group setting, where each group has a different aggregation approach (parameterized by $c,p$ and $\alpha$).
For each group $n$, we assume there is a target set
$$W^{(n)} = \left\{z \in \RR_{\ge 0}^m: \left(\sum_{i=1}^m \alpha_i^{(n)} z_i^{p^{(n)}}\right)^{\frac{1}{p^{(n)}}} \ge c^{(n)}\right\}$$ 
representing the aggregation rule across objectives for them.  We consider two types of goals that represent the effectiveness of alignment across diverse groups.

\paragraph{Consensus}
The first goal is called ``consensus", in which we wants to minimize the distance between the expected reward vector and the intersection of all target sets from diverse groups. Formally, the goal is to choose the optimal policy that minimizes the Euclidean distance
\begin{align}
    \pi^* = \arg \min_\pi d\left(S(\pi), \bigcap_{n=1}^N  W^{(n)}\right).
\end{align}

\paragraph{Malfare Function Minimization}

Another goal is to minimize the aggregated malfare function, where the malfare function for each group is the square of the distance between the expected reward vector and the group's target set. Formally, with group weight $\zeta_n>0$ and $\sum_{n=1}^N \zeta_n = 1,$  the goal is to find the optimal policy $\pi^*$ that
\begin{align}
    \pi^* = \arg\min_\pi \left(\sum_{n=1}^N \zeta_n\left(d^2(S(\pi), W^{(n)})\right)^q\right)^{1/q}, q\ge 1.\nonumber
\end{align}

\section{Algorithms for Multiple Objectives with Linear Aggregation}\label{sec:moalg}
In this section, we consider the simplest setting where the reward function is a linear aggregation, i.e. $
r(x,y) = \sum_{i=1}^m d_i r_i^*(x,y)$, where $d \in \RR^m$ is called the \textit{direction.} In fact, the linear aggregation can be viewed as projecting the reward vector onto a specific direction $d$.  As we will show later, this will become an essential sub-problem in our final algorithm for non-linear aggregation.

Given the dataset $\cD_i= \{x^j, (y_w^j,y_l^j)\}_{j \in [M]}$ containing $M$ data points for objective $i$, we provide offline and online algorithms to learn the optimal policy with respect to multiple objectives in a consistent way.  
  Now we aim to minimize the negative log-likelihood loss of preference data 
$$L_i(\theta_i) = -\sum_{(x,y_w,y_l) \in \cD_i}\log (\sigma(r_i^{\theta_i}(x, y_w)-r_i^{\theta_i}(x, y_l)))$$
for each objective $i$. 
Following \citep{cen2024value}, we can  refine our estimation of the reward by adding an additional exploration term $\max_\pi J(r^\theta, d, \pi) = \max_\pi $ $\EE_\pi[\sum_{i=1}^m d_i(r_i^{\theta_i}-\beta \DD_{\mathrm{KL}}(\pi \|\pi_{\mathrm{ref}}))] $, which represents the optimism/pessimism principle of the online/offline learning process. To be more specific, for the offline and online setting,  LLM learns the $\theta_{\mathrm{offline}}$ and $\theta_{\mathrm{online}}$ respectively by 
\begin{small}\begin{align}
\theta_{\mathrm{offline}} = \argmax_{\theta_1,\cdots, \theta_m} \left({-\max_\pi J(r^\theta, d, \pi) - \sum_{i=1}^m \eta L_i(\theta_i)}\right)\label{eq:estimate theta offline}
\\\theta_{\mathrm{online}}=\argmax_{\theta_1, \cdots, \theta_m}\left(\max_\pi J(r^\theta, d, \pi)-\sum_{i=1}^m \eta L_i(\theta_i)\right),\label{eq:estimate theta online}
\end{align}
\end{small}
where we use a single parameter $\theta$ to refer the set $\{\theta_i\}_{i \in [m]}.$ The difference lies in the optimism and the pessimism principle. In the offline setting, we subtract the exploration term to avoid over-optimization \citep{cen2024value,liu2024provably} while in the online setting, we add the exploration term to encourage the model to explore \citep{cen2024value}.
Then, the LLM executes the greedy policy $\pi^\theta = \arg\max_{\pi} J(r^\theta, d, \pi)$ to generate the response and receives the human feedback $(y_w, y_l)$.  We called the algorithm \textbf{M}ulti-\textbf{O}bjective \textbf{P}rojection (MOP), and the pseudocode for online setting is shown in Algorithm \ref{alg: vpo-fl}. (There is no Line 4 and the output only has $\theta$ for the offline setting.)
\begin{algorithm}[H] 
     \begin{algorithmic}[1] 
         \caption{MOP-Reward Based (RB)} 
         \label{alg: vpo-fl} 
         \STATE \textbf{Input}: Direction $d,$ dataset $\{\cD_i\}_{i \in [m]}$, $\eta,\beta$.
         \STATE Calculate $\theta_{\mathrm{
         offline}}$ by Eq.~\eqref{eq:estimate theta offline} or $\theta_{\mathrm{online}}$ by Eq.~\eqref{eq:estimate theta online}.
             \label{line:rlhf-obj-general}
            \STATE Execute $\pi^\theta = \arg\max_{\pi} J(r^\theta_1, r^\theta_2,\cdots, r_m^\theta, d, \pi)$.
            \STATE Given the prompt $x$, Generate two responses $y_1,y_2\sim \pi$, and get a preference $y = (y_w, y_l)$.
            \STATE \textbf{Output:} Data point $D= \{x, (y_w, y_l)\}$ and $\theta.$
     \end{algorithmic} 
\end{algorithm} 



The computational cost of Algorithm \ref{alg: vpo-fl} mainly lies on Line 2. In fact, it needs to learn multiple reward functions directly, and then get the estimation of the optimal policy, which requires a joint optimization subprocedure. In the following, we consider the reward-free version of Algorithm \ref{alg: vpo-fl}.  

\paragraph{Reward-Free Modification}
We now show that Algorithm \ref{alg: vpo-fl} can be easily adapted to a reward-free version. We mainly consider the online setting since the offline setting is similar. 
Denote $\pi^\theta = \argmax_\pi J(r^\theta, d, \pi).$ By the same derivation in \citep{cen2024value}, we can get
\begin{align*}
    &J(r^\theta, d, \pi)=C-\beta \EE_{x\sim \rho, y\sim\pi_{\mathrm{base}}}\left[\log \frac{\pi^\theta(y\mid x)} { \pi_{\mathrm{ref}}(y\mid x)}\right],
\end{align*}
where $C$ and $\pi_{\mathrm{base}}$ are the constant and the baseline policy in Eq.~\eqref{eq:theta base policy},  $\pi_{\theta_i}$ is the policy for objective $i$ and $\pi^\theta \propto \pi_{\mathrm{ref}}(y\mid x) \cdot \prod_{i=1}^m \left(\pi_{\theta_i}(y\mid x)\right)^{d_i}$ is the optimal policy for linear aggregation. The detailed derivation above will be provided in Appendix \ref{app:derivation}.
By the derivation in \citep{rafailov2024direct}, you can further get the reward-free version of Eq.~\eqref{eq:estimate theta online} as 
\begin{small}
\begin{align}
    \theta &= \argmin_{\theta}\Bigg\{ \beta \EE_{ \pi_{\mathrm{base}}}\log \pi^\theta(y\mid x)-\eta \sum_{i=1}^m \ell(\cD_i,\theta_i)\Bigg\}\label{eq:rfupdate}\end{align}
    \end{small}
    where 
    $\ell(\cD_i,\theta_i) = \sum_{(x,y_w,y_l) \in \cD_i}\log \sigma\Big(\beta \log\frac{\pi_{\theta_i}(y_w\mid x)}{ \pi_{\mathrm{ref}}(y_w\mid x)} - \beta \log\frac{ \pi_{\theta_i}(y_l\mid x)}{ \pi_{\mathrm{ref}}(y_l\mid x)}\Big)$ is the reward-free loss function, and the expectation $\EE_{\pi}[\cdot]$ means $\EE_{x\sim \rho, y\sim \pi(\cdot \mid x)}[\cdot]$. 
\begin{algorithm}[H] 
     \begin{algorithmic}[1] 
         \caption{MOP-Reward Free (RF) (Online Version)} 
         \label{alg: vpo-rf} 
         \STATE \textbf{Input}: Direction $d,$ dataset $\{\cD_i\}_{i \in [m]}$, $\eta,\beta$.
         \STATE Calculate $\theta_{\mathrm{online}}\in \RR^m$ by Eq.~\eqref{eq:rfupdate} and $\pi = \pi^{\theta}.$
            \STATE Given the prompt $x$, Generate two responses $y_1,y_2\sim \pi$, and get a preference $y = (y_w, y_l)$. 
            \STATE \textbf{Output:} Data point $D = \{x, (y_w, y_l)\}$ 
            and $\theta.$
     \end{algorithmic} 
\end{algorithm}
The Eq.~\eqref{eq:rfupdate} involves an optimization problem on $\theta$, which is a complicated joint optimization since $\theta$ refers to $m$ parameter $\theta_1, \cdots, \theta_m$. In Appendix \ref{app:derivation}, we further study the computational cost of Eq.~\eqref{eq:rfupdate}, showing that the gradient descent update rule can be easily computed once the expectation of the score function is available. 

\section{General Algorithm for Preference Aggregation}\label{sec:pref aggregation}
In this section, we introduce general offline and online algorithms that work for both linear and non-linear preference aggregation, and provide their theoretical guarantees. Both algorithms transform the non-linear aggregation into a series of linear aggregation sub-problem, using Algorithm \ref{alg: vpo-fl} and \ref{alg: vpo-rf} as their core sub-procedures. 

\subsection{Offline Algorithm}

Now we introduce our algorithm \textbf{M}ulti-\textbf{O}bjective \textbf{P}rojection \textbf{O}ptimization (MOPO), which follows from the competitive RL with Blackwell-approachability literature \citep{yu2021provably}. We receive the offline data set $\cD =\{\cD_{i}\}_{i \in [m]}$ which contains $M$ data points $\cD_i$ for each objective $i$. 
The algorithm learns the reward or optimizes the policy directly from the offline data. 
Our algorithm contains $T$ iterations. In each iteration $t$, we first project the reward vector on the direction $d^t \in \RR^m$ defined in the last iteration, i.e. $r(x,y) = \sum_{i=1}^m d_i^t r_i(x,y)$, and then using the sub-procedure in the previous section to find the estimated parameter $\theta^t$ and determine the corresponding policy $\pi^t$. Finally, we derive the estimated expected reward vector $V^t \in \RR^m$ as $(V^t)_i=\EE_{\pi^t}[r_i^{\theta^t}(x,y) - \DD_{\mathrm{KL}}(\pi^t\|\pi_{\mathrm{ref}})]$, and calculate the averaged reward vector as $\overline{V}^t = \frac{1}{t}\sum_{j=1}^t V^j.$ Finally, the direction is updated based on the projection of the estimated point $\overline{V}^t$ onto the target set, guided by either the consensus problem or the malfare function minimization problem. The pseudocode is in Algorithm \ref{alg: vpo-fl-offline}.

The key component of our algorithm is the direction calculation in each iteration. Intuitively, the algorithm aims to optimize the reward to guide the expected reward vector toward the target set as effectively as possible. Suppose the target set is $W$, the direction can be calculated by 
$d^{t+1} = \mathrm{Proj}(W,\overline{V}^t) = \frac{\Pi_W(V)-V}{\|\Pi_W(V)-V\|}.$
For the consensus problem, we can substitute into $W = \bigcap_{n=1}^N W^{(n)}$ and get 
\begin{equation}\label{eq:dir consensus}
    d^{t+1} = \mathrm{Proj}\left(\bigcap_{n=1}^N W^{(n)}, \overline{V}^t\right).
\end{equation}
For the malfare function minimization problem, we can first calculate the projection to each target set $W^{(n)}$ and then aggregate them as 
\begin{small}
\begin{equation}\label{eq:dir malfare}
    d^{t+1}=\sum_{n=1}^N \mathrm{Proj}\left(W^{(n)}, \overline{V}^t\right)\cdot \frac{\zeta_n\|W^{(n)}- \overline{V}^t\|_2^{2q-1}}{\left(\sum_{n=1}^N \zeta_n\|W^{(n)}- \overline{V}^t\|_2^{2q}\right)^{\frac{2q-1}{2q}}}.
\end{equation}
\end{small}
\begin{algorithm}[t]
     \begin{algorithmic}[1]
         \caption{MOPO-Offline}
         \label{alg: vpo-fl-offline}
         \STATE \textbf{Initial}: Dataset $\cD=\{\cD_i\}_{i \in [m]}$, $\{W^{(n)}\}_{n \in [N]}$, $\eta, \beta.$ \FOR{$t=1,2,\cdots,T$}
            \STATE Collect $\theta^t$ by MOP-RB$(\overline{d^t}, \cD)$ or MOP-RF$(\overline{d^t},\cD)$. Get the corresponding policy $\pi^t = \pi^{\theta^t}$.
            \STATE Calculate the point $V^t = \EE_{\pi^t}[r_i^{\theta^t}(x,y) - \beta\DD_{\mathrm{KL}}$ $(\pi^t\| \pi_{\mathrm{ref}})]= C-\beta \EE_{y\sim \pi_{\mathrm{base}}}\left[\log \frac{\pi^{\theta_i^t}(y\mid x)}{\pi_{\mathrm{ref}}(y\mid x)}\right] $ $+ \beta \EE_{y\sim \pi^t}\left[\log \frac{\pi^{\theta_i^t}(y\mid x)}{\pi^t(y\mid x)}\right]$, and $\overline{V}^t = \frac{t-1}{t}\overline{V}^{t-1} + V^t.$
            \STATE Calculate the direction $d^{t+1}$ by Eq. \eqref{eq:dir consensus} or Eq.~\eqref{eq:dir malfare}, and calculate $\overline{d^{t+1}} = \frac{d^{t+1}}{\|d^{t+1}\|_1}$.
         \ENDFOR
         \STATE \textbf{Return} $\tilde{\pi}^T = \frac{1}{T}\sum_{t=1}^T \pi^t.$
     \end{algorithmic}
\end{algorithm}

Note that if we apply MOPO with $p=1$, it reduces to the classical MORLHF algorithm. This is because the direction $d^t = \mathrm{Proj}(V^t, W_{1,c}^\alpha)=\alpha$ for each $t$ as long as $c$ is large. However, for $p\neq 1,$ MOPO solves the non-linear aggregation maximization problem by transforming into a series of subproblems, in which each subproblem only contains the linear aggregation and can be easily solved using any previous algorithm. Thus, MOPO serves as a general framework for MORLHF with non-linear aggregation. Moreover, suppose we use MOP-RF for each subproblem, MOPO is also a reward-free algorithm since the current reward vector can be computed as 
$$(V^t)_i = \EE_{\pi^t}[r_i^{\theta^t}(x,y) - \beta\DD_{\mathrm{KL}}(\pi^t\| \pi_{\mathrm{ref}})] = C-\beta \EE_{y\sim \pi_{\mathrm{base}}}\left[\log \frac{\pi^{\theta_i^t}(y\mid x)}{\pi_{\mathrm{ref}}(y\mid x)}\right] + \beta \EE_{y\sim \pi^t}\left[\log \frac{\pi^{\theta_i^t}(y\mid x)}{\pi^t(y\mid x)}\right].$$ You can See Appendix \ref{app:expected reward vector derivation} for the derivation.
Now we provide theoretical guarantee of Algorithm \ref{alg: vpo-fl-offline}. The following result shows that MOP-offline can learn the optimal policy well if the offline dataset $\cD$ has sufficient coverage for each objective.

\begin{theorem}[Consensus Problem]\label{thm:offline}
    Let $\eta = 1/\sqrt{M}$ and $\Sigma_{\cD_i} = \frac{1}{M}\sum_{(x,y_w,y_l) \in \cD_i} $ $(\phi(x,y_{w})-\phi(x,y_l))(\phi(x,y_{w})-\phi(x,y_l))^\top$ be the empirical  covariance matrix of the data for objective $i$. We consider the consensus problem that $W = \bigcap_{n=1}^N W^{(n)}$ and calculate the direction using Eq.~\eqref{eq:dir consensus}.   Define $D(\pi) = d(S(\pi), \cap_{n=1}^N W^{(n)})$. For $\delta \in (0,1)$, with probability at least $1-\delta$, we have
    \begin{align*}
        D(\tilde{\pi}^T )- D(\pi^*) \le \frac{m^{3/2}\sqrt{d}}{\sqrt{M}}\cdot \widetilde{\cO}\left(\mathrm{poly}\left(e^{B'}, \left(\min_i\lambda_{\min}(\Sigma_{\cD_i}) + \frac{1}{M}\right )^{-1}\right)\right)+ \widetilde{\cO}\left(\frac{B\sqrt{m}}{\sqrt{T}}\right). 
    \end{align*}
    
\end{theorem}
The above theorem shows that the final gap of returned policy depends on the coverage term $\min_i \lambda_{\min}(\Sigma_{\cD_i})$ of the offline dataset and the number of iterations $T$.  As $T$ increases, we achieve a standard convergence rate of $\widetilde{\cO}(1/\sqrt{M})$, which is standard in prior offline RL algorithms \citep{jin2021bellman,liu2020provably}.
 We also provide the theoretical guarantee for malfare function minimization. 
\begin{theorem}[Malfare] \label{thm:malfareoffline}
    With the same definitions and conditions in Theorem \ref{thm:offline}, we consider the malfare function minimization problem with an integer\footnote{We focus on the integer case to simplify the proof.} exponential parameter $q\in \NN^+$ and use Eq.~\eqref{eq:dir malfare} for the direction. Define $D_q(\pi) = \sqrt[2q]{\sum_{n=1}^N \zeta_nd^{2q}(S(\pi), W^{(n)})}$. For $\delta \in (0,1)$, with probability at least $1-\delta$ we have 
    \begin{align*}
        &D_q(\tilde{\pi}^T )- D_q(\pi^*) \\&\le \frac{Nm^{3/2}\sqrt{d}}{\sqrt{M}}\cdot \widetilde{\cO}\Bigg(\mathrm{poly}\Bigg(e^{B'}, \min_i\lambda_{\min}\left(\Sigma_{\cD_i} + \frac{1}{M}\right)^{-1}, (\min_{n \in [N]}\zeta_n)^{-1/2q}\Bigg)\Bigg) + \widetilde{\cO}\left( B\sqrt{m}T^{-1/2q}\right). 
    \end{align*}
\end{theorem}

\begin{algorithm}[t]
     \begin{algorithmic}[1]
         \caption{VPO-objective-learning-general}
         \label{alg: vpo-fl-general}
         \STATE \textbf{Initial}: $\cD = \emptyset$. parameter $\{p^{(n)}, c^{(n)}\}_{n \in [N]}$, $\eta, \beta$.\FOR{$t=1,2,\cdots,T$}
         \STATE Calculate $\tilde{\theta}_i^t = \arg\min_\theta  L_i^t(\theta)$ for all $i \in [m]$.
            \STATE Estimate $\hat{\alpha}^{t,(n)} = \{\hat{\alpha}_i^{t,(n)}\}_{i \in [m]}$ for each $n \in[N]$ by MLE with $\cD$ and $\{\tilde{\theta}^t_i\}_{i \in [m]}$ by Eq.~\eqref{eq:MLE alpha}
            \STATE Calculate $W^{t,(n)}=W^{\alpha^{t,(n)}}_{p^{(n)}, c^{(n)}}$ where $\alpha^{t,(n)} = \frac{t-1}{t}\alpha^{t-1,(n)} + \frac{1}{t}\hat{\alpha}^{t,(n)}$ for each $n \in [N].$
            \STATE Collect $D_t,\theta^t$ by MOP-RB$(\overline{d^t}, \cD)$ or MOP-RF$(\overline{d^t}, \cD)$, and update $\cD = \cD \cup D_t.$ 
            \STATE Calculate the point $V^t = \EE_{\pi^t}[r_i^{\theta^t}(x,y) - \beta\DD_{\mathrm{KL}}$ $(\pi^t\| \pi_{\mathrm{ref}})]= C-\beta \EE_{y\sim \pi_{\mathrm{base}}}[\log \frac{\pi^{\theta_i^t}(y\mid x)}{\pi_{\mathrm{ref}}(y\mid x)}] $ $+ \beta \EE_{y\sim \pi^t}[\log \frac{\pi^{\theta_i^t}(y\mid x)}{\pi^t(y\mid x)}]$, and $\overline{V}^t = \frac{t-1}{t}\overline{V}^{t-1} + V^t.$
            \STATE Calculate the direction $d^{t+1}$ by Eq.~\eqref{eq:dir consensus} or Eq.~\eqref{eq:dir malfare}, and calculate $\overline{d^{t+1}} = \frac{d^{t+1}}{\|d^{t+1}\|_1}$.
         \ENDFOR
         \STATE \textbf{Return} $\tilde{\pi}^T = \frac{1}{T}\sum_{t=1}^T \pi^t.$
     \end{algorithmic}
\end{algorithm}
\subsection{Online Algorithm}
Now we provide the online version of MOPO, which is similar to the offline setting. 
 The main difference is the adoption optimism principle (Eq.~\eqref{eq:estimate theta online}) rather than the pessimism principle (Eq.~\eqref{eq:estimate theta offline}). Additionally, the dataset is collected incrementally online, and we also estimate the importance weight $\alpha$ instead of assuming it is known.
 
Additionally, rather than assuming the weight is known, we estimate it based on the frequency with which humans report the objective. This method also works offline by using the frequency of related data in the dataset. At each round $t$, given a prompt $x^t \sim \rho$ and two responses $y_1$ and $y_2$, each group $n$ identifies an objective $I^{t,(n)} \in [m]$ showing the greatest difference and provides preference feedback $(y_w^{t,(n)}, y_l^{t,(n)})$ on that objective. The model collects the data $(x^t,y_w^{t,(n)},y_l^{t,(n)},I^{t,(n)})$ into $\cD^{(n)}$ for all group $n$. Next, we model how humans select the objective index. For responses $y_w$ and $y_l$, the gap on objective $i$ is quantified as $|\alpha_i \cdot (r_i(x, y_w) - r_i(x, y_l))|$, with the selection following a softmax distribution:
\begin{align*}
   \PP( I\mid \alpha, r^*,x,y_w,y_l) \propto \exp(\alpha_i\cdot |r_i^*(x,y_w)-r_i^*(x,y_l)|). 
\end{align*}
Then if we define the likelihood function as 
$$\LL(\alpha, \cD^{(n)}, \theta) = \sum_{(x,y_w,y_l,I) \in \cD^{(n)}}\PP(I\mid\alpha, x, y_w, y_l, r^\theta),$$ we can estimate the importance weight vector for each group by MLE as
\begin{align}\label{eq:MLE alpha}
    \hat{\alpha}^{t,(n)}=\argmax_{\alpha \in \Delta_{m-1}} \LL(\alpha, \cD^{(n)}, \tilde{\theta}^t),
\end{align}
where we use an estimated reward parameter $\tilde{\theta}^t$ to approximate $\theta^*$. 
Before we present our results, we assume there is a gap between the reward obtained by following the optimal policy $\pi^*$ and the reference  policy $\pi_{\mathrm{ref}}$. This gap is reasonable since the expected reward should be improved after fine-tuning.
 \begin{assumption}\label{assum:gap} There exists a constant $\gamma>0$ such that 
     $$\min_{i \in [m]}\EE_{x\sim \rho, y_1\sim \pi^*, y_2\sim \pi_{\mathrm{ref}}}|r_i^*(x,y_1) - r_i^*(x,y_2)|\ge \gamma.$$
 \end{assumption}
\noindent The following theorems show that Algorithm \ref{alg: vpo-fl-general} is a no-regret online algorithm that can converge to the optimal policy for the consensus problem and social malfare minimization problem, with importance weight estimation. 
 \begin{theorem}[Consensus]\label{thm:online}
     For the consensus problem, suppose the Assumption \ref{assum:gap} holds and the group $n$ has parameter $p^{(n)}$ and $c^{(n)}.$, then if we use  Eq.~\eqref{eq:dir consensus} to calculate the direction, for $\delta \in (0,1)$ and $\eta = 1/\sqrt{T}$, with probability at least $1-\delta$ we have 
     \begin{align*}
         D(\tilde{\pi}^T)-D(\pi^*) \le \gamma^{-1}\mathrm{poly}(\exp(1/\beta), m,N,e^B, d, \log(1/\delta), \kappa, (\min_{n \in [N]}p^{(n)})^{-1}), B_1) \cdot \widetilde{\cO}(1/\sqrt{T}),
     \end{align*}
     where $\tilde{\pi}^T = \frac{1}{T}\sum_{t=1}^T \pi^t$, and $ \kappa= \sup_{x,y}\frac{\pi_{\mathrm{base}}(y\mid x)}{\pi_{\mathrm{ref}}(y\mid x)}$, $B_1=2\sqrt{m}(B+\max_n c^{(n)})$ are constants.
 \end{theorem}

 \begin{theorem}[Malfare]\label{thm:malfare online}
With the same setting in Theorem \ref{thm:online}, if we consider the malfare function minimization problem with an integer exponential parameter $q\in \NN^+$ and uses Eq.~\eqref{eq:dir malfare} to compute the direction, then for $\delta \in (0,1)$ and $\eta = 1/\sqrt{T}$, with probability at least $1-\delta$ we have
\begin{align*}
    &D_q(\tilde{\pi}^T) - D_q(\pi^*) \\&\le \gamma^{-1}\mathrm{poly}(\exp(1/\beta), m,N,e^B, d, \log(1/\delta),\kappa, B_1, (\min_{n \in [N]}p^{(n)})^{-1}, (\min_{n \in [N]}\zeta_n)^{-1/2q})\cdot \widetilde{\cO}(T^{-1/2q}),
\end{align*}
where $\tilde{\pi}^T = \frac{1}{T}\sum_{t=1}^T \pi^t$, and $ \kappa= \sup_{x,y}\frac{\pi_{\mathrm{base}}(y\mid x)}{\pi_{\mathrm{ref}}(y\mid x)}$, $B_1=2\sqrt{m}(B+\max_n c^{(n)})$ are constants.

\end{theorem}

\section{Experiments} 
In this section, we provide our practical algorithm. We run the offline version of MOPO, and use MOD \citep{shi2024decoding} as the sub-procedure to solve the linear aggregation maximization problem at each round. The pseudocode is shown in Algorithm \ref{alg: vpo-fl-prac}.
\begin{algorithm}[H] 
     \begin{algorithmic}[1] 
         \caption{MOPO(Practical Version)-Offline} 
         \label{alg: vpo-fl-prac} 
         \STATE \textbf{Initial}: $\overline{d}^0 = (\frac{1}{m},\cdots, \frac{1}{m})^\top $, dataset $\cD_{\mathrm{offline}}$, $W$.
         \STATE Calculate the optimal policy $\pi_i$ for each objective $i \in [m]$ using offline dataset $\cD_{\mathrm{offline}}$.
         \FOR{$t=1,2,\cdots,T$} 
            \STATE Execute $\pi^t=\mathrm{MOD}(\{\pi_i\}_{i \le m}, \overline{d^{t-1}})$. 
            \STATE Calculate the point $V^t \in \RR^m$. 
            \STATE Calculate the direction $d^{t} = \mathrm{Proj}(W, V^t),$ and get the average direction $\overline{d^{t}} = \frac{1}{t}\sum_{j=1}^t \frac{d^j}{\|d^j\|_1}.$
         \ENDFOR 
     \end{algorithmic} 
\end{algorithm} 
Note that the algorithm average the direction instead of averaging the estimated reward vector function, which can lead to a more stable result. To execute the Line 2, following the previous paper \citep{shi2024decoding}, we first fine-tune the model LLAMA2-7B on the Anthropic-HH dataset \citep{ouyang2022training} to get the reference policy $\pi_{\mathrm{ref}}$. We then get the optimal policy $\pi_i$ for each objective $i \in \{1,2,3\}$ using PPO approach trained on three off-sheld reward model:
\begin{itemize}
    \item Harmlessness: \url{https://huggingface.co/Ray2333/gpt2-large-harmless-reward_model}
\item Helpfulness: \url{https:
//huggingface.co/Ray2333/gpt2-large-helpful-reward_model}
\item Humor: \url{https://huggingface.co/mohameddhiab/humor-no-humor}
\end{itemize}
\paragraph{Single-Group Problem with Multiple Objectives}
Note that MOPO is an iterate algorithm, thus the computational cost can still be high due to the large number of iterations. In practice, we can mitigate this by either reducing the number of iterations or computing a single gradient update per iteration \citep{guo2024direct}. In our experiments, we set the number of iterations to 7, striking a balance between computational efficiency and performance.   To compute the expected reward vector $V^t$, we calculate the expectation by taking the expectation over 100 training samples, and we believe the performance of MOPO can be improved by using more training samples to calculate the expectation. 

For $p=0.5$, we compare MOPO with the RS algorithm \citep{rame2024rewarded}, MOD algorithm \citep{shi2024decoding} (both of which use linear aggregation), and a baseline AR that directly aggregates the reward using non-linear aggregation. The experimental results show that MOPO performs generally better. 
The following table presents the results for MORLHF with the objectives (Harmless, Helpful) and (Harmless, Humor). 
Additionally, since the aggregation only works for non-negative rewards, when using AR to aggregate the reward, we take 
$\max\{r_i,0\}$ instead of 
$r_i$
  for each objective. Although this is the only reasonable approach, we observe that it performs poorly. This may be due to the vanishing gradient problem, as the gradient of 
$\max\{r_i,0\}$ becomes zero when the reward is negative. The experiment shows that our algorithm MOPO generally outperforms the previous one. 


\begin{table}[H]\footnotesize \centering
 \caption{Comparison of previous representative works for MORLHF with $p=0.5, c=0.5$ and the objective Harmless and Helpful. The score is the distance between the reward vector and the target set. The smaller one is better.}
\begin{tabular}{ccccc}
\midrule[1.5pt]
$\alpha$  &  \makecell{Ours} &  \makecell{RS}& \makecell{MOD}  &  \makecell{AR}\\ \hline
\makecell{(0.1,0.9)}    &\textbf{0.229}&  0.971   & 0.808 & 0.555\\ 
\makecell{(0.3,0.7)}      & \textbf{0.051} & 0.666  & 0.079 &  1.459\\ 
\makecell{(0.5,0.5)} &  \textbf{0.015}    &  0.078    & 0.103&  1.314 \\
\makecell{(0.7,0.3)} & \textbf{0.067}& 0.707  &0.800 & 1.004\\ 
\makecell{(0.9,0.1)} & \textbf{0.184}  &1.153 & 1.137&1.526\\ 
\bottomrule[1.5pt]
\end{tabular}
\label{table:0.5}
\end{table}
\begin{table}[H]\footnotesize \centering
 \caption{Comparison of previous representative work for MORLHF with $p=0.5$, $c = 1.3$ and the objective Harmless and Humor. The score is the distance between the evaluated reward vector and the target set. The smaller one is better.}
\begin{tabular}{ccccc}

\midrule[1.5pt]
$\alpha$  &  \makecell{Ours} &  \makecell{RS}& \makecell{MOD}  &  \makecell{AR}\\ \hline
\makecell{(0.1,0.9)}    &\textbf{0.335}&   0.362 & 0.337 & 1.767\\ 
\makecell{(0.3,0.7)}      & 0.578 & 0.678  & \textbf{0.572}  & 2.011 \\ 
\makecell{(0.5,0.5)} &   \textbf{0.720}   &   0.882   & 0.723   & 1.970  \\
\makecell{(0.7,0.3)} & \textbf{0.630} &   0.860   &  0.722 &2.411\\ 
\makecell{(0.9,0.1)} & \textbf{0.217}  & 0.391  & 0.396 & 2.068\\ 
\bottomrule[1.5pt]
\end{tabular}
\label{table:3}
\end{table}
For $p=-\infty$, we compare MOPO with max-min RLHF \citep{chakrabortymaxmin}. We choose the target set $W_{\infty, 1.5}^{\alpha}$ for objective pairs (Harmless, Humor) and $W_{\infty, 0.5}^\alpha$ for objective pairs (Harmless, Helpful). The result shows that we achieve stable and better performance.  
\begin{table}[H]\footnotesize \centering
 \caption{Comparison with max-min RLHF for objectives Humor and Harmless. The number pair represents the reward vector. The pair with the larger minimum value is better.}
\begin{tabular}{ccc}

\midrule[1.5pt]
 &\makecell{Ours} & Max-Min RLHF \\ \hline
\makecell{(Harmless, Humor)}  & (1.097,1.297) & \textbf{(1.530, 1.146)}\\  
\makecell{(Harmless, Helpful)} &\textbf{(0.034,0.497)} & (-0.135, 0.393)\\
\bottomrule[1.5pt]
\end{tabular}
\label{table:1}
\end{table}



\paragraph{Multi-Group Problem with Multiple Objectives}

We perform the experiments on Harmless and Humor dataset when we have $N=2$ groups. One group has the target set $W_{0.5,1.3}^\alpha$ and the other has the target set $W_{-\infty,1}^\alpha$. We compare our consensus algorithm with Eq.~\eqref{eq:dir consensus} and a variant of max-min RLHF. In this variant of max-min RLHF, we use $\min\{r_1,r_2, \alpha_1\cdot (\max\{r_1,0\})^{0.5} + \alpha_2\cdot (\max\{r_2,0\})^{0.5} \}$ as the reward. 
We also perform experiments on the Harmless and Helpful dataset with the target set $W_{0.5,0.5}^\alpha$ and the target set $W_{-\infty, 0}^\alpha$. 
The following tables show the experiment results. The results show that our algorithms perform relatively stable and better, while this variant of max-min RLHF performs unstable.  However, note that this variant of max-min RLHF also needs retraining whenever one group changes the aggregation approach, which is time-consuming for real-world applications.

\begin{table}[H]\footnotesize \centering
 \caption{Comparison of MOPO and a variant of Max-Min RLHF on multi-group setting. The objectives are Harmless and Humor.  The score is the distance between the evaluated reward vector and the target set. The smaller one is better.}
\begin{tabular}{ccc}

\midrule[1.5pt]
$\alpha$  &  \makecell{Ours} &   \makecell{Max-Min RLHF}\\ \hline
\makecell{(0.1,0.9)}    &\textbf{0.408}& 0.992\\ 
\makecell{(0.3,0.7)}   &\textbf{0.577} & 1.171 \\ 
\makecell{(0.5,0.5)} &  0.708 & \textbf{0.429}  \\
\makecell{(0.7,0.3)} & \textbf{0.619}&1.342\\ 
\makecell{(0.9,0.1)} & 0.406& \textbf{0.208}\\ 
\bottomrule[1.5pt]
\end{tabular}
\label{table:4}
\end{table}

\begin{table}[H]\footnotesize \centering
 \caption{Comparison of MOPO and a variant of Max-Min RLHF on multi-group setting. The objectives are Harmless and Helpful.  The score is the distance between the evaluated reward vector and the target set. The smaller one is better.}
\begin{tabular}{ccc}

\midrule[1.5pt]
$\alpha$  &  \makecell{Ours} &   \makecell{Max-Min RLHF}\\ \hline
\makecell{(0.1,0.9)}    &\textbf{\textbf{0.230}}& 1.073\\ 
\makecell{(0.3,0.7)}   &\textbf{0.052} & 0.123 \\ 
\makecell{(0.5,0.5)} &  \textbf{0.015} & 0.261 \\
\makecell{(0.7,0.3)} & \textbf{0.067}&0.204\\ 
\makecell{(0.9,0.1)} & 0.184& \textbf{0.121}\\ 
\bottomrule[1.5pt]
\end{tabular}
\label{table:5}
\end{table}



\section{Conclusion}
In this paper, we study efficient multi-objective and multi-group RLHF problems under non-linear aggregation. By transforming the non-linear aggregation maximization into a series of linear aggregation maximization sub-problems, we find a computationally efficient algorithm that can converge to the optimal policy. Theoretically, we establish a general framework with converge guarantees for both offline and online settings, and the framework is also adaptable to a reward-free version. Empirically, we present a training-free framework given the reward functions and optimal policies for all objectives.

There are many future directions worth exploring. First, one can study how to learn the parameter $p$ in the aggregation function like \citep{pardeshi2024learning} using the preference feedback. Second, one can further study the token-level MORLHF \citep{zeng2024token} based on our idea. Last, it is interesting to further study the multiple preference aggregation in Stochastic Transitivity model \citep{fishburn1973binary} instead of BTL model, and further discuss the relationship between them and previous distortion negative results \citep{anshelevich2021distortion}. 

\section{Acknowledgement}
This work is supported in part by NSF AI Institute for Societal Decision Making under award IIS2229881 and ONR award N000142212363.


\newpage
\appendix
\onecolumn
\section{Experiment Details}\label{app:experiment}

\section{Proof of Theorems}
\subsection{Proof of Theorem \ref{thm:relationship_maximin}}\label{app:proof maxmin}

\begin{proof}
    Then, suppose the reward vector $S(\pi)$ is $(s_1,\cdots, s_m)^\top$, then by the definition of $D(\pi)$, we have 
    \begin{align*}
        D(\pi) = \sum_{i=1}^m \max\{c-s_i,0\}^2\le \sum_{i=1}^m \max\{c-s_i^*,0\}^2,
    \end{align*}
    where $s_i^* = (S(\pi^*))_i = \EE_{\pi^*}[r_i^*(x,y) - \beta\DD_{\mathrm{KL}}(\pi^*\|\pi_{\mathrm{ref}})].$
    Hence we have 
    \begin{align*}\max\{c-\min_i s_i,0\}^2&\le  \sum_{i=1}^m \max\{c-s_i,0\}^2\\&\le \sum_{i=1}^m \max\{c-s_i^*,0\}^2\\&\le m\cdot (c-\min_i s_i^*)^2\le m(c-c^*)^2,\end{align*}
which implies that .
    $$c-\min_i s_i \le \sqrt{m}\cdot |c-c^*|,$$ and 
    $$c^*-\min_i s_i \le (\sqrt{m}+1) |c^*-c|.$$

    Thus, if $c$ is selected such that $|c-c^*|$ is small, then we can also find a policy $\pi$, such that $$\min_i \EE_\pi[r_i^*(x,y) - \DD_{\mathrm{KL}}(\pi \| \pi_{\mathrm{ref}})] \ge c^*-(\sqrt{m}+1) |c^*-c|.$$
    \end{proof}

\subsection{Proof of Theorem \ref{thm:offline}}

For simplicity, for the following proof, we use $\EE_{\pi^t}[\cdot]$ to represent $\EE_{x\sim \rho, y\sim \pi^t(\cdot \mid x)}[\cdot]$.
Since we do not assume the target set $W^*$ is approachable, we have the following property for the approachability:

\begin{lemma}\label{lemma:approach}
For each $\theta \in \RR_{\ge0}^m$ with $\|\theta\|_2 = 1$, we have 
\begin{align*}\min_{x \in W^*}\langle \theta, x\rangle&\le  \EE_{\pi^*}[\langle \theta, r_i^*(x,y)\rangle-\sum_{i=1}^m \theta_i\beta \DD_{\mathrm{KL}}(\pi^*\|\pi_{\mathrm{ref}})] + D(\pi^*) = \|\theta\|_1\cdot J(r_1^*, \cdots, r_m^*, \frac{\theta}{\|\theta\|_1}, W^*, \pi^*) +  D(\pi^*)\\&\le\sqrt{m}\cdot J(r_1^*, \cdots, r_m^*, \frac{\theta}{\|\theta\|_1}, W^*, \pi^*) +  D(\pi^*) \end{align*}   
\end{lemma}
\begin{proof}
    By the definition of $D(\pi^*) = d(S(\pi^*), W^*)$, we know that there exists a vector $p$ with $S(\pi^*) + p \in W^*$ and $\|p\|_2 = D(\pi^*).$
Then we can have $$\min_{x \in W^*}\langle \theta, x\rangle \le \langle \theta, S(\pi^*) + p\rangle\le \EE_{\pi^*}[\langle \theta, r_i^*(x,y)\rangle-\sum_{i=1}^m \theta_i\beta \DD_{\mathrm{KL}}(\pi^*\|\pi_{\mathrm{ref}})] + D(\pi^*).$$
The last inequality holds because of $\|\theta\|_1 \le \sqrt{m}$.\end{proof}

We can first bound the regret by 
\begin{align}
    &D(\tilde{\pi}^T) - D(\pi^*) \nonumber\\ &= d(W^*, \EE_{\tilde{\pi}^T}[r^*(x,y)]-\beta \DD_{\mathrm{KL}}(\tilde{\pi}^T\| \pi_{\mathrm{ref}})) - D(\pi^*)\nonumber\\
    &\le d\left(W^*, \EE_{\tilde{\pi}^T}[r^*(x,y)]-\frac{\beta}{T} \sum_{t=1}^T \DD_{\mathrm{KL}}(\pi^t\|\pi_{\mathrm{ref}})\right) - D(\pi^*)\nonumber\\
    &= \underbrace{d\left(W^*, \EE_{\tilde{\pi}^T}[r^*(x,y)]-\frac{\beta}{T} \sum_{t=1}^T \DD_{\mathrm{KL}}(\pi^t\|\pi_{\mathrm{ref}})\right) - d\left(W^*, \frac{1}{T}\sum_{t=1}^T\EE_{\pi^t}[\hat{r}^t(x,y)]-\frac{\beta}{T} \sum_{t=1}^T \DD_{\mathrm{KL}}(\pi^t\|\pi_{\mathrm{ref}})\right)}_{\textrm{(A)}} \nonumber\\
    &\qquad + d\left(W^*, \frac{1}{T}\sum_{t=1}^T \EE_{\pi^t}[\hat{r}^t(x,y)]-\frac{\beta}{T} \sum_{t=1}^T \DD_{\mathrm{KL}}(\pi^t\|\pi_{\mathrm{ref}})\right) - D(\pi^*)\nonumber\\
    &=\mathrm{(A)}
    + d\left(W^*, \overline{V}^T\right) - D(\pi^*).\label{eq:finalregret_key}
\end{align}
The inequality uses the fact that 
$$\DD_{\mathrm{KL}}(\tilde{\pi}\| \pi_{\mathrm{ref}}) \le \frac{1}{T}\left(\sum_{t=1}^T\DD_{\mathrm{KL}}(\pi^t \| \pi_{\mathrm{ref}})\right).$$

Recall that $$D(\pi) = d(W^*, \EE_{\pi^t}[r(x,y)]-\beta \DD_{\mathrm{KL}}(\pi\|\pi_{\mathrm{ref}}))$$ and $\pi^* = \min_\pi D(\pi^*)$. Now, by Lemma \ref{lemma:approach}, for each $\theta \in \RR^m$ with $\|\theta\|_1 \le 1$, we have 
$$\min_{x \in W^*}\langle \theta, x\rangle\le  \EE_{\pi^*}[\langle \theta, r_i^*(x,y)\rangle-\sum_{i=1}^m \theta_i\beta \DD_{\mathrm{KL}}(\pi^*\|\pi_{\mathrm{ref}})] + D(\pi^*) = J(r_1^*, \cdots, r_m^*, \theta, W^*, \pi^*) +  D(\pi^*).$$

Denote $V^t \in \RR^m$ with $(V^t)_i = \EE_{\pi^t}[\hat{r}_i^t(x,y) - \beta\DD_{\mathrm{KL}}(\pi^t \| \pi_{\mathrm{ref}})]$, and $\frac{1}{t}\overline{V}^t = \sum_{i=1}^t V^i$. We have 
\begin{align*}
    d(\overline{V}^T, W^*)^2 &= \|\overline{V}^T - \Pi_{W^*} (\overline{V}^T)\|^2 \\
    &\le \|\overline{V}^T - \Pi_{W^*} (\overline{V}^{T-1})\|^2\\
    & = \left(\frac{T-1}{T}\right)^2 d(\overline{V}^{T-1}, W^*)^2 + \frac{1}{T^2} \|V^T - \Pi_{W^*}(\overline{V}^{T-1})\|^2 \\
    &\qquad + \frac{2(T-1)}{T^2}(\overline{V}^{T-1} - \Pi_{W^*} (\overline{V}^{T-1}))\cdot (V^T - \Pi_{W^*}(\overline{V}^{T-1})).\end{align*}
First, based on the definition of $W^*$, it is easy to show that $d^t \succeq 0.$
$\pi^t$ is the optimal policy such that $$\EE_{\pi^t}[\langle d^t, \hat{r}(x,y) \rangle - \sum_{i=1}^m d^t_i\beta \DD_{\mathrm{KL}}(\pi^t \|\pi_{\mathrm{ref}})] \ge  \EE_{\pi_{\mathrm{ref}}}[\langle d^t , \hat{r}(x,y) \rangle]\ge 0,$$
thus $(\sum_{i=1}^m d_i^t)\cdot \beta \DD_{\mathrm{KL}}(\pi^t \| \pi_{\mathrm{ref}}) \le \EE_{\pi^t}[d^t\cdot  \hat{r}(x,y)] \le B$. Hence, given $d^t \succeq 0$ and $\|d^t\|_2 =1,$
\begin{align*}
    \beta \DD_{\mathrm{KL}}(\pi^t \| \pi_{\mathrm{ref}}) \le \frac{B}{\sum_{i=1}^m d_i^t}\le B.
\end{align*}
we have $|(V^t)_i| \le B$ and 
\begin{align*}
    \|V^T - \Pi_{W^*}(\overline{V}^{T-1})\|^2 \le B^2m.
\end{align*}
Thus by iteration we can have 
\begin{align*}
    T^2 d(\overline{V}^T, W^*)^2 \le T\cdot B^2m + \sum_{t=1}^T 2(t-1) (\overline{V}^{t-1}-\Pi_{W^*}(\overline{V}^{t-1}))\cdot (V^t - \Pi_{W^*}(\overline{V}^{t-1})).
\end{align*}
Now, by the definition of $d^t$, we have
\begin{align*}
    (\overline{V}^{t-1}-\Pi_{W^*}(\overline{V}^{t-1}))\cdot (V^t - \Pi_{W^*}(\overline{V}^{t-1})=
    d(\overline{V}^{t-1}, W^*)\cdot d^t\cdot (\Pi_{W^*}(\overline{V}^{t-1}) - V^t).
    \end{align*}
    Then, we prove the following lemma.
    \begin{lemma}\label{lemma:projection}
        $\min_{x\in W^*} \langle d^t, x\rangle = d^t \cdot \Pi_{W^*}(\overline{V}^{t-1}).$
    \end{lemma}
    \begin{proof}
     In fact, we only need to prove that for any $x \in W^*$, $\langle \overline{V}^{t-1}-\Pi_{W^*}(\overline{V}^{t-1}), x-\Pi_{W^*}(\overline{V}^{t-1})\rangle \le 0$. Suppose there exists $x \in W^*$ such that $\langle \overline{V}^{t-1}-\Pi_{W^*}(\overline{V}^{t-1}), x-\Pi_{W^*}(\overline{V}^{t-1})\rangle >0$, then since $W^*$ is a convex set, for any $\lambda \in (0,1)$, we have $x_\lambda = \lambda x + (1-\lambda) \Pi_{W^*}(\overline{V}^{t-1}) \in W^*$. Consider the line $$\Pi_{W^*}(\overline{V}^{t-1}) + t \frac{\Pi_{W^*}(\overline{V}^{t-1})-x}{\|\Pi_{W^*}(\overline{V}^{t-1})-x\|},\  \  t \in \RR.$$ Also, we consider the  projection of $\overline{V}^{t-1}$ on this line, and denote it as $p$. Then we can get 
    $$0<\langle\overline{V}^{t-1}-\Pi_{W^*}(\overline{V}^{t-1})- p + p, x -\Pi_{W^*}(\overline{V}^{t-1})\rangle  = \langle p-\Pi_{W^*}(\overline{V}^{t-1}),  x -\Pi_{W^*}(\overline{V}^{t-1})\rangle. $$
    Hence when $\lambda \to 0$, $x_\lambda$ is between $p$  and $\Pi_{W^*}(\overline{V}^{t-1})$. Also, $$\|\overline{V}^{t-1}-x_\lambda \|^2 = \|\overline{V}-p\|^2 + \|p-x_\lambda\|^2 \le \|\overline{V}-p\|^2 + \|p-\Pi_{W^*}(\overline{V}^{t-1})\|^2 \le \|\Pi_{W^*}(\overline{V}^{t-1})-d^t \|^2,$$ which contradicts the selection of $\Pi_{W^*}(\overline{V}^{t-1}).$
 \end{proof}
Now, by Lemma \ref{lemma:approach} and Lemma \ref{lemma:projection}, we can get 
\begin{align*}
    d^t \cdot \Pi_{W^*}(\overline{V}^{t-1}) \le J(r_1^*, \cdots, r_m^*, d^t, \pi^*) + D(\pi^*).
\end{align*}
Then, since we define $\overline{d^t} = d^t/\|d^t\|_1$,  we can continue the analysis by
    \begin{align*}
    &(\overline{V}^{t-1}-\Pi_{W^*}(\overline{V}^{t-1}))\cdot (V^t - \Pi_{W^*}(\overline{V}^{t-1}))\\&=d(\overline{V}^{t-1}, W^*)\cdot \left( \|d^t\|_1J(r_1^*, r_2^*, \cdots, r_m^*, \overline{d^t}, \pi^*) + D(\pi^*) - d^t\cdot  V^t\right)\\
    & = d(\overline{V}^{t-1}, W^*)\cdot (\|d^t\|_1\cdot \left(J(\hat{r}_1^*, \cdots, \hat{r}_m^*, \overline{d^t}, \pi^*) - J(\hat{r}_1, \cdots, \hat{r}_m, \overline{d^t}, \pi^t)\right) + D(\pi^*))\\
    & = d(\overline{V}^{t-1}, W^*) \cdot (\|d^t\|_1\cdot (\eta \sum_{i=1}^m L_i(\theta^t) - \eta \sum_{i=1}^m L_i(\theta^*) )+ D(\pi^*)).
\end{align*}
Thus we can get 
\begin{align*}
    T d(\overline{V}^T, W^*)^2 &\le   B^2m +
    \sum_{t=1}^T \frac{2(t-1)}{T}d(\overline{V}^{t-1}, W^*) \cdot (\eta \|d^t\|_1\sum_{i=1}^m L_i(\theta^t) - \eta \|d^t\|_1\sum_{i=1}^m L_i(\theta^*) + D(\pi^*)).
\end{align*}

Now we use induction method to show that $$d(\overline{V}^t, W^*) \le D(\pi^*) +  \frac{\eta}{T}\sum_{t=1}^T \|d^t\|_1\sum_{i=1}^m (L_i(\theta^t) - L_i(\theta^*)) + 2Bm/\sqrt{t}.$$ When $t = 1$, the inequality holds by 
\begin{align*}
    \|d(\overline{V}^1, W^*)-D(\pi^*)\|\le d(\overline{V}^1, S(\pi^*)) \le 2B.
\end{align*}
Denote $A_j=\eta\cdot\|d^j\|_1\cdot  (\sum_{i=1}^m (L_i(\theta^*) - L_i(\theta^j)))$ and $S_t = \sum_{j=1}^{t}A_j$, then for all $t \in [T-1],$  
suppose we have 
\begin{align*}
    d(\overline{V}^{t-1}, W^*) \le D(\pi^*) + \frac{1}{t-1} S_{t-1} + 2B\left(\frac{\sqrt{m}}{\sqrt{t-1}}\right).
\end{align*}
Then we substitute these induction hypothesis into the recursion inequality and get 
\begin{align*}
    &T d(\overline{V}^T, W^*)^2\\&\le B^2m + \sum_{t=1}^T \frac{2(t-1)}{T} \left(D(\pi^*) + \frac{1}{t-1}S_{t-1}+ 2B\left(\frac{\sqrt{m}}{\sqrt{t-1}}\right)\right)\left(A_t + D(\pi^*)\right)\\
    &\le B^2m + \sum_{t=1}^T \left( \frac{2(t-1)}{T}D(\pi^*) + \frac{1}{T}S_{t-1} + 2B\left(\frac{2\sqrt{m}\sqrt{t-1}}{T}\right)\right)\left(A_t + D(\pi^*)\right)\\
    &= B^2m + (T-1)D(\pi^*)^2 + \sum_{t=1}^T \frac{1}{T}S_{t-1}A_t + \sum_{t=1}^T \left(\frac{1}{T} S_{t-1} + \frac{2(t-1)}{T}A_t\right) D(\pi^*)\\
    &\qquad + \sum_{t=1}^T 2B\left(\frac{2\sqrt{m}\sqrt{t-1}}{T}\right)(A_t + D(\pi^*))\\
    & \le  B^2m + (T-1)D(\pi^*)^2 + \frac{1}{T} S_T^2 + \sum_{t=1}^T D(\pi^*) \cdot \left(\frac{T+t-1}{T}A_t\right)  + 2\sqrt{m}\cdot 2B\sqrt{T} D(\pi^*) + (2\sqrt{m}\cdot 2B/\sqrt{T}) S_T\\
    &\le B^2m + (T-1)D(\pi^*)^2 + \frac{1}{T} S_T^2 +  D(\pi^*) \cdot \left(2S_T\right) + 2\sqrt{m}\cdot 2B\sqrt{T} D(\pi^*) + 2\sqrt{m}\cdot 2B/\sqrt{T} S_T\\
    &\le T \cdot (2B\sqrt{m}/\sqrt{T} + D(\pi^*) + \frac{1}{T}S_T)^2.
\end{align*}

Thus we have $$d(\overline{V}^T, W^*) \le D(\pi^*) + \frac{\eta}{T}\sum_{t=1}^T \|d^t\|_1\sum_{i=1}^m (L_i(\theta^t) - L_i(\theta^*)) + \frac{2B\sqrt{m}}{\sqrt{T}}.$$

Now we derive the final regret.
By inequality Eq.~\eqref{eq:finalregret_key}, we can get
\begin{align*}D(\tilde{\pi}^T) - D(\pi^*)&\le \mathrm{(A)} + d(W^*, \overline{V}^T) - D(\pi^*)\\
&\le \mathrm{(A)} +  \underbrace{\frac{\eta}{T}\sum_{t=1}^T \|d^t\|_1\sum_{i=1}^m (L_i(\theta^t) - L_i(\theta^*))}_{\text{(B)}} + \frac{2B\sqrt{m}}{\sqrt{T}}.\end{align*}
Now we consider the error term $\text{(A)}$, which represents the approximation error of the reward function. Now we have 
\begin{align*}
    \text{(A)}&\le \frac{1}{T}\sum_{t=1}^T \EE_{\pi^t}\left[\sum_{i=1}^m  |\hat{r}_i^t(x,y) - r_i^*(x,y)|\right]\\
    & = \frac{1}{T}\sum_{t=1}^T \sum_{i=1}^m\EE_{\pi^t}\left[\|\phi_i(x,y)\|_{(\Sigma_{\cD_i}+\lambda I)^{-1}} \|\theta_i^t - \theta_i^*\|_{\Sigma_{\cD_i}+\lambda I } \right].
\end{align*}
Similar to \citep{cen2024value}, since $(r^\theta,\pi^\theta)$ can be formulated as a saddle point of the objective $J(r,d,\pi) + \sum_{i=1}^m \eta L_i(\theta_i)$ for any direction $d \in (\RR^{+})^m$, we have 
\begin{align*}
    \eta \nabla_{\theta_i} L_i(\theta_i) + d_i \EE_{x\sim \rho, y\sim \pi^{\theta}}[\phi_i(x,y)] + \lambda_1\EE_{x\sim \rho, y\sim \pi_{\mathrm{base}}}[\phi_i(x,y)] = 0.
\end{align*}
Also, denote $\theta_{\mathrm{MLE}}=\argmin_{\theta \in \Theta} \sum_{i=1}^m \eta L_i(\theta_i), $ we have 
\begin{align*}
    \eta \nabla_{\theta_i} L_i(\theta_{i,\mathrm{MLE}}) + \lambda_2 \EE_{x\sim \rho, y\sim \pi_{\mathrm{base}}}[\phi_i(x,y)] = 0.
\end{align*}

Follow the same derivation in \citep{cen2024value}, we can get 
\begin{align*}
    \|\theta_i^t - \theta_{i,\mathrm{MLE}}\|_{\Sigma_{\cD_i + \lambda  I}} &\le \frac{d_i}{\eta} \cdot \frac{(3+e^{B'})4(\lambda_{\min}(\Sigma_{\cD_i}) + \lambda)^{-1}}{M} + 2\sqrt{
    \lambda (B')^2
    }\\
    &\le \frac{(3+e^{B'})4(\lambda_{\min}(\Sigma_{\cD_i}) + \lambda)^{-1}}{\sqrt{M}} + 2\sqrt{
    \lambda (B')^2},
\end{align*}
where $B'$ is the upper bound of norm of $\theta,$ i.e. $\max_{\theta \in \Theta}\|\theta\|_2 \le B'.$

Now we recall the Lemma 3.1 in \citep{zhu2023principled}, which bounds the true parameter and the MLE parameter. 
\begin{lemma}[Lemma 3.1 in \citep{zhu2023principled}] $\lambda>0$ is a positive constant. For $\delta \in (0,1)$, with probability at least $1-\delta$, we will have
    \begin{align*}
    \|\theta_i^* - \theta_{i,\mathrm{MLE}}\|_{\Sigma_{\cD_i} + \lambda I } \le \cO\left((3+e^{B'})\sqrt{\frac{d+\log(1/\delta)}{M}} + \sqrt{\lambda (B')^2}\right).
\end{align*}
Also, $L_i(\theta)$ is a convex function. In fact, 
\begin{align*}
    \frac{1}{3+e^{B'}}\Sigma_{\cD_i} \preceq \frac{1}{M}\nabla_\theta^2 L_i(\theta) \preceq \frac{1}{4}\Sigma_{\cD_i}.
\end{align*}
\end{lemma}
Hence, we get
\begin{align*}
    \text{(A)} &\le \frac{1}{T}\sum_{t=1}^T \sum_{i=1}^m\EE_{\pi^t}\left[\|\phi_i(x,y)\|_{(\Sigma_{\cD_i}+\lambda I)^{-1}} \|\theta_i^t - \theta_i^*\|_{\Sigma_{\cD_i}+\lambda I } \right]\\
    &\le \frac{1}{T}\sum_{t=1}^T \sum_{i=1}^m \|\EE_{\pi^t}\phi_i(x,y)\|_{(\Sigma_{\cD_i}+\lambda I)^{-1}} \cdot \cO\left(\frac{(3+e^{B'})4(\lambda_{\min}(\Sigma_{\cD_i}) + \lambda)^{-1}\sqrt{d+\log(1/\delta)}}{\sqrt{M}}+\sqrt{\lambda (B')^2}\right)\\
    &\le \widetilde{\cO}\left(\frac{m(3+e^{B'})4(\lambda_{\min}(\Sigma_{\cD_i}) + \lambda)^{-2}\sqrt{d+\log(1/\delta)}}{\sqrt{M}}+\sqrt{\lambda (B')^2}\right).
\end{align*}
The notation $\widetilde{\cO}(\cdot)$ hides all the logarithm term like $\log(1/\delta)$.

Now we consider the term $\text{(B)}.$ First, based on the convexity of $L_i(\theta)$, we have 
\begin{align*}
    L_i(\theta_i^t) - L_i(\theta_{i,\mathrm{MLE}}) & \le \langle \nabla_\theta L_i(\theta^t), \theta_i^t - \theta_{i,\mathrm{MLE}}\rangle \\
    & = \frac{1}{\eta} \langle -d_i \EE_{x\sim \rho, y\sim \pi^\theta}[\phi_i(x,y)] - \lambda_1 \EE_{x\sim \rho, y\sim \pi_{\mathrm{base}}}[\phi_i(x,y)], \theta_i^t - \theta_{i,\mathrm{MLE}}\rangle\\
     & = \frac{d_i}{\eta}\langle -\EE_{x\sim \rho, y\sim \pi^\theta}[\phi_i(x,y)] - \EE_{x\sim \rho, y\sim \pi_{\mathrm{base}}}[\phi_i(x,y)], \theta_i^t - \theta_{i,\mathrm{MLE}}  \rangle \\
     &\le \frac{d_i}{\eta} \|\EE_{x\sim \rho, y\sim \pi^\theta}[\phi_i(x,y)]-\EE_{x\sim \rho, y\sim \pi_{\mathrm{base}}}[\phi_i(x,y)]\|_{(\Sigma_{\cD_i} + \lambda I )^{-1}}\|\theta_i^t - \theta_{i,\mathrm{MLE}}\|_{\Sigma_{\cD_i} + \lambda I }\\
     &\le \frac{2d_i}{\eta}\cdot (\lambda_{\min}(\Sigma_{\cD_i})+\lambda)^{-1}\cdot \|\theta_i^t - \theta_{i,\mathrm{MLE}}\|_{\Sigma_{\cD_i} + \lambda I }\\
     &\le \cO\left( \frac{(3+e^{B'})(\lambda_{\min}(\Sigma_{\cD_i}) + \lambda)^{-2}}{\sqrt{M}} + \frac{4}{\eta}\sqrt{\lambda (B')^2}\cdot (\lambda_{\min}(\Sigma_{\cD_i})+\lambda )^{-1}\right).
\end{align*}
The last inequality uses the fact that $d_i \le 1.$
Also, with probability at least $1-\delta$, we have 
$$L_i(\theta_{i,\mathrm{MLE}}) - L_i(\theta^*) \le \tilde{\cO}(1).$$
Now sum over $t \in [T], $ we can get 
\begin{align*}
    \text{(B)}&\le \frac{\eta}{T}\sum_{t=1}^T\|d^t\|_1 \sum_{i=1}^m (L_i(\theta_i^t) - L_i(\theta^*))\\
    & \le \sqrt{m}\cdot \frac{m}{\sqrt{M}}\widetilde{\cO}\left(\frac{(3+e^{B'})(\min_i\lambda_{\min}(\Sigma_{\cD_i}) + \lambda)^{-2}}{\sqrt{M}} + \frac{4}{\eta}\sqrt{\lambda (B')^2}\cdot (\min_i\lambda_{\min}(\Sigma_{\cD_i})+\lambda )^{-1} + 1\right),
\end{align*}
where the last inequality uses the fact that $\eta = 1/\sqrt{M}$ and $\|d^t\|_1 \le \sqrt{m}$.
Hence, we have 
\begin{small}
\begin{align*}
    &D(\tilde{\pi}^T)-D(\pi^*) \\&\le \text{(A)} + \text{(B)} + \frac{2Bm}{\sqrt{T}}\\
    &\le \widetilde{\cO}\left(\frac{m^{3/2}(3+e^{B'}) (\min_i\lambda_{\min}(\Sigma_{\cD_i} + \lambda)^{-2}\sqrt{d+\log(1/\delta)}}{\sqrt{M}} + \frac{4m^{3/2}}{\eta\sqrt{M}}B'\sqrt{\lambda}\cdot (\min_i\lambda_{\min}(\Sigma_{\cD_i}) + \lambda)^{-1} + \frac{m^{3/2}}{\sqrt{M}}+\frac{B\sqrt{m}}{\sqrt{T}}\right)\\
    &\le \widetilde{\cO}\left(\frac{m^{3/2}(3+e^{B'}) (\min_i\lambda_{\min}(\Sigma_{\cD_i} + \lambda)^{-2}\sqrt{d+\log(1/\delta)}}{\sqrt{M}} + \frac{4m^{3/2}B'}{\sqrt{M}}\cdot (\min_i \lambda_{\min}(\Sigma_{\cD_i}) + \lambda)^{-1} + \frac{m^{3/2}}{\sqrt{M}} + \frac{B\sqrt{m}}{\sqrt{T}}\right)\\
    & = \frac{m}{\sqrt{M}}\cdot \widetilde{\cO}\left(\text{poly}\left(e^{B'}, \min_i\lambda_{\min}(\Sigma_{\cD_i})^{-1}, \sqrt{d+\log(1/\delta)},B'\right)\right)+ \widetilde{\cO}\left(\frac{B\sqrt{m}}{\sqrt{T}}\right).
\end{align*}
\end{small}
The last step is because $\eta = 1/\sqrt{M}, \lambda = 1/M$. Hence we complete the proof. \qed
\subsection{Proof of Theorem \ref{thm:malfareoffline}}\label{sec: malfare_proof_offline}

\begin{proof}
The main proof framework is similar to Theorem \ref{thm:offline}. The difference lies in the approach to deal with the aggregated $p$-norm of the distance. 
\begin{align}
    \sum_{n=1}^N \zeta_nd^{2q}(\overline{V}^t, W_n^*) & = \sum_{n=1}^N\zeta_n \|\overline{V}^T - \Pi_{W_n^*}(\overline{V}^T)\|^{2q}\\
    & \le \sum_{n=1}^N \zeta_n\|\overline{V}^T - \Pi_{W_n^*}(\overline{V}^{T-1})\|^{2q}\\
    & = \sum_{n=1}^N \zeta_n\left\|\frac{T-1}{T}(\overline{V}^{T-1}-\Pi_{W_n^*}(\overline{V}^{T-1})) + \frac{1}{T}(V^T - \Pi_{W_n^*}(\overline{V}^{T-1}))\right\|^{2q}. \label{ineq:pnorm}
\end{align}
For the vector $x_n,y_n \in \RR^m$ with $x_n=(T-1)(\overline{V}^{T-1}-\Pi_{W_n^*}(\overline{V}^{T-1})), y_n= (V^T-\Pi_{W_n^*}(\overline{V}^{T-1}))$  we know $\|x_n\| \le 2TB\sqrt{m}, \|y_n\| \le 2B\sqrt{m}$. Hence, since $q>1$, we have
\begin{align*}\|x_n+y_n\|^{2q} &\le (\|x_n\|^2 + \|y_n\|^2 + 2\langle x_n,y_n\rangle)^{q}\\&\le \|x_n\|^{2q} + 2\langle x_n,y_n\rangle \|x_n\|^{2q-2} + 3^q\cdot T^{2q-2}(2B)^{2q}m^q\end{align*}
 We can further bound the inequality \eqref{ineq:pnorm} as 
\begin{align}
    T^{2q} \sum_{n=1}^N \zeta_n d^{2q}(\overline{V}^{T}, W_n^*) &\le  \sum_{n=1}^N \zeta_n\|x_n+y_n\|^{2q} \\&\le (T-1)^{2q}\sum_{n=1}^N \zeta_nd^{2q}(\overline{V}^{T-1}, W_n^*) + 12^q\cdot T^{2q-2}B^{2q}m^q\\
    &\qquad  + 2 (T-1)\sum_{n=1}^N \zeta_n(\overline{V}^{T-1}-\Pi_{W_n^*}(\overline{V}^{T-1}))(V^T-\Pi_{W_n^*}(\overline{V}^{T-1}))\|x_n\|^{2q-2}.
\end{align}
Then since $\|x_n\| = (T-1) d(\overline{V}^{T-1}, W_n^*)$, we can finally get 
\begin{align}
T^{2q} \sum_{n=1}^N \zeta_nd^{2q}(\overline{V}^{T}, W_n^*) &\le (T-1)^{2q} \sum_{n=1}^N \zeta_nd^{2q}(\overline{V}^{T-1}, W_n^*) + 12^qT^{2q-2}B^{2q}m^q\\
&\qquad + 2(T-1)^{2q-1}\sum_{n=1}^N \zeta_n(\overline{V}^{T-1}-\Pi_{W_n^*}(\overline{V}^{T-1}))(V^T-\Pi_{W_n^*}(\overline{V}^{T-1}))d^{2q-2}(\overline{V}^{T-1}, W_n^*).\label{eq:recursion before}
\end{align}
Hence by the recursion, we have 
\begin{align}
    T^{2q} \sum_{n=1}^N \zeta_n d^{2q}(\overline{V}^{T}, W_n^*) &\le 12^qT^{2q-1}B^{2q}m^q \nonumber\\&\qquad + 2(t-1)^{2q-1}\sum_{t=1}^T\sum_{n=1}^N \zeta_n(\overline{V}^{t-1}-\Pi_{W_n^*}(\overline{V}^{t-1}))(V^t-\Pi_{W_n^*}(\overline{V}^{t-1}))d^{2q-2}(\overline{V}^{t-1}, W_n^*).\nonumber
\end{align}
Now the last term at the right side can be further bounded by 
\begin{small}
\begin{align}
    &\sum_{t=1}^T \sum_{n=1}^N \zeta_n (t-1)^{2q-1}(\overline{V}^{t-1}-\Pi_{W_n^*}(\overline{V}^{t-1}))(V^t-\Pi_{W_n^*}(\overline{V}^{t-1}))d^{2q-2}(\overline{V}^{t-1}, W_n^*)\nonumber\\
    &\le \sum_{t=1}^T \sum_{n=1}^N \zeta_n(t-1)^{2q-1}d^{2q-1}(\overline{V}^{t-1}, W_n^*)d_n^t\cdot (\Pi_{W_n^*}(\overline{V}^{t-1}) - V^t)\nonumber\\
    &\le  \sum_{t=1}^T\sum_{n=1}^N\zeta_n (t-1)^{2q-1}d^{2q-1}(\overline{V}^{t-1}, W_n^*)\left(\|d_n^t\|_1 J(r_1^*, \cdots, r_m^*, \overline{d_n^t}, \pi^*)+d(S(\pi^*), W_n^*) - \|d_n^t\|_1 J(\hat{r}_1, \cdots, \hat
    {r}_m, \overline{d_n^t}, \pi^t)\right)\nonumber\\
    &\le \sum_{t=1}^T(t-1)^{2q-1}\left(\sum_{n=1}^N\zeta_n d^{2q}(\overline{V}^{t-1}, W_n^*)\right)^{\frac{2q-1}{2q}}  \nonumber\\&\qquad \cdot \left(\|d^t\|_1 J(r_1^*, \cdots, r_m^*, \overline{d^t}, \pi^*)+ \sqrt[2q]{\sum_{n=1}^N\zeta_n d^{2q}(S(\pi^*), W_n^*)}- \|d^t\|_1 J(\hat{r}_1, \cdots, \hat
    {r}_m, \overline{d^t}, \pi^t)\right)\label{cauchy explain}\\
    &\le \sum_{t=1}^T(t-1)^{2q-1}\left(\sum_{n=1}^N \zeta_n d^{2q}(\overline{V}^{t-1}, W_n^*)\right)^{\frac{2q-1}{2q}} \cdot \left(D_{q}(\pi^*) + \eta \|d^t\|_1\left(\sum_{i=1}^m L_i^t(\theta^*) - \eta \sum_{i=1}^m L_i^t(\theta^t)\right)\right).\label{last term bound}
\end{align}
\end{small}
The inequality Eq.~\eqref{cauchy explain} derives from the definition of $d^t$ in Eq.~\eqref{eq:dir malfare} and Cauchy's inequality.
Let $S_T = \sqrt[2q]{\sum_{n=1}^N \zeta_n d^{2q}(\overline{V}^T, W_n^*)}$, then we can get 
\begin{align*}
    T S_T^{2q}\le 12^q\cdot B^{2q}m^q +  \sum_{t=1}^T \frac{2(t-1)^{2q-1}}{T^{2q-1}}S_{t-1}^{2q-1}\cdot \left(D_q(\pi^*) + \eta \|d^t\|_1\cdot \left( \sum_{i=1}^m L_i^t(\theta^*) -  \sum_{i=1}^m L_i^t(\theta^t)\right)\right).
\end{align*}
Define $A_t =  D_q(\pi^*) + \eta \|d^t\|_1\cdot \left(\sum_{i=1}^m L_i^t(\theta^*) - \sum_{i=1}^m L_i^t(\theta^t)\right),$ then we use the induction to show that there exists a constant $C_q$ such that 
\begin{align}
    S_{t} \le  \left(\frac{1}{t}\sum_{s=1}^t A_s + C_qT^{-1/2q}\right). \nonumber
\end{align}
In fact, it holds when $t = 1$. Now suppose it holds for $t=1,2,\cdots, T-1$, we have 
\begin{align*}
    S_T^{2q}&\le 12^q\cdot B^{2q}m^q/T +  \sum_{t=1}^T \frac{2(t-1)^{2q-1}}{T^{2q-1}}S_{t-1}^{2q-1}\cdot \frac{A_t}{T}\\
    &\le 12^q\cdot B^{2q}m^q/T + 2\sum_{t=1}^T \left(\frac{1}{T}\sum_{s=1}^{t-1}A_s + C_qT^{-1/2q}\right)^{2q-1} \cdot \frac{A_t}{T}\\
    &\le 12^q\cdot B^{2q}m^q/T + 2\sum_{t=1}^T \sum_{k=0}^{2q-1}\binom{2q-1}{k}\frac{1}{T}\left(\sum_{s=1}^{t-1}A_s\right)^{k+1}\cdot \frac{A_t}{T}\cdot (C_q)^{2q-1-k}T^{-\frac{2q-1-k}{2q}}\\
    &\le 12^q\cdot B^{2q}m^q/T  + \sum_{k=0}^{2q-1}\binom{2q-1}{k}(C_q)^{2q-1-k}T^{-\frac{2q-1-k}{2q}}\left(\frac{1}{T}\sum_{t=1}^T A_t\right)^{k+1}.
\end{align*}
Now we choose $C_q = (12^q\cdot B^{2q}m^q)^{\frac{1}{2q}} = \sqrt{12B^2m}$, then we have 
\begin{align*}
    S_T^{2q} &\le \cO(C_q^{2q}/T)  + \sum_{k=0}^{2q-1}\binom{2q-1}{k}(C_q)^{2q-1-k}T^{-\frac{2q-1-k}{2q}}\left(\frac{1}{T}\sum_{t=1}^T A_t\right)^{k+1}\\
    &\le \cO(C_q^{2q}/T)  + \sum_{k=0}^{2q-1}\binom{2q}{k+1}(C_q)^{2q-1-k}T^{-\frac{2q-1-k}{2q}}\left(\frac{1}{T}\sum_{t=1}^T A_t\right)^{k+1}\\
    &\le \left(\frac{1}{T}\sum_{t=1}^T A_t + C_qT^{-1/2q} \right)^{2q}.
\end{align*}
which implies that 
\begin{align}
    S_T \le  \frac{1}{T}\sum_{s=1}^T A_s + C_qT^{-1/2q}=\frac{1}{T}\sum_{s=1}^T A_s + 4B\sqrt{m}\cdot T^{-1/2q}.
\end{align}
Hence we have 
\begin{align*}
    \sqrt[2q]{\sum_{n=1}^N\zeta_n d^{2q}(\overline{V}^T , W_n^*)} - D_q(\pi^*)\le \frac{\eta}{T}\sum_{t=1}^T\|d^t\|_1\sum_{i=1}^m (L_i^t(\theta^*) -L_i^t(\theta^t)) + \widetilde{\cO}(B\sqrt{m} T^{-1/2q}).
\end{align*}
Now we derive the final regret. We can see 
\begin{align}
    &D_q(\tilde{\pi}^T) - D_q(\pi^*)\\
    &=\sqrt[2q]{\sum_{n=1}^N\zeta_n d^{2q}(S(\tilde{\pi}^T), W_n^*)}  - D_q(\pi^*)\nonumber\\
    &=\sqrt[2q]{\sum_{n=1}^N\zeta_n d^{2q}(W_n^*, \EE_{\tilde{\pi}^T}\left[r^*(x,y) \right]- \frac{\beta}{T}\sum_{t=1}^T \DD_{\mathrm{KL}}(\pi^t\|\pi_{\mathrm{ref}}))\cdot \mathbf{1}^m} \nonumber\\
    &\qquad - \sqrt[2q]{\sum_{n=1}^N\zeta_n d^{2q}(W_n^*, \frac{1}{T}\sum_{t=1}^T \EE_{\pi^t}\left[r^{\theta^t}(x,y) \right]- \frac{\beta}{T}\sum_{t=1}^T \DD_{\mathrm{KL}}(\pi^t\|\pi_{\mathrm{ref}}))\cdot \mathbf{1}^m}+ \sqrt[2q]{\sum_{n=1}^N\zeta_n d^{2q}(W_n^*, \overline{V}^T)} - D_q(\pi^*)\nonumber\\
    &\le \underbrace{\sqrt[2q]{\sum_{n=1}^N\zeta_n \left(d(W_n^*, \EE_{\tilde{\pi}^T}\left[r^*(x,y) \right]- \frac{\beta}{T}\sum_{t=1}^T \DD_{\mathrm{KL}}(\pi^t\|\pi_{\mathrm{ref}})\cdot \mathbf{1}^m)-d(W_n^*, \frac{1}{T}\sum_{t=1}^T\EE_{\pi^t}\left[\hat{r}^t(x,y) \right]- \frac{\beta}{T}\sum_{t=1}^T \DD_{\mathrm{KL}}(\pi^t\|\pi_{\mathrm{ref}})\cdot \mathbf{1}^m)\right)^{2q}}}_{\text{(A)}}\nonumber\\&\qquad + \underbrace{\frac{\eta}{T}\sum_{t=1}^T\|d^t\|_1\sum_{i=1}^m (L_i^t(\theta^*) -L_i^t(\theta^t)) }_{\text{(B)}} + \widetilde{\cO}(B\sqrt{m} T^{-1/2q}).\label{final result 2}
\end{align}
The last inequality uses the triangle inequality for $2q$-norm. Now also note that \begin{align*}d(W_n^*, &\EE_{\tilde{\pi}^T}\left[r^*(x,y) \right]- \frac{\beta}{T}\sum_{t=1}^T \DD_{\mathrm{KL}}(\pi^t\|\pi_{\mathrm{ref}})\cdot \mathbf{1}^m)-d(W_n^*, \frac{1}{T}\sum_{t=1}^T\EE_{\pi^t}\left[r^*(x,y) \right]- \frac{\beta}{T}\sum_{t=1}^T \DD_{\mathrm{KL}}(\pi^t\|\pi_{\mathrm{ref}})\cdot \mathbf{1}^m) \\&\qquad \le \frac{1}{T}\sum_{t=1}^T\sum_{i=1}^m \EE_{\pi^t}|r_i^*(x,y) - \hat{r}_i^t(x,y)|,\end{align*} we have 
\begin{align*}
    \text{(A)}\le \sqrt[2q]{\left(\sum_{n=1}^N \zeta_n\right) \left(\sum_{i=1}^m\Bigg|\frac{1}{T}\sum_{t=1}^T \EE_{\pi^t}[r_i^*(x,y) -\hat{r}_i^t(x,y)]\Bigg|\right)^{2q}}  = \sum_{i=1}^m\Bigg|\frac{1}{T}\sum_{t=1}^T \EE_{\pi^t}[r_i^*(x,y) -\hat{r}_i^t(x,y)]\Bigg|
\end{align*}
Now follow the same proof as Theorem \ref{thm:offline}, 
\begin{align*}
    \text{(A)}\le \widetilde{\cO}\left(\frac{m(3+e^{B'})4(\lambda_{\min}(\Sigma_{\cD_i}) + \lambda)^{-2}\sqrt{d+\log(1/\delta)}}{\sqrt{M}}+\sqrt{\lambda (B')^2}\right),
\end{align*}
and 
\begin{align*}
    &\text{(B)}\le \frac{\eta}{T}\sum_{t=1}^T\|d^t\|_1\sum_{i=1}^m (L_i^t(\theta^*) -L_i^t(\theta^t))\\&\qquad \le \frac{N\sqrt{m}\cdot m}{\sqrt{M}}\widetilde{\cO}\left(\frac{(3+e^{B'})(\min_i\lambda_{\min}(\Sigma_{\cD_i}) + \lambda)^{-2}}{\sqrt{M}} + \frac{4}{\eta}\sqrt{\lambda (B')^2}\cdot (\min_i\lambda_{\min}(\Sigma_{\cD_i})+\lambda )^{-1} + 1\right),
\end{align*}
where the last inequality we use the fact that 
\begin{align*}
    \|d^t\|_1 =  \left\|\sum_{n=1}^Nd_n^t\cdot \frac{\zeta_n\|W^{(n)}- \overline{V}^t\|_2^{2q-1}}{\left(\sum_{n=1}^N \zeta_n\|W^{(n)}- \overline{V}^t\|_2^{2q}\right)^{\frac{2q-1}{2q}}}\right\|_1\le \sum_{n=1}^N \|d_n^t\|_1\cdot \zeta_n^{1/2q} \le N\sqrt{m}.
\end{align*}
Combining the Eq.~\eqref{final result 2} and the upper bounds for (A) and (B), substitute into $\eta = 1/\sqrt{M}$ and $\lambda = 1/M$, we can complete the final proof.
\end{proof}

\subsection{Proof of Theorem \ref{thm:online}}
\begin{proof}
Recall that $V^t \in \RR^m$ with $(V^t)_i = \EE_{\pi^t}[r_i^{\theta_i^t}(x,y) - \beta \DD_{\mathrm{KL}}(\pi^t \|\pi_{\mathrm{ref}})]$, and $\overline{V}^t = \frac{1}{t}\sum_{i=1}^t V^i.$ We also define $W^0 = \{(0,0)\}.$ Since $W^t$ is the estimation of $W^*$ at round $t$, we have 
\begin{align}
    d(\overline{V}^T, W^T)^2 &= \|\overline{V}^T - \Pi_{W^T}(\overline{V}^T)\|^2\nonumber\\
    &\le \|\overline{V}^T - \Pi_{W^T}(\overline{V}^{T-1})\|^2\nonumber\\
    &\le \|\overline{V}^T - \Pi_{W^{T-1}}(\overline{V}^{T-1})\|^2 + \|\Pi_{W^T}(\overline{V}^{T-1}) - \Pi_{W^{T-1}}(\overline{V}^{T-1})\|^2\nonumber \\&\qquad + 2\langle \overline{V}^T - \Pi_{W^{T-1}}(\overline{V}^{T-1}), \Pi_{W^T}(\overline{V}^{T-1}) - \Pi_{W^{T-1}}(\overline{V}^{T-1})\rangle.\label{ineq:first}
\end{align}
Now by Lemma \ref{lemma:dis of proj}, since $\|V^{T-1}\|_\infty\le B$ is bounded, we have 
$$\|\Pi_{W^T}(\overline{V}^{T-1}) - \Pi_{W^{T-1}}(\overline{V}^{T-1})\|_2^2 \le 4d(\overline{V}^{T-1}, W^{T-1})d_{B_1}(W^T, W^{T-1}) + 2d^2_{B_1}(W^T, W^{T-1}).$$
Then we can get 
\begin{align*}
    \|\Pi_{W^T}(\overline{V}^{T-1}) - \Pi_{W^{T-1}}(\overline{V}^{T-1})\|_2\le 2\sqrt{d(\overline{V}^{T-1}, \Pi_{W^{T-1}}(\overline{V}^{T-1}))d_{B_1}(W^T, W^{T-1})}+\sqrt{2}d_{B_1}(W^T, W^{T-1}).
\end{align*}
Then the third term on the right side can be bounded by 
\begin{align*}
    &\langle \overline{V}^T - \Pi_{W^{T-1}}(\overline{V}^{T-1}), \Pi_{W^T}(\overline{V}^{T-1}) - \Pi_{W^{T-1}}(\overline{V}^{T-1})\rangle\\&\le 
    \langle \overline{V}^{T-1} - \Pi_{W^{T-1}}(\overline{V}^{T-1}), \  \Pi_{W^T}(\overline{V}^{T-1}) - \Pi_{W^{T-1}}(\overline{V}^{T-1})\rangle\\
    &\qquad + \|\overline{V}^T - \overline{V}^{T-1}\|\cdot \| \Pi_{W^T}(\overline{V}^{T-1}) - \Pi_{W^{T-1}}(\overline{V}^{T-1})\|\\
    &\le d(\overline{V}^{T-1}, W^{T-1})\cdot \left\langle d^t,  \Pi_{W^T}(\overline{V}^{T-1}) - \Pi_{W^{T-1}}(\overline{V}^{T-1})\right\rangle + \frac{1}{T}\| \Pi_{W^T}(\overline{V}^{T-1}) - \Pi_{W^{T-1}}(\overline{V}^{T-1})\|.
\end{align*}
Now denote $\tilde{d}^t = \frac{\Pi_{W^T}(\overline{V}^{T-1})-\overline{V}^{T-1}}{\|\Pi_{W^T}(\overline{V}^{T-1})-\overline{V}^{T-1}\|}$, then by Lemma \ref{lemma:direc}, we can get 
\begin{align*}
    d(\overline{V}^{T-1}, W^{T-1})\cdot \|d^t - \tilde{d}^t\| \le 4\sqrt{d(\overline{V}^{T-1}, W^{T-1})d_{B_1}(W^{T-1}, W^T)} + 2d_{B_1}(W^{T-1}, W^T),
\end{align*}
Then we can bound the inner product term as 
\begin{align*}
    &d(\overline{V}^{T-1}, W^{T-1})\cdot \left\langle d^t,  \Pi_{W^T}(\overline{V}^{T-1}) - \Pi_{W^{T-1}}(\overline{V}^{T-1})\right\rangle\\
    &\le d(\overline{V}^{T-1}, W^{T-1})\cdot \|d^t - \tilde{d}^t\|\cdot \| \Pi_{W^T}(\overline{V}^{T-1}) - \Pi_{W^{T-1}}(\overline{V}^{T-1})\|\\
    &\qquad + d(\overline{V}^{T-1}, W^{T-1})\cdot \left\langle \tilde{d}^t,  \Pi_{W^T}(\overline{V}^{T-1}) - \Pi_{W^{T-1}}(\overline{V}^{T-1})\right\rangle.
\end{align*}
By the definition of $\tilde{d}^t$, we know that 
\begin{align*}
   \langle \tilde{d}^t,  \Pi_{W^T}(\overline{V}^{T-1})\rangle &= \min_{x \in W^T}\langle \tilde{d}^t, x\rangle \le \langle \tilde{d}^t, \Pi_{W^T}(\Pi_{W^{T-1}}(\overline{V}^{T-1}))\rangle \\
   &\le d_{B_1}(W^T, W^{T-1}) + \langle \tilde{d}^t, \Pi_{W^{T-1}}(\overline{V}^{T-1})\rangle .
\end{align*}
Hence the inner product term can be further bounded by  
\begin{align}
    &d(\overline{V}^{T-1}, W^{T-1})\cdot \left\langle d^t,  \Pi_{W^T}(\overline{V}^{T-1}) - \Pi_{W^{T-1}}(\overline{V}^{T-1})\right\rangle\nonumber\\
    &\le \left(4\sqrt{d(\overline{V}^{T-1}, W^{T-1})d_{B_1}(W^{T-1}, W^T)} + 2d_{B_1}(W^{T-1}, W^T) \right)^2\nonumber\\
    &\qquad + d(\overline{V}^{T-1}, W^{T-1})\cdot d_{B_1}(W^T, W^{T-1})\nonumber\\
    &\le 33d(\overline{V}^{T-1}, W^{T-1})\cdot d_{B_1}(W^T, W^{T-1}) + 8d^2_{B_1}(W^{T-1}, W^T).\label{eq:inner product estimate alpha}
\end{align}
Now continue to bound the right side in Eq. \eqref{ineq:first}, we can further get that
\begin{small}
\begin{align}
    &T^2d(\overline{V}^T, W^T)^2\le T^2\|\overline{V}^T - \Pi_{W^{T-1}}(\overline{V}^{T-1})\|^2 + 37T^2 d(\overline{V}^{T-1}, W^{T-1})\cdot d_{B_1}(W^T, W^{T-1}) + 10T^2 d^2_{B_1}(W^{T-1}, W^T).\label{ineq:second}
\end{align}
\end{small}
Now we can further bound the Eq. \eqref{ineq:second} by expanding the first term on the right side:
\begin{align}
    T^2\|\overline{V}^T - \Pi_{W^{T-1}}(\overline{V}^{T-1})\|^2 &= \left(T-1\right)^2 \|\overline{V}^{T-1} - \Pi_{W^{T-1}}(\overline{V}^{T-1})\|^2 + \|V^T - \Pi_{W^{T-1}}(\overline{V}^{T-1})\|^2\nonumber\\
    &\qquad  + 2(T-1)\left\langle \overline{V}^{T-1} - \Pi_{W^{T-1}}(\overline{V}^{T-1}), V^{T} - \Pi_{W^{T-1}}(\overline{V}^{T-1})\right\rangle\\
    &\le \left(T-1\right)^2 \|\overline{V}^{T-1} - \Pi_{W^{T-1}}(\overline{V}^{T-1})\|^2 + (B+B_1)^2m\nonumber\\
    &\qquad  + 2(T-1)\left\langle \overline{V}^{T-1} - \Pi_{W^{T-1}}(\overline{V}^{T-1}), V^{T} - \Pi_{W^{T-1}}(\overline{V}^{T-1})\right\rangle.
\end{align}
The inner product term is 
\begin{align*}
    \left\langle \overline{V}^{T-1} - \Pi_{W^{T-1}}(\overline{V}^{T-1}), V^{T} - \Pi_{W^{T-1}}(\overline{V}^{T-1})\right\rangle = d(\overline{V}^{T-1}, W^{T-1})\cdot \left\langle d^{T-1},  \Pi_{W^{T-1}}(\overline{V}^{T-1})-V^T \right\rangle.
\end{align*}
Note that $\langle d^{T-1}, \Pi_{W^{T-1}}(\overline{V}^{T-1})\rangle = \min_{z \in W^{T-1}}\langle d^{T-1}, z\rangle. $ Because $\|\Pi_{W^*}(\overline{V}^{T-1})\|\le B_1,$ there is a $z' \in W^{T-1}$ such that $\|z'-\Pi_{W^*}(\overline{V}^{T-1})\| \le d_{B_1}(W^*, W^{T-1}).$ Hence, 
\begin{align*}\langle d^{T-1}, \Pi_{W^{T-1}}(\overline{V}^{T-1})\rangle &\le \min_{z \in W^{T-1}}\langle d^{T-1},z\rangle \le  \langle d^{T-1}, z'\rangle  \le \min_{z \in W^*}\langle d^{T-1},z\rangle + d_{B_1}(W^*, W^{T-1})\\
&\le J(r^*, d^{T-1}, \pi^*) + D(\pi^*) + d_{B_1}(W^*, W^{T-1}).\end{align*}
The last inequality holds by Lemma \ref{lemma:approach}. Now we continue to bound the inner product term. We have  
\begin{align*}
&\left\langle \overline{V}^{T-1} - \Pi_{W^{T-1}}(\overline{V}^{T-1}), V^{T} - \Pi_{W^{T-1}}(\overline{V}^{T-1})\right\rangle\\&\qquad  = d(\overline{V}^{T-1}, W^{T-1})\cdot \left\langle d^{T-1},  \Pi_{W^{T-1}}(\overline{V}^{T-1})-V^T \right\rangle 
    \\&\qquad \le d(\overline{V}^{T-1}, W^{T-1})\cdot (\|d^{T-1}\|_1\cdot J(r_1^*, \cdots, r_m^*, \overline{d^{T-1}}, \pi^*)+D(\pi^*)+d_{B_1}(W^{T-1}, W^*)- J(\hat{r}_1^t, \cdots, \hat{r}_m^t, d^{T-1}, \pi^t))\\
    &\qquad \le d(\overline{V}^{T-1}, W^{T-1})\cdot \left(\eta \|d^{T-1}\|_1\cdot \left(\sum_{i=1}^m L_i^{T-1}(\theta^*)-\sum_{i=1}^m L_i^{T-1}(\theta^{T-1}) \right)+ D(\pi^*)+d_{B_1}(W^{T-1}, W^*)\right).
\end{align*}
Thus the Eq. \eqref{ineq:second} can be rewritten as 
\begin{small}
\begin{align*}
    & T^2d(\overline{V}^T, W^T)^2
    \\& \le \left(T-1\right)^2 \|\overline{V}^{T-1} - \Pi_{W^{T-1}}(\overline{V}^{T-1})\|^2+(B+B_1)^2m + 10T^2 d_{B_1}^2(W^{T-1}, W^T)\\&\qquad  + 2(T-1) d(\overline{V}^{T-1}, W^{T-1})\cdot \left(\eta \|d^t\|_1\cdot \left( \sum_{i=1}^m L_i^{T-1}(\theta^*)-\sum_{i=1}^m L_i^{T-1}(\theta^t) \right)+ D(\pi^*)+d_{B_1}(W^{T-1}, W^*) + 37 Td_{B_1}(W^T, W^{T-1})\right).
\end{align*}
\end{small}
Then by the recursion, we can get 
\begin{align*}
    & Td(\overline{V}^T, W^T)^2
    \\&\le (B+B_1)^2m + \sum_{t=1}^T \frac{10t^2 d_{B_1}^2(W^{t-1}, W^t)}{T} \\
    &\qquad +\sum_{t=1}^T \frac{2(t-1)}{T}d(\overline{V}^{t-1}, W^{t-1})\cdot \Bigg(\eta \|d^{T-1}\|_1\cdot \left( \sum_{i=1}^m L_i^{T-1}(\theta^*)-\sum_{i=1}^m L_i^{T-1}(\theta^{T-1}) \right)\\&\qquad \qquad + D(\pi^*)+d_{B_1}(W^{t-1}, W^*) + 37 td_{B_1}(W^t, W^{t-1})\Bigg).
\end{align*}
By this recursion formula, we can use the induction method to prove that 
\begin{align*}
    d(\overline{V}^T, W^T) &\le \frac{(B+B_1)^2m}{\sqrt{T}} + \underbrace{\sum_{t=1}^T \frac{10t^2}{T^{3/2}} d_{B_1}^2(W^{t-1}, W^t) }_{\textrm{(A)}} + D(\pi^*) + \underbrace{\frac{\eta}{T}\sum_{t=1}^{T-1}\|d^t\|_1\sum_{i=1}^m (L_i^t(\theta^*) - L_i^t(\theta^t))}_{\textrm{(B)}}\\&\qquad +\underbrace{\frac{1}{T}\sum_{i=1}^T  d_{B_1}(W^{t-1}, W^*)}_{\textrm{(C)}} + \underbrace{\frac{1}{T}\sum_{t=1}^T37td_{B_1}(W^t, W^{t-1})}_{\textrm{(D)}}.
\end{align*}
 Now we bound all four terms. We first prove that term (A), (C) and (D) are all at level $\widetilde{\cO}(1/\sqrt{T}).$

\paragraph{Term (A):} First we consider term (A). Since $W^{t} = \bigcap_{n=1}^N W_{p^{(n)},c^{(n)}}^{\alpha^{t,(n)}}$, the term $d_{B_1}^2(W^{t-1}, W^t)$ can be bounded by 
\begin{align*}d_{B_1}^2(W^{t-1}, W^t) \le  \left(\sum_{n=1}^Nd_{B_1}\left(W_{p^{(n)},c^{(n)}}^{\alpha^{t-1,(n)}}, W_{p^{(n)},c^{(n)}}^{\alpha^{t,(n)}}\right)\right)^2\le N\sum_{n=1}^N d^2_{B_1}\left(W_{p^{(n)},c^{(n)}}^{\alpha^{t-1,(n)}}, W_{p^{(n)},c^{(n)}}^{\alpha^{t,(n)}}\right).\end{align*} Since 
 $\alpha^t = \frac{t-1}{t}\alpha^{t-1} +\frac{1}{t}\hat{\alpha}^t,$ we can know $\|\alpha^t- \alpha^{t-1}\|_\infty \le \frac{1}{t}\|\hat{\alpha}^t\|_\infty \le \frac{1}{t}$. Then, by Lemma \ref{lemma:estimation error of parameterized target set}, we have 
 \begin{align}d_{B_1}(W^{\alpha^{t-1,(n)}}_{p^{(n)},c^{(n)}}, W^{\alpha^{t,(n)}}_{p^{(n)},c^{(n)}}) \le \frac{m^{3/2}B_1}{|p^{(n)}|}\cdot \frac{1}{t}.\label{ineq:primal bound for distance of W}\end{align}
 Thus by Eq. \eqref{ineq:primal bound for distance of W}, we know that 
 \begin{align}
     \textrm{(A)}&\le \frac{10N}{T^{3/2}}\sum_{n=1}^N \sum_{t=1}^T \frac{m^3B_1^2}{(p^{(n)})^2}\nonumber\\
     &\le \sum_{n=1}^N \frac{10Nm^3B_1^2}{(p^{(n)})^2}\cdot \frac{1}{\sqrt{T}}.\label{ineq:upperbound of (A)}
 \end{align}

 \paragraph{Term (C):}
 We have 
 \begin{align}
     \textrm{(C)}&\le \frac{B_1}{T} + \frac{1}{T}\cdot \sum_{n=1}^N\frac{m^{3/2}B_1}{|p^{(n)}|}\cdot \sum_{t=2}^T  \|\alpha^{t-1,(n)}-\alpha^*\|_\infty\nonumber\\
     &\le \frac{1}{T}\cdot \sum_{n=1}^N\frac{m^{3/2}B_1}{|p^{(n)}|}\cdot \gamma^{-1}\exp(4/\beta)\cdot \widetilde{\cO}\left(\mathrm{poly}(m,e^B, d,\log(1/\delta))\right)\cdot \left(\sum_{t=1}^T \frac{1}{\sqrt{t}}+1\right)\nonumber\\
     &\le\frac{1}{\sqrt{T}}\cdot \frac{Nm^{3/2}B_1}{\min_{n \in [N]}|p^{(n)}|}\cdot \gamma^{-1}\exp(4/\beta)\cdot \widetilde{\cO}\left(\mathrm{poly}(m,e^B, d,\log(1/\delta))\right). \label{ineq:(C)general: first}
 \end{align}

 \paragraph{Term (D):} 
First, we have 
\begin{align*}
    \textrm{(D)}\le \frac{1}{T}\sum_{t=1}^T 37t \sum_{n=1}^N d_{B_1}(W^{t,(n)}, W^{t-1,(n)}).
\end{align*}
Then, by Lemma \ref{lemma:estimation error of parameterized target set}, 
 \begin{align}
     \textrm{(D)}&\le \frac{37B_1}{T} + \frac{1}{T}\sum_{t=2}^T\sum_{n=1}^N  \frac{37tm^{3/2}B_1}{|p^{(n)}|}\|\alpha^{t,(n)} - \alpha^{t-1,(n)}\|_\infty\nonumber\\
     &\le \frac{37m^{3/2}B_1}{T}\cdot \sum_{n=1}^N \frac{1}{|p^{(n)}|}\sum_{t=2}^T \left(\left\|\hat{\alpha}^{t,(n)}-\alpha^{t-1,(n)}\right\|_\infty + 1\right)\nonumber\\
     &\le \frac{37m^{3/2}B_1}{T}\cdot \sum_{n=1}^N \frac{1}{|p^{(n)}|}\left(\sum_{t=2}^T \left(\| \hat{\alpha}^{t,(n)}-\alpha^{*,(n)}\|_\infty + \left\|\alpha^{*,(n)}-\alpha^{t-1,(n)}\right\|_\infty\right) + 1\right)\nonumber\\
     &\le \frac{37m^{3/2}B_1}{T} \cdot \sum_{n=1}^N \frac{1}{|p^{(n)}|}\left(\underbrace{\sum_{t=2}^T \|\hat{\alpha}^{t,(n)} - \alpha^{*,(n)}\|_\infty}_{\textrm{(E)}} + \underbrace{\sum_{t=2}^T \|\alpha^{t-1,(n)}-\alpha^{*,(n)}\| }_{\textrm{(F)}}\right) + \frac{37m^{3/2}B_1}{T}.\label{ineq:(D) first}
 \end{align}
 For the term (E), by Eq.~\eqref{eq:estimate alpha final result}, we have 
 \begin{align*}
     \mathrm{(E)}\le \sum_{t=1}^T \|\hat{\alpha}^{t,(n)}-\alpha^{*,(n)}\|_\infty\le \gamma^{-1}\cdot \widetilde{\cO}\left(\mathrm{poly}(m,e^B, \exp(1/\beta), d,\log(1/\delta))\right)\cdot \sqrt{T}.
 \end{align*}
 Also by Eq.~\eqref{eq:estimate alpha final result}, 
 \begin{align*}
     \textrm{(F)}&\le \sum_{t=2}^T \gamma^{-1}\cdot \widetilde{\cO}\left(\mathrm{poly}(m,e^B,\exp(1/\beta) d,\log(1/\delta))\right)\cdot \frac{1}{\sqrt{t}}\\
     &\le \gamma^{-1}\cdot \widetilde{\cO}\left(\mathrm{poly}(m,e^B, \exp(1/\beta), d,\log(1/\delta))\right)\cdot \sqrt{T}.
 \end{align*}
 By Theorem \ref{thm:est of alpha}, the term (E) can be bounded by 
 \begin{align*}
     \textrm{(E)}\le \gamma^{-1}\mathrm{poly}(\exp(1/\beta), m, e^B, d, \log(1/\delta))\widetilde{\cO}(\sqrt{T}).
 \end{align*}
 Thus substitute these upper bound to the Eq. \eqref{ineq:(D) first}, we get
 \begin{align}
     \textrm{(D)}&\le \frac{1}{T}\sum_{t=1}^T 37t \sum_{n=1}^N d_{B_1}(W^{t,(n)}, W^{t-1,(n)})\nonumber\\&\le \gamma^{-1}\mathrm{poly}(\exp(1/\beta), m, N, e^B, d, \log(1/\delta), B_1, (\min_{n \in [N]}p^{(n)})^{-1})\cdot\widetilde{\cO}(1/\sqrt{T}).\label{ineq:(D)final}
 \end{align}

 \paragraph{Combine them:} Now we combine the upper bound of (A), (C), (D), i.e., Eq. \eqref{ineq:upperbound of (A)}, \eqref{ineq:(C)general: first}, \eqref{ineq:(D)final}, we can get 
 \begin{align}
     d(\overline{V}^T, W^T) \le \frac{(B+B_1)^2m}{\sqrt{T}} + \gamma^{-1}\mathrm{poly}(\exp(1/\beta), m, e^B, d, \log(1/\delta), \min_{n \in [N]}\frac{1}{p^{(n)}}, B_1)\widetilde{\cO}(1/\sqrt{T}) + D(\pi^*) + \textrm{(B)}. \label{eq:second step}
 \end{align}
 Now we consider the proof of Theorem \ref{thm:online}. 
\begin{align*}
    &D(\tilde{\pi}^T) - D(\pi^*) \\ &= d(W^*, \EE_{\tilde{\pi}^T}[r^*(x,y)]-\beta \DD_{\mathrm{KL}}(\tilde{\pi}^T\| \pi_{\mathrm{ref}})) - D(\pi^*)\\
    &\le d(W^*, \EE_{\tilde{\pi}^T}[r^*(x,y)]-\frac{\beta}{T} \sum_{t=1}^T \DD_{\mathrm{KL}}(\pi^t\|\pi_{\mathrm{ref}})) - D(\pi^*)\\
    &= \underbrace{d(W^*, \EE_{\tilde{\pi}^T}[r^*(x,y)]-\frac{\beta}{T} \sum_{t=1}^T \DD_{\mathrm{KL}}(\pi^t\|\pi_{\mathrm{ref}})) - d(W^*, \frac{1}{T}\sum_{t=1}^T\EE_{\pi^t}[\hat{r}^t(x,y)]-\frac{\beta}{T} \sum_{t=1}^T \DD_{\mathrm{KL}}(\pi^t\|\pi_{\mathrm{ref}}))}_{\textrm{(*)}} \\
    &\qquad + d(W^*, \frac{1}{T}\sum_{t=1}^T \EE_{\pi^t}[\hat{r}^t(x,y)]-\frac{\beta}{T} \sum_{t=1}^T \DD_{\mathrm{KL}}(\pi^t\|\pi_{\mathrm{ref}})) - D(\pi^*)\\
    &=\mathrm{(*)}
    + d\left(W^*, \overline{V}^T\right) - D(\pi^*)\\
    &\le \mathrm{(*)}+ d(W^*, W^T)
    + \underbrace{d\left(W^T, \overline{V}^T\right)  - D(\pi^*)}_{\mathrm{(**)}}.
\end{align*}
\paragraph{Term $(\ast)$:}
First, the term $(\ast)$ can be bounded by 
\begin{align*}
\mathrm{(*)}&\le \sum_{i=1}^m\Bigg|\frac{1}{T}\sum_{t=1}^T \EE_{\pi^t}\left[  \hat{r}_i^t(x,y) - r_i^*(x,y)\right]\Bigg|.\end{align*}
Now note that 
\begin{align}
&\frac{1}{T}\sum_{t=1}^T \EE_{\pi^t}\left[  \hat{r}_i^t(x,y) - r_i^*(x,y)\right]= \frac{1}{T}\sum_{t=1}^T \EE_{y_1\sim \pi^t,y_2\sim \pi_{\mathrm{base}}}\left[ \left((\hat{r}_i^t(x,y_1)- r_i^t(x,y_2)) - (r_i^*(x,y_1)  - r_i^*(x,y_2))\right) \right].\label{eq:first step}
\end{align}
Now since the reward contains a linear structure, by Lemma \ref{lemma:linearstructure} with $d_{\mathrm{cover}}(1/T) = \widetilde{\cO}(d),$ for any $\mu_i>0$ we can derive that 
\begin{align}
    (*)&\le \sum_{i=1}^m \mu_i \cdot \sum_{t=1}^T   \sum_{j=1}^{t-1}\EE_{y_1\sim \pi^j, y_2\sim \pi_{\textrm{base}}}[\left(r_i^t(x,y_1)-r_i^t(x,y_2) - (r_i^*(x,y_1)-r_i^*(x,y_2))\right)^2] + \frac{d_{\mathrm{cover}}(1/T)}{4\mu_i} + \widetilde{\cO}(Bd)\nonumber\\
    &\le\sum_{i=1}^m\mu_i \exp(4/\beta)\kappa\cdot \sum_{t=1}^T   \sum_{j=1}^{t-1}\EE_{y_1\sim \pi^j, y_2\sim \pi^j}[\left(r_i^t(x,y_1)-r_i^t(x,y_2) - (r_i^*(x,y_1)-r_i^*(x,y_2))\right)^2] + \frac{d_{\mathrm{cover}}(1/T)}{4\mu_i}+ \widetilde{\cO}(Bd)
    \nonumber\\& =\sum_{i=1}^m \mu_i \exp(4/\beta)\kappa\cdot \sum_{t=1}^T \sum_{j=1}^{t-1}\EE_{y_1,y_2\sim \pi^j}[\Delta_i^t(x,y)^2] + \frac{d_{\mathrm{cover}}(1/T)}{4\mu_i}+ \widetilde{\cO}(Bd),\label{ineq:(A)}
\end{align}
The last inequality uses the fact that 
$$\sup_{x,y} \frac{\pi_{\mathrm{base}}(y\mid x)}{\pi^j(y\mid x)} \le \sup_{x,y} \frac{\pi_{\mathrm{base}}(y\mid x)}{\pi_{\mathrm{ref}}(y\mid x)}\cdot \sup_{x,y}\frac{\pi_{\mathrm{ref}}(y\mid x)}{\pi^j(y\mid x)} \le \exp(4/\beta)\cdot \kappa,$$
where $\kappa = \sup_{x,y} \frac{\pi_{\mathrm{base}}(y\mid x)}{\pi_{\mathrm{ref}(y\mid x)}}$ \citep{cen2024value}.
\paragraph{Term ($\ast\ast$):}
Now we consider the term $(\ast\ast)$. 
By Eq.~\eqref{eq:second step}, we know that 
\begin{align}
    (**) &\le \frac{(B+B_1)^2m}{\sqrt{T}} + \gamma^{-1}\mathrm{poly}(\exp(1/\beta), m, e^B, d, \log(1/\delta), \min_{n \in [N]}\frac{1}{p^{(n)}}, B_1)\widetilde{\cO}(1/\sqrt{T}) \nonumber\\&\qquad + \frac{1}{T}\sum_{t=1}^T \sum_{i=1}^m \eta \|d^t\|_1 (L_i^t(\theta_i^*) - L_i^t(\theta_i^t)). \label{ineq:bound (**)}
\end{align}
Now by the MLE loss, there exists a constant $C'$ such that 
\begin{align}
    &\frac{1}{T}\sum_{t=1}^T \sum_{i=1}^m \eta \|d^t\|_1(L_i^t(\theta_i^*) - L_i^t(\theta_i^t))
    \nonumber\\&\le 2\sum_{i=1}^m \eta \|d^t\|_1 \log(|\cR|/\delta) - \frac{C'}{T}\sum_{t=1}^T \sum_{i=1}^m\eta\|d^t\|_1\sum_{j \in \cD_i^{t-1}}\EE_{y\sim \pi^j} \left[\Delta_i^t(x,y)^2\right]\nonumber\\
    &=\widetilde{\cO}(2m \eta \sqrt{m}d) - \frac{C'}{T}\sum_{t=1}^T \sum_{i=1}^m\eta\|d^t\|_1\sum_{j \in \cD_i^{t-1}}\EE_{y\sim \pi^j} \left[\Delta_i^t(x,y)^2\right],\label{eq:etaloss}
\end{align}
Now consider the second term in Eq.~\eqref{eq:etaloss}. We can bound it by
\begin{align}&\sum_{t=1}^T \sum_{i=1}^m \sum_{j \in \cD_i^{t-1}}\EE_{y_1,y_2\sim \pi^j}[\Delta_i^t(x,y)^2]\nonumber \\ &= \sum_{t=1}^T \sum_{i=1}^m \sum_{j=1}^{t-1}\EE_{y_1,y_2\sim \pi^j, I \sim \PP(\cdot \mid \alpha^*, x,y_1,y_2, r^*)}[\Delta_i^t (x,y)^2 \mathbb{I}\{I^j = i\}]\nonumber\\
    &\ge \kappa_1\sum_{i=1}^m \sum_{j=1}^T \sum_{t=j+1}^T \EE_{y_1,y_2\sim \pi^j, I= i}[\Delta_i^t(x^j,y^j)^2 \mathbb{I}\{I^j = i\}]\nonumber\\
    &= \kappa_1\sum_{i=1}^m \sum_{j=1}^T \sum_{t=j+1}^T \EE_{y_1,y_2\sim \pi^j} [\Delta_i^t(x,y)^2]\nonumber\\
    &= \kappa_1\cdot \sum_{i=1}^m \sum_{t=1}^T \sum_{j=1}^{t-1}\EE_{y_1,y_2\sim \pi^j} [\Delta_i^t(x,y)^2],\label{eq:estimate technique}
\end{align}
where the inequality uses the fact that 
$\inf_{y,x,j, I}\frac{1}{ \PP(I\mid \alpha^*, x,y_1,y_2, r^*)} = \kappa_1$ for some constant $\kappa_1.$ Since the distribution of index is a bounded softmax distribution, we can derive that $\kappa_1\ge \frac{e^0}{e^0+(m-1)e^B}\ge \frac{1}{me^B}$. Thus we can get 
\begin{align}
    \sum_{t=1}^T \sum_{i=1}^m \sum_{j \in \cD_i^{t-1}}\EE_{y_1,y_2\sim \pi^j}[\Delta_i^t(x,y)^2] \ge \frac{1}{me^B}\cdot \sum_{i=1}^m \sum_{t=1}^T \sum_{j=1}^{t-1}\EE_{y_1,y_2\sim \pi^j} [\Delta_i^t(x,y)^2].\nonumber
\end{align}
Hence, the Eq.~\eqref{eq:etaloss} can be further bounded by 
\begin{align}
    \frac{1}{T}\sum_{t=1}^T \sum_{i=1}^m \eta \|d^t\|_1(L_i^t(\theta_i^*) - L_i^t(\theta_i^t)) &\le \widetilde{\cO}(2m\eta \sqrt{m}d) - \frac{\eta C' \|d^t\|_1}{Tme^B}\sum_{i=1}^m \sum_{t=1}^T \sum_{j=1}^{t-1}\EE_{y_1,y_2\sim \pi^j} [\Delta_i^t(x,y)^2]\label{eq:lossgap1}\\
    &\le \widetilde{\cO}(2m\eta \sqrt{m}d) - \frac{\eta C'}{Tme^B}\sum_{i=1}^m \sum_{t=1}^T \sum_{j=1}^{t-1}\EE_{y_1,y_2\sim \pi^j} [\Delta_i^t(x,y)^2].\label{eq:lossgap}
\end{align}
The last inequality uses the fact that $\|d^t\|_1 \ge 1.$
 Now combining $(\ast)$ (Eq.~\eqref{ineq:(A)}) and $(\ast\ast)$ (Eq.~\eqref{ineq:bound (**)}), 
by choosing $\mu_i = \frac{C'}{ me^B\exp(4/\beta)\kappa\sqrt{T}}$, we can get 
\begin{align}
    &D(\tilde{\pi}^T) - D(\pi^*) \nonumber\\& \le (*) + (**) + d(W^*, W^T)\nonumber\\
    & \le \frac{me^B\exp(4/\beta)\kappa d_{\mathrm{cover}}(1/T)}{4C'\sqrt{T}} + \frac{(B+B_1)^2m}{\sqrt{T}} +  \widetilde{\cO}(Bd) + \widetilde{\cO}\left(\frac{m^{3/2}d}{\sqrt{T}}\right) + d(W^*, W^T).\label{eq:final_step1}
\end{align}
Note that 
\begin{align}d(W^*, W^T)&\le  \sum_{n=1}^N d_{B_1}(W_{p^{(n)}, c^{(n)}}^{\alpha^{T,(n)}}, W_{p^{(n)},c^{(n)}}^{\alpha^{*,(n)}})\le m^{3/2}B_1\sum_{n=1}^N \frac{1}{|p^{(n)}|}\cdot \|\alpha^{T,(n)}-\alpha^{*,(n)}\|_\infty\nonumber\\
&\le \frac{m^{3/2}B_1}{\sqrt{T}}\cdot \left(\sum_{n=1}^N \frac{1}{p^{(n)}}\right)\cdot \gamma^{-1}\mathrm{poly}(\exp(1/\beta), m, e^B, d, \log(1/\delta))\label{eq:step2explain}\\
&\le \frac{m^{3/2}B_1N}{\sqrt{T}}\cdot (\min_{n \in [N]}p^{(n)})^{-1} \cdot \gamma^{-1}\mathrm{poly}(\exp(1/\beta), m, e^B, d, \log(1/\delta)),\label{eq:finalstep2}\end{align}
where the inequality Eq.~\eqref{eq:step2explain} holds by Theorem \eqref{thm:est of alpha}.

Hence, combining Eq.~\eqref{eq:final_step1} and Eq.~\eqref{eq:finalstep2}, we complete the proof.
\end{proof}


\subsection{Proof of Theorem \ref{thm:malfare online}}\label{app:proof malfare online}

\begin{proof}
First, note that \begin{align}
     d(\overline{V}^T, W^{T,(n)})^{2q} &= \|\overline{V}^T - \Pi_{W^{T,(n)}}(\overline{V}^T)\|^{2q}\nonumber\\
    &\le \|\overline{V}^T - \Pi_{W^{T,(n)}}(\overline{V}^{T-1})\|^{2q}\nonumber\\
    &\le \Bigg(\|\overline{V}^T - \Pi_{W^{T-1,(n)}}(\overline{V}^{T-1})\|^2 + \|\Pi_{W^{T,(n)}}(\overline{V}^{T-1}) - \Pi_{W^{T-1,(n)}}(\overline{V}^{T-1})\|^2\nonumber \\&\qquad + 2\langle \overline{V}^T - \Pi_{W^{T-1,(n)}}(\overline{V}^{T-1}), \Pi_{W^{T,(n)}}(\overline{V}^{T-1}) - \Pi_{W^{T-1,(n)}}(\overline{V}^{T-1})\rangle\Bigg)^{q}.\label{ineq:first malfare}
\end{align}
Now by Lemma \ref{lemma:dis of proj}, since $\|V^{T-1}\|_\infty\le B$ is bounded, we have 
\begin{align}\|\Pi_{W^{T,(n)}}(\overline{V}^{T-1}) - \Pi_{W^{T-1,(n)}}(\overline{V}^{T-1})\|_2^2 &\le 4d(\overline{V}^{T-1}, W^{T-1,(n)})d_{B_1}(W^{T,(n)}, W^{T-1,(n)}) + 2d^2_{B_1}(W^{T,(n)}, W^{T-1,(n)}).
\end{align}
Also, by Eq.~\eqref{eq:inner product estimate alpha}, we can also have 
\begin{align}
    &\langle \overline{V}^T - \Pi_{W^{T-1,(n)}}(\overline{V}^{T-1}), \Pi_{W^{T,(n)}}(\overline{V}^{T-1}) - \Pi_{W^{T-1,(n)}}(\overline{V}^{T-1})\rangle \nonumber\\&\le 33d(\overline{V}^{T-1}, W^{T-1})\cdot d_{B_1}(W^{T,(n)}, W^{T-1,(n)}) + 8d^2_{B_1}(W^{T-1,(n)}, W^{T,(n)}).\nonumber
\end{align}
Hence, by Eq.~\eqref{ineq:first malfare}, we can get 
\begin{align}
    &T^{2q} d(\overline{V}^T, W^{T,(n)})^{2q}  \nonumber\\&\le \Bigg(\|\overline{V}^T - \Pi_{W^{T-1,(n)}}(\overline{V}^{T-1})\|^2 + \|\Pi_{W^{T,(n)}}(\overline{V}^{T-1}) - \Pi_{W^{T-1,(n)}}(\overline{V}^{T-1})\|^2\nonumber \\&\qquad + 2\langle \overline{V}^T - \Pi_{W^{T-1,(n)}}(\overline{V}^{T-1}), \Pi_{W^{T,(n)}}(\overline{V}^{T-1}) - \Pi_{W^{T-1,(n)}}(\overline{V}^{T-1})\rangle\Bigg)^{q}.\nonumber\\
    &\le T^{2q}\left(\|\overline{V}^T - \Pi_{W^{T-1,(n)}}(\overline{V}^{T-1})\|^{2} + 37d(\overline{V}^{T-1}, W^{T-1})\cdot d_{B_1}(W^{T,(n)}, W^{T-1,(n)}) + 10d^2_{B_1}(W^{T-1,(n)}, W^{T,(n)})\right)^q.\label{eq:malfare further}
    \end{align}
    Now, since $d(W^{T,(n)}, W^{T-1,(n)}) \le \frac{m^{3/2}B_1}{|p^{(n)}|}\cdot \|\alpha^{t,(n)}-\alpha^{t-1,(n)}\|_\infty \le \frac{m^{3/2}B_1}{|p^{(n)}|T}$, we know 
\begin{align}
    37d(\overline{V}^{T-1}, W^{T-1})\cdot d_{B_1}(W^{T,(n)}, W^{T-1,(n)}) + 10d^2_{B_1}(W^{T-1,(n)}, W^{T,(n)}) &\le \frac{37B_1^2m^{3/2}}{|p^{(n)}|T} + \frac{10B_1^2m^3}{|p^{(n)}|^2 T^2}.
\end{align}
Hence, the Eq.~\eqref{eq:malfare further} can be further bounded by 
    \begin{align*}
    &T^{2q} d(\overline{V}^T, W^{T,(n)})^{2q}\\
    &\le T^{2q}d^{2q}(\overline{V}^T, W^{T-1,(n)}) + \widetilde{\cO}(\mathrm{poly}(B_1^{q},m^{q}, (\min_{n \in [N]}p^{(n)})^{-q})T^{2q-2})\\&\qquad + qT^{2q}\|\overline{V}^T - \Pi_{W^{T-1,(n)}}(\overline{V}^{T-1})\|^{2q-2}\cdot \left(37d(\overline{V}^{T-1}, W^{T-1,(n)})\cdot d_{B_1}(W^{T,(n)}, W^{T-1,(n)}) + 10d^2_{B_1}(W^{T-1,(n)}, W^{T,(n)})\right)\\
    &\le T^{2q}d^{2q}(\overline{V}^T, W^{T-1,(n)}) + \widetilde{\cO}(\mathrm{poly}(B_1^{q},m^{q}, (\min_{n \in [N]}p^{(n)})^{-q})T^{2q-2}) \\&\qquad + 37qT^{2q}d^{2q-1}(\overline{V}^{T-1}, W^{T-1,(n)}) \cdot d_{B_1}(W^{T,(n)}, W^{T-1,(n)}) .
\end{align*}
The last inequality is because $\|P_{T,n}\| = \mathrm{poly}(B_1,m,(\min_{n \in [N]}p^{(n)})^{-1} \cdot \widetilde{\cO}(1/T)$, and $$\|\overline{V}^T - \Pi_{W^{T-1,(n)}}(\overline{V}^{T-1})\|^{2q-2} - d^{2q-2}(\overline{V}^{T-1}, W^{T-1,(n)}) \le \widetilde{\cO}(\mathrm{poly}(B_1^{q},m^{q}, (\min_{n \in [N]}p^{(n)})^{-q})T^{2q-3}).$$
Now we further bound the first term $T^{2q}d^{2q}(\overline{V}^{T}, W^{T-1,(n)}).$ 
Using the same derivation for Eq.~\eqref{eq:recursion before}, 
we know that \begin{align*}
    &T^{2q} \sum_{n=1}^N \zeta_nd^{2q}(\overline{V}^{T}, W^{T-1,(n)}) \\&\le (T-1)^{2q} \sum_{n=1}^N \zeta_nd^{2q}(\overline{V}^{T-1}, W^{T-1,(n)}) + 12^qT^{2q-2}B_1^{2q}m^q\\
&\qquad + 2(T-1)^{2q-1}\sum_{n=1}^N \zeta_n(\overline{V}^{T-1}-\Pi_{W^{T-1,(n)}}(\overline{V}^{T-1}))(V^T-\Pi_{W^{T-1,(n)}}(\overline{V}^{T-1}))d^{2q-2}(\overline{V}^{T-1}, W^{T-1,(n)}).
\end{align*}
Hence, we can derive 
\begin{align*}
    &T^{2q}\sum_{n=1}^N \zeta_n d(\overline{V}^T, W^{T,(n)})^{2q} \nonumber\\&\le (T-1)^{2q} \sum_{n=1}^N \zeta_nd^{2q}(\overline{V}^{T-1}, W^{T-1,(n)})  + \widetilde{\cO}(\mathrm{poly}(B_1^{q},m^{q}, (\min_{n \in [N]}p^{(n)})^{-q})T^{2q-2}) \\&\qquad +37qT^{2q}\sum_{n=1}^N \zeta_nd^{2q-1}(\overline{V}^{T-1}, W^{T-1,(n)})\cdot d_{B_1}(W^{T,(n)}, W^{T-1,(n)}) \\
&\qquad\qquad   + 2(T-1)^{2q-1}\sum_{n=1}^N \zeta_n(\overline{V}^{T-1}-\Pi_{W^{T-1,(n)}}(\overline{V}^{T-1}))(V^T-\Pi_{W^{T-1,(n)}}(\overline{V}^{T-1}))d^{2q-2}(\overline{V}^{T-1}, W^{T-1,(n)}).
\end{align*}
Now we consider the last term in the inequation above. Similar to the Eq.~\eqref{last term bound}, we have 
\begin{align}
&\sum_{n=1}^N \zeta_n(\overline{V}^{T-1}-\Pi_{W^{T-1,(n)}}(\overline{V}^{T-1}))(V^T-\Pi_{W^{T-1,(n)}}(\overline{V}^{T-1}))d^{2q-2}(\overline{V}^{T-1}, W^{T-1,(n)})\nonumber\\&\qquad \le \left(\sum_{n=1}^N \zeta_n d^{2q}(\overline{V}^{T-1}, W^{T-1,(n)})\right)^{\frac{2q-1}{2q}} \cdot \left(D_{q}(\pi^*) + \underbrace{\eta \|d^t\|_1\left(\sum_{i=1}^m L_i^t(\theta^*) -  \sum_{i=1}^m L_i^t(\theta^t)\right)}_{\textrm{($\ast$)}}\right),\nonumber
\end{align}
then we can get 
\begin{align*}
    &T^{2q}\sum_{n=1}^N \zeta_n d(\overline{V}^T, W^{T,(n)})^{2q} \nonumber\\&\le (T-1)^{2q} \sum_{n=1}^N \zeta_nd^{2q}(\overline{V}^{T-1}, W^{T-1,(n)})  + \widetilde{\cO}(\mathrm{poly}(B_1^{q},m^{q}, (\min_{n \in [N]}p^{(n)})^{-q})T^{2q-2}) \\&\qquad +37qT^{2q}\left(\sum_{n=1}^N \zeta_nd^{2q}(\overline{V}^{T-1}, W^{T-1,(n)})\right)^{\frac{2q-1}{2q}}\cdot \sqrt[2q]{\sum_{n=1}^N d_{B_1}^{2q}(W^{T,(n)}, W^{T-1,(n)}) }\\
&\qquad\qquad  2(T-1)^{2q-1}\left(\sum_{n=1}^N \zeta_n d^{2q}(\overline{V}^{T-1}, W^{T-1,(n)})\right)^{\frac{2q-1}{2q}} \cdot \left(D_{q}(\pi^*) + (\ast)\right)\\
&\le (T-1)^{2q} \sum_{n=1}^N \zeta_nd^{2q}(\overline{V}^{T-1}, W^{T-1,(n)})  + \widetilde{\cO}(\mathrm{poly}(B_1^{q},m^{q}, (\min_{n \in [N]}p^{(n)})^{-q})T^{2q-2}) \\&\qquad +2(T-1)^{2q-1}\left(\sum_{n=1}^N \zeta_nd^{2q}(\overline{V}^{T-1}, W^{T-1,(n)})\right)^{\frac{2q-1}{2q}} \\
&\qquad\qquad   \cdot \left(\frac{37qT^{2q}}{(T-1)^{2q-1}}\sqrt[2q]{\sum_{n=1}^N 2^q T d_{B_1}^{2q}(W^{T,(n)}, W^{T-1,(n)}) } + D_{q}(\pi^*) + (\ast)\right).
\end{align*}
Hence, by the reduction and the fact that $\frac{T}{T-1}\le 2$ for $T \ge 2,$, we can further get 
\begin{align*}
    &T^{2q}\sum_{n=1}^N \zeta_n d^{2q}(\overline{V}^T, W^{T,(n)})\nonumber\\&\le\widetilde{\cO}(\mathrm{poly}(B_1^{q},m^{q}, (\min_{n \in [N]}p^{(n)})^{-q})T^{2q-1}) + \sum_{t=1}^{T}2(t-1)^{2q-1}\left(\sum_{n=1}^N \zeta_nd^{2q}(\overline{V}^{T-1}, W^{T-1,(n)})\right)^{\frac{2q-1}{2q}} \\
&\qquad   \cdot \left(37q\cdot 2^q T\cdot \sqrt[2q]{\sum_{n=1}^N d_{B_1}^{2q}(W^{T,(n)}, W^{T-1,(n)}) } + D_{q}(\pi^*) + \sqrt[2q]{\sum_{n=1}^N d^{2q}_{B_1}(W^{T-1,(n)}, W^*)}+(\ast)\right).
\end{align*}
Similar to the Section \ref{sec: malfare_proof_offline}, we can use the induction method to derive

\begin{align*}
    &\sqrt[2q]{\sum_{n=1}^N \zeta_n d^{2q}(\overline{V}^T, W_n^T)} - D_q(\pi^*) \\& \le \widetilde{\cO}\left(\mathrm{poly}(B_1, m, (\min_{n \in [N]}p^{(n)})^{-1})T^{-1/2q}\right)  + (\ast) +\frac{1}{T}\sum_{i=1}^T  \sqrt[2q]{\sum_{n=1}^N \zeta_n d_{B_1}^{2q}(W^{t-1,(n)}, W^*)} \\&\qquad+ \frac{1}{T}\sum_{t=1}^T 37q\cdot 2^qt\sqrt[2q]{\sum_{n=1}^N \zeta_nd_{B_1}^{2q}(W^{t,(n)}, W^{t-1,(n)}) }.\end{align*}
Now note that 
\begin{align}
    \sqrt[2q]{\sum_{n=1}^N \zeta_n d_{B_1}^{2q}(W^{t-1,(n)}, W^*)} \le \sqrt[2q]{\sum_{n=1}^N\zeta_n \frac{m^{3q}B_1^{2q}\|\hat{\alpha}^{t-1,(n)} - \alpha^{*,(n)}\|_\infty^{2q}}{|p^{(n)}|^{2q}}}\le \frac{\gamma^{-1}\cdot \widetilde{\cO}(\mathrm{poly}(m,e^B,\exp(1/\beta), d, \log(1/\delta)))}{\min_{n \in [N]}|p^{(n)}|} \cdot \frac{1}{\sqrt{t}}.\label{eq:derivation 1}
\end{align}
Also, by Eq.~\eqref{ineq:(D)final}, we have
\begin{align}
    \frac{1}{T}\sum_{t=1}^T t\sqrt[2q]{\sum_{n=1}^N \zeta_n d_{B_1}^{2q}(W^{t,(n)}, W^{t-1,(n)}) } &\le \frac{1}{T}\sum_{t=1}^Tt\sum_{n=1}^N d_{B_1}(W^{t,(n)}, W^{t-1,(n)})\nonumber\\&\le \gamma^{-1}\mathrm{poly}(\exp(1/\beta), m, N, e^B, d, \log(1/\delta), B_1, (\min_{n \in [N]}p^{(n)})^{-1})\cdot\widetilde{\cO}(1/\sqrt{T}).\label{eq:derivation 2}
\end{align}
Hence, combining Eq.~\eqref{eq:derivation 1} and Eq.~\eqref{eq:derivation 2}, 
    \begin{align*}
    &\sqrt[2q]{\sum_{n=1}^N \zeta_n d^{2q}(\overline{V}^T, W_n^T)} - D_q(\pi^*)\\&\le \widetilde{\cO}\left(\mathrm{poly}(B_1, m, (\min_{n \in [N]}p^{(n)})^{-1})T^{-1/2q}\right) + \gamma^{-1}\mathrm{poly}(\exp(1/\beta), m, N, e^B, d, \log(1/\delta), B_1, (\min_{n \in [N]}p^{(n)})^{-1})\widetilde{\cO}(1/\sqrt{T}) + \mathrm{(B)}\\
    &\le \widetilde{\cO}\left((\gamma^{-1}\mathrm{poly}(\exp(1/\beta), m, N,e^B, d, \log(1/\delta), B_1, (\min_{n \in [N]}p^{(n)})^{-1})T^{-1/2q}\right) + (\ast).
\end{align*}
Now we derive the proof. First,
\begin{small}
\begin{align*}
    &D_q(\tilde{\pi}^T) - D_q(\pi^*)\\
    &=\sqrt[2q]{\sum_{n=1}^N\zeta_n d^{2q}(S(\tilde{\pi}^T), W_n^*)} - \sqrt[2q]{\sum_{n=1}^N\zeta_n d^{2q}(S(\tilde{\pi}^T), W_n^T)} + \sqrt[2q]{\sum_{n=1}^N\zeta_n d^{2q}(S(\tilde{\pi}^T), W_n^T)}  - D_q(\pi^*)\\
    &\le \sqrt[2q]{\sum_{n=1}^N \zeta_n |d(S(\tilde{\pi}^T), W_n^T) - d(S(\tilde{\pi}^T),W_n^*)|^{2q}}\\&\qquad +\sqrt[2q]{\sum_{n=1}^N\zeta_n d^{2q}(W_n^T, \EE_{\tilde{\pi}^T}\left[r^*(x,y) \right]- \frac{\beta}{T}\sum_{t=1}^T \DD_{\mathrm{KL}}(\pi^t\|\pi_{\mathrm{ref}})\cdot \mathbf{1}^m)} - \sqrt[2q]{\sum_{n=1}^N\zeta_n d^{2q}(W_n^T, \overline{V}^T)}\\
    &\qquad\qquad  + \sqrt[2q]{\sum_{n=1}^N \zeta_n d^{2q}(W_n^T, \overline{V}^T)} - D_q(\pi^*)\\
    &\le \sum_{n=1}^N d(W_n^*, W_n^T)
    \\&\qquad + \underbrace{\sqrt[2q]{\sum_{n=1}^N\zeta_n \left(d(W_n^T, \EE_{\tilde{\pi}^T}\left[r^*(x,y) \right]- \frac{\beta}{T}\sum_{t=1}^T \DD_{\mathrm{KL}}(\pi^t\|\pi_{\mathrm{ref}})\cdot \mathbf{1}^m)-d(W_n^T, \frac{1}{T}\sum_{t=1}^T\EE_{\pi^t}\left[r^{\theta^t}(x,y) \right]- \frac{\beta}{T}\sum_{t=1}^T \DD_{\mathrm{KL}}(\pi^t\|\pi_{\mathrm{ref}})\cdot \mathbf{1}^m)\right)^{2q}}}_{(\ast\ast)}\\&\qquad\qquad  + \widetilde{\cO}\left((\gamma^{-1}\mathrm{poly}(\exp(1/\beta), m, e^B, d, \log(1/\delta), \min_{n \in [N]}\frac{1}{p^{(n)}}, B_1)N^{1/2q}T^{-1/2q}\right) + (\ast).
\end{align*}
\end{small}
First, for the term $\sum_{n=1}^N d(W_n^*, W_n^T),$ we can bound it by 
\begin{align*}
    \sum_{n=1}^N d(W_n^*, W_n^T) \le \sum_{n=1}^N \frac{m^{3/2}B_1}{|p^{(n)}|}\cdot \|\alpha^{*,(n)}-\alpha^{T,(n)}\|_\infty.
\end{align*}
From the Theorem \ref{thm:est of alpha}, we can get 
\begin{align}
    \sum_{n=1}^N d(W_n^*, W_n^T) \le \frac{m^{3/2}B_1N}{\min_{n \in [N]}p^{(n)}} \gamma^{-1} \cdot \widetilde{\cO}\left(\mathrm{poly}(m,e^B,\exp(1/\beta), d,\log(1/\delta))\right)\cdot \frac{1}{\sqrt{T}},\label{eq:malfare online step 1}
\end{align}
Now by Eq.~\eqref{ineq:(A)}, we have
\begin{align}
    \mathrm{(\ast\ast)}&\le \sqrt[2q]{\sum_{n=1}^N \zeta_n \left(\sum_{i=1}^m\Bigg|\frac{1}{T}\sum_{t=1}^T \EE_{\pi^t}[r_i^*(x,y) -\hat{r}_i^t(x,y)]\Bigg|\right)^{2q}} \le \left(\sum_{n=1}^N \zeta_n\right)\cdot \sum_{i=1}^m\Bigg|\frac{1}{T}\sum_{t=1}^T  \EE_{\pi^t}[r_i^*(x,y) -\hat{r}_i^t(x,y)]\Bigg|\nonumber\\
    &\le \sum_{i=1}^m \mu_i \exp(4/\beta)\kappa\cdot \sum_{t=1}^T \sum_{j=1}^{t-1}\EE_{y_1,y_2\sim \pi^j}[\Delta_i^t(x,y)^2] + \frac{d_{\mathrm{cover}}(1/T)}{4\mu_i} + \widetilde{\cO}(NBd).\label{malfare online step 2}
\end{align}
Consider the term (B). By Eq.~\eqref{eq:lossgap1}, we can get
\begin{align}
    \mathrm{(B)}=\frac{1}{T}\sum_{t=1}^T \sum_{i=1}^m \eta \|d^t\|_1(L_i^t(\theta_i^*) - L_i^t(\theta_i^t)) &\le \widetilde{\cO}(2m\eta \sqrt{m}d) - \frac{\eta C' \|d^t\|_1}{Tme^B}\sum_{i=1}^m \sum_{t=1}^T \sum_{j=1}^{t-1}\EE_{y_1,y_2\sim \pi^j} [\Delta_i^t(x,y)^2]\nonumber\\
    &\le \widetilde{\cO}(2m\eta \sqrt{m}d) - \frac{\eta C' \zeta_n^{{1/2q}}}{Tme^BN^{\frac{2q-1}{2q}}}\sum_{i=1}^m \sum_{t=1}^T \sum_{j=1}^{t-1}\EE_{y_1,y_2\sim \pi^j} [\Delta_i^t(x,y)^2].\label{eq:ineq (B)}
\end{align}
The last inequality is because, if we choose $n' = \max_{n \in [N]}\zeta_n \|W^{(n)}-\overline{V}^t\|_2^{2q}$, then 
\begin{align*}
    \frac{\zeta_{n'}\|W^{(n')}- \overline{V}^t\|_2^{2q-1}}{\left(\sum_{n=1}^N \zeta_n\|W^{(n)}- \overline{V}^t\|_2^{2q}\right)^{\frac{2q-1}{2q}}} \ge \frac{\zeta_{n'}\|W^{(n')}- \overline{V}^t\|_2^{2q-1}}{N^{\frac{2q-1}{2q}}\cdot \zeta_{n'}^{\frac{2q-1}{2q}}\|W^{(n')}-\overline{V}^t\|_2^{2q-1}} = \frac{\zeta_{n'}^{1/2q}}{N^{\frac{2q-1}{2q}}}\ge \frac{\min_{n \in [N]}\zeta_n^{1/2q}}{N^{\frac{2q-1}{2q}}}.
\end{align*}
Hence, we have 
\begin{align*}
    \|d^t\|_1 = \left\|\sum_{n=1}^Nd_n^t\cdot \frac{\zeta_n\|W^{(n)}- \overline{V}^t\|_2^{2q-1}}{\left(\sum_{n=1}^N \zeta_n\|W^{(n)}- \overline{V}^t\|_2^{2q}\right)^{\frac{2q-1}{2q}}}\right\|_1 \ge \|d_{n'}^t\|_1\cdot \frac{\min_{n \in [N]}\zeta_n^{1/2q}}{N^{\frac{2q-1}{2q}}}.
\end{align*}
Now we choose $\mu_i = \eta \cdot \frac{C'\min_{n \in [N]}\zeta_n^{1/2q}}{me^B \exp(4/\beta)\kappa N^{\frac{2q-1}{2q}}}$, $\eta = 1/\sqrt{T}$, and use the inequality Eq.~\eqref{eq:ineq (B)} for bounding (B), we finally get 
\begin{align*}
    &D_q(\tilde{\pi}^T)-D_q(\pi^*) \\&\le \sum_{n=1}^N d(W_n^*, W_n^T) + \mathrm{(A)} + \mathrm{(B)} + \widetilde{\cO}\left((\gamma^{-1}\mathrm{poly}(\exp(1/\beta), m, e^B, d, \log(1/\delta), \min_{n \in [N]}\frac{1}{p^{(n)}}, B_1)N^{1/2q}T^{-1/2q}\right)\\
    &\le \gamma^{-1}\mathrm{poly}(\exp(1/\beta), m,N,e^B, d, \log(1/\delta),\kappa, B_1, (\min_{n \in [N]}p^{(n)})^{-1}, (\min_{n \in [N]}\zeta_n)^{-1/2q})\cdot \widetilde{\cO}(T^{-1/2q}).
\end{align*}

\end{proof}

\subsection{Estimation of $\alpha$}
In this subsection, we give a theoretical upper bound of the estimation error of $\alpha$. 
\begin{theorem}[Estimation of $\alpha$]\label{thm:est of alpha}
    Assume that we execute the Algorithm \ref{alg: vpo-fl-general} with the Assumption \ref{assum:gap}, then for each $t \in [T],$ with probability at least $1-\delta$ we have 
    \begin{align}
        \|\alpha^* - \alpha^t\|_\infty \le \frac{1}{t}\sum_{k=1}^{t}\|\alpha^* - \hat{\alpha}^k\|_\infty\le \gamma^{-1}\cdot \widetilde{\cO}\left(\mathrm{poly}(m,e^B, \exp(1/\beta), d,\log(1/\delta))\right)\cdot \frac{1}{\sqrt{t}}.\label{eq:estimate alpha final result}
    \end{align}
\end{theorem}
\begin{proof}
First,  for any $k \in [t],$  we estimate $\hat{\alpha}$ with $\tilde{\theta}_1^k, \tilde{\theta}_2^k, \cdots, \tilde{\theta}_m^k$, where $\tilde{\theta}_i^k = \argmin_\theta L_i^k(\theta)$ only minimizes the log-likelihood loss without optimistic exploration. Define $\delta_i^k(x,y)=\left|r_i^{\tilde{\theta}_i^k}(x,y_1) - r_i^{\tilde{\theta}_i^k}(x,y_2) - (r_i^*(x,y_1) - r_i^*(x,y_2))\right|$.
then, by theorem \ref{thm:mle}, with probability $1-\delta$ we have 
\begin{align*}
    \sum_{s=1}^{k-1} \EE_{x, y\sim \cD_s}&\left\| \textrm{Softmax}(\hat{\alpha}_i^k \cdot |r_i^{\tilde{\theta}_i^k}(x,y_1)-r_i^{\tilde{\theta}_i^k}(x,y_2)|)-\textrm{Softmax}(\alpha_i^*\cdot |r_i^*(x,y_1)-r_i^*(x,y_2)|)\right\| _{\mathrm{TV}}^2\\
    &\le 2 \log (d_\cF(1/k^2)/\delta)+1/k,
\end{align*}
    where $\cF = \{\mathrm{Softmax}(x_i)\mid 1\le i\le m, x_i \le 1\}$, and the log of $\varepsilon-$covering number $\log (d_\cF(1/k^2)) = \widetilde{\cO}(m)$.

    thus by the Cauchy's inequality, we can get 
    \begin{align}
        &\sqrt{2k\log(d_\cF(1/k^2)/\delta)+1}\nonumber \\&\ge \sum_{s=1}^{k-1} \EE_{x, y\sim \cD_s}\left\| \textrm{Softmax}(\hat{\alpha}_i^k \cdot |r_i^{\tilde{\theta}_i^k}(x,y_1)-r_i^{\tilde{\theta}_i^k}(x,y_2)|)-\textrm{Softmax}(\alpha_i^*\cdot |r_i^*(x,y_1)-r_i^*(x,y_2)|)\right\| _{\mathrm{TV}}\nonumber\\
        &\ge \sum_{s=1}^{k-1} \EE_{x, y\sim \cD_s}\left\| \textrm{Softmax}(\hat{\alpha}_i^k \cdot |r_i^*(x,y_1)-r_i^*(x,y_2)|)-\textrm{Softmax}(\alpha_i^*\cdot |r_i^*(x,y_1)-r_i^*(x,y_2)|)\right\| _{\mathrm{TV}}\nonumber\\
        &\qquad - \sum_{s=1}^{k-1} \EE_{x, y\sim \cD_s}\left\| \textrm{Softmax}(\hat{\alpha}_i^k \cdot |r_i^{\tilde{\theta}_i^k}(x,y_1)-r_i^{\tilde{\theta}_i^k}(x,y_2)|)-\textrm{Softmax}(\hat{\alpha}_i^k\cdot |r_i^*(x,y_1)-r_i^*(x,y_2)|)\right\| _{\mathrm{TV}}.\label{alpha:step1}
    \end{align}

    Now we bound the difference of $\alpha$ based on the difference of the softmax distribution.

    Fixed $k$, since the upper bound of $0\le r_i^{\tilde{\theta}_i^{k}}(x,y) \le B$ and $0\le r^*(x,y) \le B$, define $X_i = |r_i^{\tilde{\theta}_i^k}(x,y_1)-r_i^{\tilde{\theta}_i^k}(x,y_2)|\le B$ and $X_i^* = |r^*_i(x,y_1) - r^*_i(x,y_2)|\le B$, then 
    \begin{align*}
        &\left\|\textrm{Softmax}(\hat{\alpha}_i^k \cdot |r_i^{\tilde{\theta}_i^k}(x,y_1)-r_i^{\tilde{\theta}_i^k}(x,y_2)|)-\textrm{Softmax}(\hat{\alpha}_i^{k}\cdot |r_i^*(x,y_1)-r_i^*(x,y_2)|)\right\|_{\mathrm{TV}}\\
        &= \sum_i \left|\frac{e^{X_i\cdot \hat{\alpha}_i^k}}{\sum_j e^{X_i\cdot \hat{\alpha}_j^k}}-\frac{e^{X_j^*\cdot \hat{\alpha}_i^k}}{\sum_j e^{X_j^*\cdot \hat{\alpha}_j^k}}\right|\\
        &=\sum_i \left|\frac{\sum_{j\neq i}e^{X_j^*\cdot \hat{\alpha}_j^k+X_i\hat{\alpha}_i^k }-e^{X_j\cdot \hat{\alpha}_j^k+X_i^*\hat{\alpha}_i^k }}{(\sum_j e^{X_j\cdot \hat{\alpha}_j^k})(\sum_j e^{X_j^*\cdot \hat{\alpha}_j^k})}\right|\\
        &\le \sum_i \left|\frac{\sum_{j\neq i}e^{X_j^* \hat{\alpha}_j^k+X_i\hat{\alpha}_i^k}(e^{\delta_j^{k} \hat{\alpha}_j^k + \delta_i^{k} \hat{\alpha}_i^k}-1)}{m^2}\right|,\end{align*}
where the last inequality uses the fact that $\sum_{j}e^{X_j\cdot \hat{\alpha}_i^k} \ge m$ and $\sum_j e^{X_j^* \hat{\alpha}_j^k} \ge m$.
Now since $e^{X_j^* \hat{\alpha}_j^k + X_i \hat{\alpha}_i^k}\le e^{B(\hat{\alpha}_i^k + \hat{\alpha}_j^k)}\le e^B$, and $e^a-1\le e^B\cdot a$ for every $0\le a \le B$, we can have 
        \begin{align*}
        &\le \sum_i \left|\frac{\sum_{j\neq i}e^{2B}(\delta_j^{k} \hat{\alpha}_j^k + \delta_i^{k} \hat{\alpha}_i^k)}{m^2}\right|\\
        &\le \frac{e^{2B}}{m^2}\sum_i \sum_{j\neq i}(\delta_j^{k} \hat{\alpha}_j^k + \delta_i^{k} \hat{\alpha}_i^k)\\
        &\le \frac{e^{2B}}{m}\sum_i \delta_i^{k} \hat{\alpha}_i^k.
    \end{align*}

    Now 
    choose the index $l = \argmax_{i \in [m]} X_i^* \circ |\alpha_{i}^* - \hat{\alpha}_i^k|$, and WLOG, assume $\hat{\alpha}_{l}^k = \alpha_{l}^* + \varepsilon$,
    then 
    we can bound 
    \begin{align*}
        &\left\|\textrm{Softmax}(\hat{\alpha}_{l}^k \cdot |r_{l}^*(x,y_1)-r_{l}^*(x,y_2)|)-\textrm{Softmax}(\alpha_{l}^*\cdot |r_{l}^*(x,y_1)-r_{l}^*(x,y_2)|)\right\|_{\mathrm{TV}}\\
        &\ge \left|\frac{e^{X_{l}^*\hat{\alpha}_l^k}}{\sum_j e^{X_j^*\hat{\alpha}_j^k}} - \frac{e^{X_l^*\cdot \alpha_l^*}}{\sum_j e^{X_j^*\alpha_j^*}}\right|\\
        &= \left|\frac{e^{X_l^*(\alpha_l^* + \varepsilon)}}{e^{X_l^*(\alpha_l^* + \varepsilon)}+\sum_{j\neq l} e^{X_j^*\hat{\alpha}_j^k}} - \frac{e^{X_l^*\cdot \alpha_l^*}}{e^{X_l^* \alpha_l^*}+\sum_{j\neq l} e^{X_j^*\alpha_j^*}}\right|\\
        &=\left|\frac{\sum_{j\neq l}e^{X_j^* \alpha_j^* + X_l^*(\alpha_l^* + \varepsilon)}-e^{X_j^* \hat{\alpha}_j^k + X_l^* \alpha_l^*}}{(\sum_j e^{X_j^*\hat{\alpha}_j^k})(\sum_j e^{X_j^*\alpha_j^*})}\right|.
    \end{align*}
    Now by the selection of the $l$, we can have 
    $$X_j^* \alpha_j^* + X_l^*(\alpha_l^* + \varepsilon) \ge X_j^* \hat{\alpha}_j^k + X_l^* \alpha_l^*,$$
    hence 
    $$e^{X_j^* \alpha_j^* + X_l^*(\alpha_l^* + \varepsilon)}\ge e^{X_j^* \hat{\alpha}_j^k + X_l^* \alpha_l^*}.$$
    Also, since $\sum_i\alpha_i^* = \sum_i \hat{\alpha}_i^k = 1,$ and the fact that $\hat{\alpha}_l^k = \alpha_l^* + \varepsilon,$ we can further derive 
    $$\sum_{j\neq l} \alpha_j^* = \sum_{j\neq l} \hat{\alpha}_j^k + \varepsilon.$$

    then at least one $j'\neq l$ such that $\alpha_{j'}^* \ge \hat{\alpha}_{j'}^k + \varepsilon/m$.
    then 
    \begin{align*}
        e^{X_{j'}^* \alpha_{j'}^*+ X_l^*(\alpha_l^*+\varepsilon)}-e^{\hat{\alpha}_{j'}^k X_{j'}^* + X_l^*\alpha_l^*}&\ge ^{X_{j'}^* \hat{\alpha}_{j'}^k+ X_l^*(\alpha_l^*+\varepsilon)}-e^{\hat{\alpha}_{j'}^k X_{j'}^* + X_l^*\alpha_l^*}
        \\&\ge e^{X_l^*(\alpha_l^* + \varepsilon)} - e^{X_l^* \alpha_l^*}\\
        &\ge e^{\alpha_l^* X_l^*}(e^{\varepsilon X_l^*}-1)\\
        &\ge e^{\alpha_l^* X_l^*}\cdot \varepsilon X_l^*.
    \end{align*}
    Thus,
    \begin{align*}
        &\left\|\textrm{Softmax}(\hat{\alpha}_l^k \cdot |r_l^*(x,y_1)-r_l^*(x,y_2)|)-\textrm{Softmax}(\alpha_l^*\cdot |r_l^*(x,y_1)-r_l^*(x,y_2)|)\right\|_{\mathrm{TV}}\\
        &\ge \frac{e^{\alpha_l^* X_l^*}}{(\sum_j e^{X_j^*\hat{\alpha}_j^k})(\sum_j e^{X_j^*\alpha_j^*})} \cdot \varepsilon X_l^*\\
        &\ge \frac{1}{(me^B)^2}  \cdot |\hat{\alpha}_l^k - \alpha_l^*| X_l^*.
    \end{align*}
    Now define $X^* =  (X_1^*, X_2^*, \cdots, X_m^*)^\top \in \RR^m$ and $|\alpha^* - \hat{\alpha}^k| = (|\alpha_1^*-\hat{\alpha}_1^k|, \cdots, |\alpha_m^*-\hat{\alpha}_m^k|)^\top \in \RR^m.$ We can get 
    $$\|X^*\circ |\alpha^* - \hat{\alpha}^k|\|_\infty\le m^2e^{2B} \left\|\textrm{Softmax}(\hat{\alpha}_i^k \cdot |r_i^*(x,y_1)-r_i^*(x,y_2)|)-\textrm{Softmax}(\alpha_i^*\cdot |r_i^*(x,y_1)-r_i^*(x,y_2)|)\right\|_{\mathrm{TV}},$$
    where $X\circ Y$ denotes the Hadamard product.
    then take the expectation we can get 
    \begin{align*}
        &\EE_{x,y\sim \cD_s}\|X^*\circ |\alpha^* - \hat{\alpha}^k|\|_\infty\\&\qquad \le m^2e^{2B} \EE_{x,y\sim \cD_s}\left\|\textrm{Softmax}(\hat{\alpha}_i^k \cdot |r_i^*(x,y_1)-r_i^*(x,y_2)|)-\textrm{Softmax}(\alpha_i^*\cdot |r_i^*(x,y_1)-r_i^*(x,y_2)|)\right\|_{\mathrm{TV}}.
    \end{align*}

     Hence, by Eq.\eqref{alpha:step1}, we have
    \begin{align*}
        &\sqrt{2k\log(d_\cF(1/k^2)/\delta)+1}\\
        &\ge \sum_{s=1}^{k-1} \EE_{x, y\sim \cD_s}\left\| \textrm{Softmax}(\hat{\alpha}_i^k \cdot |r_i^*(x,y_1)-r_i^*(x,y_2)|)-\textrm{Softmax}(\alpha_i^*\cdot |r_i^*(x,y_1)-r_i^*(x,y_2)|)\right\| _{\mathrm{TV}}\\
        &\qquad - \sum_{s=1}^{k-1} \EE_{x, y\sim \cD_s}\left\| \textrm{Softmax}(\hat{\alpha}_i^k \cdot |r_i^{\tilde{\theta}_i^k}(x,y_1)-r_i^{\tilde{\theta}_i^k}(x,y_2)|)-\textrm{Softmax}(\hat{\alpha}_i^k\cdot |r_i^*(x,y_1)-r_i^*(x,y_2)|)\right\| _{\mathrm{TV}}\\
        &\ge \sum_{s=1}^{k-1}\EE_{x, y\sim \cD_s}\frac{1}{m^2e^{2B}}\|X^*(x,y)|\alpha^* - \hat{\alpha}^k|\|_\infty \\
        &\qquad - \sum_{s=1}^{k-1}\EE_{x, y\sim \cD_s}\left\| \textrm{Softmax}(\hat{\alpha}_i^k \cdot |r_i^{\tilde{\theta}_i^k}(x,y_1)-r_i^{\tilde{\theta}_i^k}(x,y_2)|)-\textrm{Softmax}(\hat{\alpha}_i^k\cdot |r_i^*(x,y_1)-r_i^*(x,y_2)|)\right\| _{\mathrm{TV}} \\
        &\ge \sum_{s=1}^{k-1}\EE_{x, y\sim \cD_s}\frac{1}{m^2e^{2B}}\|X^*(x,y)\cdot |\alpha^* - \hat{\alpha}^k|\|_\infty - \frac{e^{2B}}{m}\sum_{s=1}^{k-1} \EE_{x,y\sim \cD_s}[\delta_i^{k}(x,y)\hat{\alpha}_i^k].
    \end{align*}
    Hence we finally get 
    \begin{align}
        \sum_{s=1}^{k-1}\EE_{x, y\sim \cD_s}\|X^*(x,y)\circ|\alpha^* - \hat{\alpha}^k|\|_\infty &\le m^2e^{2B}\left(\sqrt{2k\log(d_\cF(1/k^2)/\delta)+1} + \frac{e^{2B}}{m} \sum_{s=1}^{k-1}\sum_{i=1}^m \delta_i^{k}(x^s,y^s) \hat{\alpha}_i^k\right)\nonumber\\
        & = \mathrm{poly}(m, \exp(B))\cdot \widetilde{\cO}\left(\sqrt{km\log (1/\delta))} + \sum_{s=1}^{k-1}\sum_{i=1}^m \EE_{y_1,y_2\sim \pi^s}[\delta_i^{k}(x,y)]\hat{\alpha}_i^k\right).\label{alpha:step2}
    \end{align}
    the last inequality holds by Azuma-Hoeffding's inequality with probability at least $1-\delta$. 
    Now by Lemma \ref{lemma:policy diff}, we can get  $\sup_{s,x,y}\frac{\pi^*(y\mid x)}{\pi^s(y\mid x)} \le  \exp(4/\beta)$ and $\sup_{x,y}\frac{\pi_{\mathrm{ref}}(y\mid x)}{\pi^s(y\mid x)} \le \exp(4/\beta)$, we can get 
    \begin{align}
    \gamma(k-1) \|\alpha^* - \hat{\alpha}^{k}\|_\infty
        &\le (k-1) \EE_{y_1\sim \pi^*,y_2\sim \pi_{\mathrm{ref}}}\|X^*(x,y)\circ|\alpha^* - \hat{\alpha}^k|\|_\infty\nonumber\\
        &\le  \exp(8/\beta)
        \sum_{s=1}^{k-1}\EE_{y_1,y_2\sim \pi^s}\|X^*(x,y)\circ|\alpha^* - \hat{\alpha}^s|\|_\infty. \label{alpha:step3}
        \end{align}
The first inequality uses the Assumption \ref{assum:gap} that $\EE_{y_1\sim \pi^*, y_2\sim \pi_{\mathrm{ref}}}[X_i^*(x,y)] \ge \gamma $. Now combining Eq.\eqref{alpha:step2} and Eq.\eqref{alpha:step3}, we can further get 
        \begin{align}
        \gamma(k-1)\|\alpha^* - \hat{\alpha}^k\|_\infty &\le  \exp(8/\beta)\cdot \mathrm{poly}(m, \exp(B))\cdot \widetilde{\cO}\left(\sqrt{km\log (1/\delta))} +  \sum_{s=1}^{k-1}\sum_{i=1}^m \EE_{y_1,y_2\sim \pi^s}[\delta_i^{k}(x,y)]\hat{\alpha}_i^k\right)\nonumber\\
        &\le \exp(8/\beta)\cdot \mathrm{poly}(m, \exp(B))\cdot \widetilde{\cO}\left(\sqrt{km\log (1/\delta))} +  \sum_{s=1}^{k-1}\sum_{i=1}^m \EE_{y_1,y_2\sim \pi^s}[\delta_i^{k}(x,y)]\right).\label{eq:estimate alpha step 1}
    \end{align}
    Now we further derive the  final result.
    Frist, by $\alpha^t = \frac{1}{t}\sum_{k=1}^t \hat{\alpha}^k$, we can get 
    \begin{align}
        \|\alpha^* - \alpha^t\|_\infty &\le \frac{1}{t}\sum_{k=1}^{t}\|\alpha^* - \hat{\alpha}^k\|_\infty\nonumber\\
        &\le \frac{\gamma^{-1}\exp(8/\beta)\cdot \mathrm{poly}(m,\exp(B))}{t}\cdot \sum_{k=1}^t\widetilde{\cO}\left(\frac{\sqrt{m\log(1/\delta)}}{\sqrt{k}} + \frac{1}{k} \sum_{s=1}^{k-1}\sum_{i=1}^m \EE_{y_1,y_2\sim \pi^s}[\delta_i^{k}(x,y)]\right).\label{eq:alpha1}
    \end{align}

    Now we derive the final result. First, we can get
\begin{align*}\delta_i^k(x^s,y^s) 
& = \left| \langle \tilde{\theta}_i^k-\theta_i^*, \phi_i(x^s,y_1^s) - \phi_i(x^s, y_2^s) \rangle\right|\\
& \le \|\tilde{\theta}_i^k - \theta_i^* \|_{\Sigma_{\cD_i^{k-1}}}\cdot \|\phi_i(x^s,y_1^s)- \phi_i(x^s, y_2^s)\|_{(\Sigma_{\cD_i^{k-1}})^{-1}},
\end{align*}
where $\cD_{i}^{k-1} = \{s \in [k-1]\mid I_s = i\}$ and $\Sigma_{\cD_i^{k-1}} = \sum_{s \in \cD_i^{k-1}}\phi_i(x^s,y^s)\phi_i(x^s,y^s)^\top $ is the covariance matrix. 
then by Lemma 3.1 in \cite{zhu2023principled}, we can get $\|\tilde{\theta}_i^k - \theta_i^*\|_{\Sigma_{\cD_i^{k-1}}} \le C(d, B, \delta) = \textrm{poly}(d,B, \log(1/\delta))$ for some constant $C(d,B,\delta)$, and then we can get 
\begin{align*}
    \delta_i^k(x^s,y^s) \le C(d,B,\delta)\cdot \|\phi_i(x^s,y_1^s)- \phi_i(x^s, y_2^s)\|_{(\Sigma_{\cD_i^{k-1}})^{-1}}.
\end{align*}
Now apply the same technique in Eq.\eqref{eq:estimate technique}, we can get 
\begin{align*}
    \sum_{i=1}^m \sum_{k=1}^t \sum_{s=1}^{k-1} \EE_{y_1,y_2\sim \pi^s} \frac{1}{k}[\delta_i^k(x,y)\hat{\alpha}_i^k] &\le me^B\sum_{k=1}^t \sum_{i=1}^m \sum_{s \in \cD_i^{k-1}}\EE_{y_1,y_2\sim \pi^s} \frac{1}{k}[\delta_i^k(x,y)\hat{\alpha}_{i}^k]\\
    & =me^B\sum_{i=1}^m\sum_{s=1}^t  \EE_{y_1,y_2\sim \pi^s} \sum_{k>s} \left[\frac{1}{k}\delta_{I_s}^k(x,y)\hat{\alpha}_{I_s}^k\right]\\
    & \le me^B\sum_{s=1}^t  \EE_{y_1,y_2\sim \pi^s} \sum_{k>s} \left[\frac{1}{k}\delta_{I_s}^k(x,y)\right].
\end{align*}
The second line is because that, the summation is over
\begin{align*}\{(k,i,s)\mid k \in [t],i \in [m] ,s \in \cD_i^{k-1}\}&=\{(k,i,s)\mid k \in [t], i \in [m], s \le k-1, I^s = i\}\\
&=\{(k,i,s)\mid s \in [t], k > s, i = I^s\}.\end{align*}
the last inequality uses the fact that $\hat{\alpha}_{I_s}^k\le 1$.
then we can use the Azuma-Hoeffding's inequality to further get 
\begin{align}
    \sum_{i=1}^m \sum_{k=1}^t \sum_{s=1}^{k-1} \EE_{y_1,y_2\sim \pi^s} \frac{1}{k}[\delta_i^k(x,y)\hat{\alpha}_i^k] &\le me^B\sum_{k=1}^t  \sum_{k\ge  s} \left[\frac{1}{k}\delta_{I_s}^k(x^s,y^s)\right] + \cO(\sqrt{t}\log (t/\delta))\nonumber\\&\le me^B\sum_{k=1}^t  \sum_{k\ge  s} \left[\frac{1}{k}C(d,B,\delta)\cdot \|\phi_{I_s}(x^s,y_1^s) - \phi_{I_s}(x^s, y_2^s)\|_{\Sigma_{\cD_{I_s}^{k-1}}}\right]\nonumber\\&\qquad  + \cO(\sqrt{t}\log (t/\delta))\label{eq:alpha2}
\end{align}
with probability at least $1-\delta$.
Now to present the proof in a simple way, we simplify $\Sigma_{\cD_{I_s}^{k-1}} $ as $\Sigma^{k-1,(I_s)}$.
We will have 
\begin{align}
    & me^B\sum_{s=1}^t \sum_{k>s} \left[\frac{1}{k}\cdot C(d,B,\delta) \cdot \|\phi_{I_s}(x^s,y_1^s)- \phi_{I_s}(x^s, y_2^s)\|_{(\Sigma^{k-1,(s)})^{-1}}\right]\nonumber\\
    &\le  me^B\sum_{s=1}^t \sum_{k>s}\frac{1}{k}\cdot C(d,B,\delta) \cdot \|\phi_{I_s}(x^s,y_1^s)- \phi_{I_s}(x^s, y_2^s)\|_{(\Sigma^{s,(s)})^{-1}}\nonumber\\
    &\le  me^B\sum_{s=1}^t C(d,B,\delta)\|\phi_{I_s}(x^s,y_1^s)- \phi_{I_s}(x^s, y_2^s)\|_{(\Sigma^{s,(s)})^{-1}}\sum_{k>s}^t \frac{1}{k}\nonumber\\
    &\le \frac{\log t}{\kappa_1} \cdot \sum_{s=1}^t C(d,B,\delta)\|\phi_{I_s}(x^s,y_1^s)- \phi_{I_s}(x^s, y_2^s)\|_{(\Sigma^{s,(s)})^{-1}}.\label{eq:alpha3}
    \end{align}
Now, we can decompose $\{1,2,\cdots, t\}$ into $m$ different set $\cD_i = \{s \in [t]: I_s = i\}$. then, we fixed $i$ and denote $M_s=\|\phi_i(x^s,y_1^s)-\phi_i(x^s,y_2^s)\|^2_{(\Sigma_{\cD_{I_s}^{s}})^{-1}}$ with $\|M_s\| \le B^2$, by Cauchy's inequality,
\begin{align}
    &\sum_{s \in \cD_i}\|\phi_{I_s}(x^s,y_1^s)- \phi_{I_s}(x^s, y_2^s)\|_{(\Sigma^{s,(s)})^{-1}}\nonumber\\
    &\le \sqrt{t}\sqrt{\sum_{s \in \cD_i}M_s}\nonumber\\
    &\le \sqrt{t}\sqrt{\sum_{s \in \cD_i}M_s\II\{M_s\le 1\}} + \sqrt{t}\sqrt{\sum_{s \in \cD_i}M_s \II\{M_s > 1\}}\nonumber\\
    &\le \sqrt{t}\cdot \left(\sqrt{\sum_{s \in \cD_i}\min\{1,M_s\}} + \sqrt{B^2\sum_{s \in \cD_i}\II\{M_s>1\}}\right)\nonumber\\
    &\le \widetilde{\cO}(Bd\sqrt{t}).\nonumber
\end{align}
Then, by summing over $i \in [m]$, we can get 
    \begin{align}
    &\frac{\log t}{\kappa_1} \cdot \sum_{s=1}^t C(d,B,\delta)\|\phi_{I_s}(x^s,y_1^s)- \phi_{I_s}(x^s, y_2^s)\|_{(\Sigma_{\cD_i}^{j})^{-1}}\nonumber\\
    &\le \frac{\log t}{\kappa_1} \cdot C(d,B,\delta) \cdot m \cdot \widetilde{\cO}(Bd\sqrt{t})\nonumber\\
    & = \widetilde{\cO}(m^2 e^B\cdot Bd\cdot C(d,B,\delta)\sqrt{t}). \label{ineq: bound(C)}
    \end{align}
    Now combining Eq.\eqref{eq:alpha1}, Eq.\eqref{eq:alpha2}, Eq.\eqref{eq:alpha3} and Eq.\eqref{ineq: bound(C)}, we can finally get 
    \begin{align*}
        \|\alpha^* - \alpha^t\|_\infty \le \frac{1}{t}\sum_{k=1}^{t}\|\alpha^* - \hat{\alpha}^k\|_\infty&\le \gamma^{-1}\exp(8/\beta)\cdot \widetilde{\cO}\left(\mathrm{poly}(m,e^B, d,\log(1/\delta))\right)\cdot \frac{1}{\sqrt{t}}\\
        &=\gamma^{-1}\cdot \widetilde{\cO}\left(\mathrm{poly}(m,e^B, \exp(1/\beta),d,\log(1/\delta))\right)\cdot \frac{1}{\sqrt{t}}
    \end{align*}
    with probability at least $1-3\delta$. By substituting $\delta/3$ with $\delta$, we complete the proof.
 \end{proof}

\section{Error of Estimating the Target Set} 
First we provide a lemma to show that the projection on $W^*$ is also bounded. 
\begin{lemma}\label{lemma: bounded proj}
    Fixed the requirement $p^{(n)}, c^{(n)}$ for all $k \in [K]$. For  any importance weight $\{\alpha^{(n)}\}_{k \in [K]}$ such that $\alpha^{(n)}\succeq 0$ and $\|\alpha^{(n)}\|_1=1$ for all $k \in [K]$, for $B_1 = 2\sqrt{m}(B+\max_n c^{(n)})$, we have 
    $$\|\Pi_{W^*}(x)\|_\infty\le B_1, \ \ \ W^* = \bigcap_{i=1}^K W^{\alpha^{(n)}}_{p^{(n)},c^{(n)}}$$
    holds for all $\|x\|_\infty\le B.$
\end{lemma}

\begin{proof}
    Suppose we choose any $y \in W^*$, then by the definition of projection, we can get $$\|\Pi_{W^*}(x)\|_\infty-\sqrt{m}B\le \|x - \Pi_{W^*}(x)\|_\infty \le \|x - \Pi_{W^*}(x)\|_2 \le \|x - y\| \le \sqrt{m}B + \|y\|,$$
    which induces $$\|\Pi_{W^*}(x)\| \le 2\sqrt{m}B + \|y\|.$$
    Now consider $y = (z, \cdots, z)^\top \in \RR^m$, when $z= \max_n c^{(n)}$, for any $\alpha^{(n)}$
    \begin{align*}
        \left(\sum_{i=1}^m \alpha_i^{(n)} y_i^{|p^{(n)}|}\right)^{1/p^{(n)}} = z\cdot \left(\sum_{i=1}^m \alpha_i^{(n)}\right)^{1/p^{(n)}}= z \ge c^{(n)}.
    \end{align*}
    That means $y \in W^{\alpha^{(n)}}_{p^{(n)}, c^{(n)}}$ and then $y \in W^*$ for any $k \in [K].$ Hence we have 
    \begin{align*}
        \|\Pi_{W^*}(x)\| \le 2B + \|y\| \le 2\sqrt{m}(B+\max_n c^{(n)}).
    \end{align*}
    We complete the proof of lemma.
\end{proof}
Now we consider the estimation of the $W^*$. First, we consider the estimation error of $W^\alpha$ when we have an estimation error of $\alpha.$ The following lemma tells us the estimation error of parameterized target set.

 \begin{lemma}[Estimation error of parameterized target set]\label{lemma:estimation error of parameterized target set}
     Suppose we have two different $\alpha, \alpha'$, the distance between $W_{p,c}^{\alpha}$ and $W_{p,c}^{\alpha'}$ can be bounded by 
     \begin{equation*}
         d_B(W_{p,c}^{\alpha}, W_{p,c}^{\alpha'}) \le \frac{m^{3/2}B\|\alpha-\alpha'\|_\infty}{|p|},
     \end{equation*}
     where $$d_B(S,S') = \max\left\{\max_{x \in S, \|x\|_\infty \le B}d(x,\Pi_{S'}(x)), \max_{x \in S', \|x\|_\infty \le B}d(x,\Pi_{S}(x))\right\}$$ represents the distance of two sets $S$ and $S'$ restricted to some bounded set.
 \end{lemma}

 \begin{proof}
     Suppose $p \in [0,1]$ and $x \in W_{p,c}^{\alpha}$ with $\|x\|_\infty \le B$, then we have 
     $$\sum_{i=1}^m \alpha_i x_i^p \ge c^p.$$
First, if $\sum_{i=1}^m \alpha_i' x_i^p \ge c^p,$ then $x \in W_{p,c}^{\alpha'}$ and the distance $d(x, \Pi_{W_{p,c}^{\alpha'}}(x))  = 0 .$
Now we consider the auxillary vector $y \in \RR^m$ where $y_i = x_i^p$ for $i \in [m].$ Then $\sum_{i=1}^m \alpha_i y_i \ge c^p.$ By the formula of the distance between one point to a line, the distance between $y$ and $W_{p,c}^{\alpha'} = \{y: \sum_{i=1}^m \alpha_i y_i \ge c^p, y_i \succeq 0\}$ can have the following upper bound:
\begin{equation*}
    d(y, \Pi_{W_{p,c}^{\alpha'}}(y)) = \frac{\max\{c^p - \sum_{i=1}^m\alpha_i'y_i, 0\}}{\sqrt{\sum_{i=1}^m (\alpha_i')^2}} \le \frac{\max\{\sum_{i=1}^m (\alpha_i-\alpha_i')y_i,0\}}{\sqrt{\sum_{i=1}^m (\alpha_i')^2}}\le \frac{\|\alpha-\alpha'\|_\infty mB^p}{\sqrt{\sum_{i=1}^m (\alpha_i')^2}}.
\end{equation*}
Now consider $p<0$ we have $\sum_{i=1}^m \alpha_i x_i^p \le c^p.$ If $\sum_{i=1}^m\alpha_i' y_i \le c^p,$ then $x \in W_{p,c}^{\alpha'}$ and the distance $d(x, \Pi_{W_{p,c}^{\alpha'}}(x)) = 0.$ Otherwise, note that we can rewrite $W_{p,c}^{\alpha} = \{y:\sum_{i=1}^m \alpha_i y_i \le c^p, y\succeq 0\}.$ We have 
\begin{equation*}
    d(y, \Pi_{W_{p,c}^{\alpha'}}(y)) = \frac{\sum_{i=1}^m \alpha_i'y_i -c^p}{\sqrt{\sum_{i=1}^m (\alpha_i')^2}}\le \frac{\|\alpha-\alpha'\|_\infty \sum_{i=1}^m y_i}{\sqrt{\sum_{i=1}^m (\alpha_i')^2}}\le \frac{\|\alpha-\alpha'\|_\infty \cdot mB^p}{\sqrt{\sum_{i=1}^m (\alpha_i')^2}}.
\end{equation*}
So in both cases, we can find 
$$d(y, \Pi_{W_{p,c}^{\alpha'}}(y)) \le \frac{\|\alpha- \alpha'\|_\infty \cdot mB^p}{\sqrt{\sum_{i=1}^m (\alpha_i')^2}} \le \frac{\|\alpha- \alpha'\|_\infty \cdot mB^p}{1/\sqrt{m}} = m^{3/2}B^p \cdot \|\alpha-\alpha'\|_\infty.$$
Now since by Langarian mean value theorem we have $|x^p-y^p| \ge |pB^{p-1}||x-y|$, the distance between $x$ can be bounded by 
\begin{align*}
    d(x, \Pi_{W_{p,c}^{\alpha'}}(x)) \le \frac{1}{|pB^{p-1}|}d(y, \Pi_{W_{p,c}^{\alpha'}}(y)) \le \frac{m^{3/2}B^p \cdot \|\alpha-\alpha'\|_\infty}{|p|B^{p-1}}= \frac{m^{3/2}B\|\alpha-\alpha'\|_\infty}{|p|}.
\end{align*}
\end{proof}
The second lemma shows that the distance between the projection of one point on different convex set.
\begin{lemma}[Distance of Projections]\label{lemma:dis of proj}
    Fixed a point $x$ with $\|x\|_\infty \le B$. Suppose we have two convex sets $A_1,A_2$, then the distance of two projections can be bounded by 
    \begin{equation*}
        \|\Pi_{A_1}(x)-\Pi_{A_2}(x)\|_2^2 \le 4d(x, A_1)d_{B_1}(A_1,A_2)+2d_{B_1}(A_1,A_2)^2.
    \end{equation*}
\end{lemma}
 \begin{proof}
 WLOG, we can assume $d(x, A_1)\le d(x, A_2).$
     First, we consider $\Pi_{A_2}(\Pi_{A_1}(x)) \in A_2$ and $d(\Pi_{A_2}(\Pi_{A_1}(x)), \Pi_{A_1}(x)) \le d_{B_1}(A_1,A_2)$, where $B_1$ is from the bounded assumption of the target set.
     Now we only need to consider $d(\Pi_{A_2}(\Pi_{A_1}(x)), \Pi_{A_2}(x))$.
     Since $A_2$ is a convex set and $\Pi_{A_2}(\Pi_{A_1}(x)) \in A_2$, we can have 
     $$\langle x - \Pi_{A_2}(x), \Pi_{A_2}(x) - \Pi_{A_2}(\Pi_{A_1}(x))\ge 0,$$ then it is easy to get
     \begin{equation*}
         d(\Pi_{A_2}(\Pi_{A_1}(x)),x)^2 \ge d(x, A_2)^2 + d(\Pi_{A_2}(\Pi_{A_1}(x)), \Pi_{A_2}(x))^2.
     \end{equation*}
     Also, by the triangle inequality, we can derive 
     \begin{equation*}
         d(\Pi_{A_2}(\Pi_{A_1}(x)),x) \le d(x, A_1) + d(\Pi_{A_1}(x), \Pi_{A_2}(\Pi_{A_1}(x)))\le d(x, A_1) + d_{B_1}(A_1,A_2).
    \end{equation*}
    By combining these two inequality we can get 
    \begin{align*}
        d(\Pi_{A_2}(\Pi_{A_1}(x)), \Pi_{A_2}(x))^2 \le 2d(x, A_1)d_{B_1}(A_1,A_2) + d_{B_1}(A_1,A_2)^2.
    \end{align*}
     Hence we can finally get 
     \begin{align*}
         \|\Pi_{A_1}(x)- \Pi_{A_2}(x)\|_2^2 &\le 2d(\Pi_{A_2}(\Pi_{A_1}(x)), \Pi_{A_2}(x))^2 + 2d(\Pi_{A_2}(\Pi_{A_1}(x)), \Pi_{A_1}(x))^2 \\
         &\le 4d(x, A_1)d_{B_1}(A_1,A_2)+2d_{B_1}(A_1,A_2)^2.
     \end{align*}
 \end{proof}
Now we derive the difference between the direction. 
\begin{lemma}\label{lemma:direc}
    If the angle between the direction $ \frac{\Pi_{A_1}(x)-x}{d(x, A_1)}$ and $\frac{\Pi_{A_2}(x)-x}{d(x, A_2)}$ is less than $\pi/2$, then the  difference between them can be bounded by 
    $$\frac{\Pi_{A_1}(x)-x}{d(x, A_1)} - \frac{\Pi_{A_2}(x)-x}{d(x, A_2)}\le\frac{4\sqrt{d(x, A_1)d_{B_1}(A_1,A_2)}+2d_{B_1}(A_1,A_2)}{\max\{d(x, A_1), d(x, A_2)\}}. $$
\end{lemma}
\begin{proof}
    Denote the angle as $\Delta$ Consider the triangle $(x, \Pi_{A_1}(x), \Pi_{A_2}(x))$. By the law of sines, we can get 
    \begin{align*}
        \sin \Delta \le \frac{d(\Pi_{A_1}(x), d(\Pi_{A_2}(x)))}{\max\{d(x, A_1), d(x, A_2)\}}.
    \end{align*}
    By Lemma \ref{lemma:dis of proj}, we can get 
    \begin{align*}
         \sin \Delta \le \frac{2\sqrt{d(x, A_1)d_{B_1}(A_1,A_2)}+\sqrt{2}d_{B_1}(A_1,A_2)}{\max\{d(x, A_1), d(x, A_2)\}}. 
    \end{align*}
    Now since $\Delta \le \pi/2$ and the direction can be bounded by 
    \begin{align*}
        \frac{\Pi_{A_1}(x)-x}{d(x, A_1)} - \frac{\Pi_{A_2}(x)-x}{d(x, A_2)}&\le \frac{\sin \Delta}{\sin(\frac{\pi-\Delta}{2})}\le \sqrt{2}\sin \Delta\le \frac{4\sqrt{d(x, A_1)d_{B_1}(A_1,A_2)}+2d_{B_1}(A_1,A_2)}{\max\{d(x, A_1), d(x, A_2)\}}.
    \end{align*}
\end{proof}

\section{Auxiliary Lemmas}

\begin{lemma}[MLE Lemma]\label{thm:mle}
    We are given a dataset $D:=\{(x_i,y_i)\},$ where $x_i \sim \cD_i = \cD_i(x_{1:i-1}, y_{1:i-1})$ and $y_i \sim p(\cdot \mid x_i) = f^*(x_i,\cdot)$. Now if we calculate the MLE by 
    \begin{align*}
        \hat{f} = \argmax_{f \in \cF} \sum_{i=1}^n \log f(x_i,y_i),
    \end{align*}
    then fixed $\delta \in (0,1)$, assume $|\cF|<\infty$ and $f^* \in \cF$, then with probability at least $1-\delta$, we have 
    \begin{align*}
        \sum_{i=1}^n \EE_{x \in \cD_i}\left\|\hat{f}(x,\cdot)  - f^*(x,\cdot)\right\|_{\mathrm{TV}}^2 \le 2\log (|\cF|/\delta).
    \end{align*}
\end{lemma}
\begin{lemma}\label{lemma:policy diff}
    For any $\pi, \pi' \in \{\pi^1, \cdots, \pi^t, \pi^*, \pi_{\mathrm{ref}}\}$, we can have $$\sup_{x,y}\frac{\pi(y\mid x)}{\pi'(y\mid x)} \le \exp(4/\beta).$$
\end{lemma}
\begin{proof}
    First, note that $\pi$ and $\pi'$ are both optimal policy with respect to some reward $\hat{r}$, then $\pi$  can be rewritten as 
    \begin{gather*}
        \pi(y\mid x) \propto \pi_{\mathrm{ref}}(y\mid x) \exp(\langle \hat{\alpha}, \hat{r}\rangle/\beta).
    \end{gather*}
    Thus by the Appendix A.2 in \citep{cen2022fast}, then for any $y$ and $x$, we have 
    \begin{align*}
        |\log \pi(y\mid x) - \log \pi_{\mathrm{ref}}(y\mid x)|\le 2B/\beta.
    \end{align*}
    Then $$\sup_{x,y} \frac{\pi(y\mid x)}{\pi_{\mathrm{ref}}(y\mid x)} \le \exp(2B/\beta),\ \ \sup_{x,y}\frac{\pi_{\mathrm{ref}}(y\mid x)}{\pi(y\mid x)} \le \exp(2B/\beta).$$
    Now from the two inequalities following
    \begin{gather*}
              \sup_{x,y} \frac{\pi(y\mid x)}{\pi_{\mathrm{ref}}(y\mid x)} \le \exp(2B/\beta),\\
              \sup_{x,y} \frac{\pi_{\mathrm{ref}}(y\mid x)}{\pi'(y\mid x)} \le \exp(2B/\beta).
    \end{gather*}
    we can multiply them and get 
    \begin{align*}
        \sup_{x,y}\frac{\pi(y\mid x)}{\pi'(y\mid x)}\le \exp(4B/\beta).
    \end{align*}
\end{proof}

\begin{lemma}[Linear Structure]\label{lemma:linearstructure}
    Suppose that we have reward sequence $\{r^t(x)\}_{t \in [T]}$ with $r^t(x) = \langle \theta^t, \phi(x)\rangle$ with $\|\theta\| \le 1, \|\phi(x) \|\le B$, then for any policy $\{\pi^t\}_{t \in [T]}$ for any $\mu>0$, we can have 
    \begin{align*}
        \sum_{t=1}^T \EE_{x \sim \pi^t}[r^t(x)] \le \mu \cdot \sum_{t=1}^T\sum_{j=1}^{t-1} \EE_{x\sim \pi^j}[(r^t(x)]^2 + \widetilde{\cO}(Bd) +  \frac{d_{\mathrm{cover}}(1/T)}{4\mu}.
    \end{align*}
\end{lemma}
\begin{proof}
First, denote $X^t = \EE_{x\sim pi^t}[\phi(x)]$, then
    \begin{align*}
        \sum_{t=1}^T \EE_{x\sim \pi^t}[r^t(x)]&= \sum_{t=1}^T \EE_{x\sim \pi^t}[\langle \theta^t, \phi(x) \rangle]\\
        &=\sum_{t=1}^T \langle \theta^t, X^t\rangle.
    \end{align*}
   Now define $\Sigma_t = \varepsilon I + \sum_{i=1}^{t-1}X^i (X^i)^\top$, then we can decompose the term above as 
    \begin{align*}
        \sum_{t=1}^T \langle \theta^t, X^t\rangle &= \underbrace{\sum_{t=1}^T \langle \theta^t, X^t\rangle \II\{\|X^t\|_{\Sigma_t^{-1}}\le 1\}}_{\textrm{(A)}} + \underbrace{\sum_{t=1}^T \langle \theta^t, X^t\rangle \II\{\|X^t\|_{\Sigma_t^{-1}}> 1\}}_{\textrm{(B)}}.
    \end{align*}
The term (A) can be bounded as 
\begin{align*}
    \textrm{(A)}&= \sum_{t=1}^T \|\theta^t\|_{\Sigma_t} \|X^t\|_{\Sigma_{t}^{-1}} \II\{\|X^t\|_{\Sigma_t^{-1}}\le  1\}\\
    &\le\sum_{t=1}^T\|\theta^t\|_{\Sigma_t} \min\{1, \|X^t\|_{\Sigma_t^{-1}}^2\}^{1/2}\\
    &\le \sum_{t=1}^T \left[\varepsilon\|\theta^t\|^2 + \sum_{i=1}^{t-1}\langle \theta^t, X^i\rangle ^2\right]^{1/2}\min\{1, \|X^t\|_{\Sigma_t^{-1}}^2\}^{1/2}\\
    &\le \sqrt{\left[\sum_{t=1}^T\left(\varepsilon\|\theta^t\|^2 + \sum_{i=1}^{t-1}\langle \theta^t, X^i\rangle ^2\right)\right]\cdot \left[\sum_{t=1}^T  \min\{1,\|X^t\|_{\Sigma_t^{-1}}^2\}\right]},
\end{align*}
where the last inequality uses the Cauchy's inequality.

    Now we recall the elliptical potential lemma in \citep{abbasi2011improved}, we can get 
\begin{align}
    \sum_{t=1}^T \min\{1,\|X^t\|^2_{\Sigma_t^{-1}}\}\le d(\varepsilon) = \widetilde{\cO}(d\log(1/\varepsilon)).\label{ineq:eplemma}
\end{align}
Thus substitute it into the the inequality for (A), we can get 
\begin{align*}
    \textrm{(A)} \le \sqrt{d(\varepsilon)\cdot \left[\sum_{t=1}^T\left(\varepsilon\|\theta^t\|^2 + \sum_{i=1}^{t-1}\langle \theta^t, X^i\rangle ^2\right)\right]}.
\end{align*}
Now by the inequality that $\sqrt{a+b} \le \sqrt{a} + \sqrt{b}$,  we can get 
\begin{align*}
    \textrm{(A)} &\le \sqrt{d(\varepsilon)\cdot \left[\sum_{t=1}^T\left(\varepsilon\|\theta^t\|^2 + \sum_{i=1}^{t-1}\langle \theta^t, X^i\rangle ^2\right)\right]}\\
    &\le \sqrt{d(\varepsilon)\varepsilon T} + \sqrt{d(\varepsilon)\cdot \sum_{t=1}^T \sum_{t=1}^{t-1}\langle \theta^t, X^i\rangle^2}\\
    & \le \sqrt{d(\varepsilon)\varepsilon T} + \frac{d(\varepsilon)}{4\mu} + \mu \cdot \sum_{t=1}^T \sum_{i=1}^{t-1}\langle \theta^t, X^i\rangle^2\\
    & = \sqrt{d(\varepsilon)\varepsilon T} + \frac{d(\varepsilon)}{4\mu} + \mu \cdot \sum_{t=1}^T \sum_{j=1}^{t-1} (\EE_{\pi^i}[r^t(x)])^2.
    \end{align*}

    Now if we choose $\varepsilon=1/T$, then $d(\varepsilon) = \widetilde{\cO}(d)$, and the upper bound of $\textrm{(A)}$ becomes
    \begin{align*}
        \textrm{(A)} \le \sqrt{d_{\mathrm{cover}}(1/T)}+\frac{d_{\mathrm{cover}}(1/T)}{4\mu} + \mu \cdot \sum_{t=1}^T \sum_{j=1}^{t-1} (\EE_{\pi^i}[r^t(x)])^2.
    \end{align*}
    Now we derive the upper bound of (B).
    \begin{align*}
        \textrm{(B)}& = \sum_{t=1}^T\langle \theta^t, X^t\rangle \II\{\|X^t\|_{\Sigma_t^{-1}}>1\}\\
        &\le B\cdot \sum_{t=1}^T\II\{\|X^t\|_{\Sigma_t^{-1}}>1\}\\
        &\le B \sum_{t=1}^T \min\{1,\|X^t\|_{\Sigma_t^{-1}}^2\}\\
        &\le B d_{\mathrm{cover}}(1/T) = \widetilde{\cO}(Bd).
    \end{align*}
    So by adding (A) and (B), we can finally get 
    \begin{align*}
        \sum_{t=1}^T \langle \theta^t, X^t \rangle &\le Bd_{\mathrm{cover}}(1/T) + \sqrt{d_{\mathrm{cover}}(1/T)} +  \frac{d_{\mathrm{cover}}(1/T)}{4\mu} + \mu \cdot \sum_{t=1}^T \sum_{j=1}^{t-1} (\EE_{\pi^i}[r^t(x)])^2\\
        &\le \widetilde{\cO}(Bd) +  \frac{d_{\mathrm{cover}}(1/T)}{4\mu} + \mu \cdot \sum_{t=1}^T \sum_{j=1}^{t-1} (\EE_{\pi^i}[r^t(x)])^2.
    \end{align*}

\end{proof}

\section{Some Derivations in Section \ref{sec:moalg} and Section \ref{sec:pref aggregation}}\label{app:derivation}
\subsection{Derivation of Reward-free Modification}
Now we derive the equation 
\begin{align*}
    J(r_1^{\theta_1}, r_2^{\theta_2}, \cdots, r_m^{\theta_m}, \alpha, \pi^\theta)-\sum_{i=1}^m \eta L_i(\theta_i)=C-\beta \EE_{x\sim \rho, y\sim\pi_{\mathrm{base}}}\left[\log \frac{\pi^\theta(y\mid x)} { \pi_{\mathrm{ref}}(y\mid x)}\right]-\eta \sum_{i=1}^m L_i(\theta_i).
\end{align*}
In fact, since 
\begin{align*}J(r_1^{\theta_1}, r_2^{\theta_2},\cdots, r_m^{\theta_m},\alpha,\pi) &= \EE_{y\sim\pi^\theta(\cdot \mid x)}\left[\sum_{i=1}^m \alpha_i r_i^{\theta_i}(x,y) - \beta\cdot \sum_{i=1}^m \alpha_i\cdot (\log \pi^\theta(y\mid x)-\log \pi_{\mathrm{ref}}(y\mid x))\right]\\
&=\EE_{y\sim \pi^\theta(\cdot \mid x)}\left[\sum_{i=1}^m \alpha_ir(x,y) - \beta\cdot \sum_{i=1}^m \alpha_i\cdot (\log \pi^\theta(y\mid x)-\log \pi_{\mathrm{ref}}(y\mid x))\right]\\
&=\EE_{y\sim \pi^\theta(\cdot \mid x)}\left[\log Z(r,x)\right],\end{align*}
where $Z(r,x) = \sum_{y \in \cY}\pi_{\mathrm{ref}}(y\mid x)\exp(r(x,y)/\beta)$ is a normalization factor independent with $y$ \citep{rafailov2024direct}. Now, since $Z(r,x)$ is independent with $y$, we can get 
\begin{align*}
    J(r_1^{\theta_1}, r_2^{\theta_2},\cdots, r_m^{\theta_m},\alpha,\pi) &=\EE_{y\sim \pi^\theta(\cdot \mid x)}\left[\log Z(r,x)\right]\\
    &=\EE_{y\sim \pi_{\mathrm{base}}(\cdot \mid x)}\left[\log Z(r,x)\right]\\
    &=\EE_{y\sim \pi_{\mathrm{base}}(\cdot \mid x)}\left[r(x,y) - \beta (\log \pi^\theta(y\mid x)-\log \pi_{\mathrm{ref}}(y\mid x))\right]\\
    &=C-\beta \EE_{y\sim \pi_{\mathrm{base}}(\cdot \mid x)}\left[\log \frac{\pi^\theta(y\mid x)}{\pi_{\mathrm{ref}}(y\mid x)}\right].
\end{align*}
We complete the derivation. 
\subsection{Update Rule of Gradient Descent}
In this section, we show that the computational cost of Eq.~\eqref{eq:rfupdate} can be easily computed once the expectation of the score function can be derived. 

In fact, 
\begin{align*}
    &\nabla_{\theta_1}\left(-\beta \EE_{x\sim \rho, y\sim\pi_{\mathrm{base}}}[\log \pi^\theta(y\mid x)]\right)-\eta \nabla_{\theta_1}\sum_{i=1}^m \ell(\cD_i, \theta_i)\\
    &=-\beta \underbrace{\EE_{x\sim \rho, y\sim\pi_{\mathrm{base}}}[\nabla_{\theta_1}\log \pi^\theta(y\mid x)]}_{\text{(a)}}-\underbrace{\eta \nabla_{\theta_1}\ell(\cD_1, \theta_1)}_{\text{(b)}}.
\end{align*}
Term (b) in the last line is the gradient of log-likelihood loss that appears in classical reward-free algorithm like DPO. For term (a), note that if $\|d\|_1 = 1$, we have \begin{align*}\pi^\theta  \propto \pi_{\mathrm{ref}}(y \mid x) \cdot \exp\left(\sum_{i=1}^m \beta d_i r_i^{\theta_i}(x,\cdot)\right) =  \prod_{i=1}^m (\pi^{\theta_i}(y\mid x))^{d_i}.\end{align*} Hence, denote $s(\theta,\pi^*) = \EE_{\pi^*}[\nabla_\theta \log\pi^\theta(y\mid x)]$ is the expectation of the score function, we can then derive that 
\begin{align*}
    \text{(a)}= \beta d_1 \left(s(\theta_1, \pi_{\mathrm{base}}) - s(\theta_1, \pi^\theta)\right).
\end{align*}
Hence, the update rule can be efficiently computed as long as the score function is available, which commonly appears in previous RL algorithms such as REINFORCE.

Thus, if the learning rate is $\xi>0$, the gradient descent update rule of $\theta_1$ is 
\begin{align*}\theta_1^t &= \theta_1^{t-1}-\xi\left(\beta d_1(s(\theta_1,\pi_{\mathrm{base}}) - s(\theta_1, \pi^\theta))  - \eta \nabla_{\theta_1^{t-1}}L_1^t(\theta_1^{t-1})\right).\end{align*}
Also, for the reward-free version, we can change the term $L_1^t(\theta_1^{t-1})$ to $$\sum_{(x,y_w,y_l) \in \cD_1} \log \sigma \left(\beta\cdot \left(\log\frac{\pi^{\theta_1}(y_w\mid x)}{\pi_{\mathrm{ref}}(y_w\mid x)}-\log\frac{\pi^{\theta_1}(y_l\mid x)}{\pi_{\mathrm{ref}}(y_l\mid x)}\right)\right).$$

\subsection{Derivation of the reward-free equation of expected reward vector}\label{app:expected reward vector derivation}
We now prove that 
\begin{align*}
    (V^t_i) = \EE_{\pi^t}[r_i^{\theta_i^t}(x,y) - \beta\DD_{\mathrm{KL}}(\pi^t\| \pi_{\mathrm{ref}})] = C-\beta \EE_{y\sim \pi_{\mathrm{base}}}\left[\log \frac{\pi^{\theta_i^t}(y\mid x)}{\pi_{\mathrm{ref}}(y\mid x)}\right] - \beta \EE_{y\sim \pi^t}\left[\log \frac{\pi^{\theta_i^t}(y\mid x)}{\pi^t(y\mid x)}\right].
\end{align*}
\begin{proof}
We note that 
\begin{align*}
    \EE_{\pi^t}[r_i^{\theta_i^t}(x,y) - \beta\DD_{\mathrm{KL}}(\pi^t\| \pi_{\mathrm{ref}})] &= \EE_{\pi^t}\left[r_i^{\theta_i^t}(x,y) - \beta \left(\log\frac{\pi^t(y\mid x)}{\pi_{\mathrm{ref}}(y\mid x)}\right)\right]\\
    &=\EE_{\pi^t}\left[Z(r_i^{\theta_i^t},x)  + \beta\left(\log\frac{\pi^{\theta_i^t}(y\mid x)}{\pi_{\mathrm{ref}}(y\mid x)}\right) - \beta\left(\log\frac{\pi^t(y\mid x)}{\pi_{\mathrm{ref}}(y\mid x)}\right)\right]\\
    &=\EE_{\pi^t}[Z(r_i^{\theta_i^t},x)] +\beta \EE_{\pi^t}\left[\left(\log\frac{\pi^{\theta_i^t}(y\mid x)}{\pi^t(y\mid x)}\right)\right].
\end{align*}
Now note that $Z(r_i^{\theta_i^t},x)$ is independent on $y$, hence 
\begin{align*}
    \EE_{\pi^t}[Z(r_i^{\theta_i^t},x)] &=\EE_{\pi_{\mathrm{base}}}[Z(r_i,x)]\\
    &=\EE_{\pi_{\mathrm{base}}}\left[r_i^{\theta_i^t}(x,y) - \beta (\log \pi^{\theta_i^t}(y\mid x) - \log \pi_{\mathrm{ref}}(y\mid x))\right]\\
    &=C - \beta \EE_{y\sim \pi_{\mathrm{base}}}\left[\log \frac{\pi^{\theta_i^t}(y\mid x)}{\pi_{\mathrm{ref}}(y\mid x)}\right]. 
\end{align*}
\end{proof}


\begin{thebibliography}{}

\bibitem[Abbasi-Yadkori et~al., 2011]{abbasi2011improved}
Abbasi-Yadkori, Y., P{\'a}l, D., and Szepesv{\'a}ri, C. (2011).
\newblock Improved algorithms for linear stochastic bandits.
\newblock {\em Advances in neural information processing systems}, 24.

\bibitem[Anshelevich et~al., 2021]{anshelevich2021distortion}
Anshelevich, E., Filos-Ratsikas, A., Shah, N., and Voudouris, A.~A. (2021).
\newblock Distortion in social choice problems: The first 15 years and beyond.
\newblock {\em arXiv preprint arXiv:2103.00911}.

\bibitem[Azar et~al., 2024]{azar2024general}
Azar, M.~G., Guo, Z.~D., Piot, B., Munos, R., Rowland, M., Valko, M., and Calandriello, D. (2024).
\newblock A general theoretical paradigm to understand learning from human preferences.
\newblock In {\em International Conference on Artificial Intelligence and Statistics}, pages 4447--4455. PMLR.

\bibitem[Bai et~al., 2022]{bai2022training}
Bai, Y., Jones, A., Ndousse, K., Askell, A., Chen, A., DasSarma, N., Drain, D., Fort, S., Ganguli, D., Henighan, T., et~al. (2022).
\newblock Training a helpful and harmless assistant with reinforcement learning from human feedback.
\newblock {\em arXiv preprint arXiv:2204.05862}.

\bibitem[Cen et~al., 2022]{cen2022fast}
Cen, S., Cheng, C., Chen, Y., Wei, Y., and Chi, Y. (2022).
\newblock Fast global convergence of natural policy gradient methods with entropy regularization.
\newblock {\em Operations Research}, 70(4):2563--2578.

\bibitem[Cen et~al., 2024]{cen2024value}
Cen, S., Mei, J., Goshvadi, K., Dai, H., Yang, T., Yang, S., Schuurmans, D., Chi, Y., and Dai, B. (2024).
\newblock Value-incentivized preference optimization: A unified approach to online and offline rlhf.
\newblock {\em arXiv preprint arXiv:2405.19320}.

\bibitem[Chakraborty et~al., 2024]{chakrabortymaxmin}
Chakraborty, S., Qiu, J., Yuan, H., Koppel, A., Manocha, D., Huang, F., Bedi, A., and Wang, M. (2024).
\newblock Maxmin-rlhf: Alignment with diverse human preferences.
\newblock In {\em Forty-first International Conference on Machine Learning}.

\bibitem[Chen et~al., 2024]{chen2024pal}
Chen, D., Chen, Y., Rege, A., and Vinayak, R.~K. (2024).
\newblock Pal: Pluralistic alignment framework for learning from heterogeneous preferences.
\newblock {\em arXiv preprint arXiv:2406.08469}.

\bibitem[Conitzer et~al., 2024]{conitzer2024social}
Conitzer, V., Freedman, R., Heitzig, J., Holliday, W.~H., Jacobs, B.~M., Lambert, N., Moss{\'e}, M., Pacuit, E., Russell, S., Schoelkopf, H., et~al. (2024).
\newblock Social choice should guide ai alignment in dealing with diverse human feedback.
\newblock {\em arXiv preprint arXiv:2404.10271}.

\bibitem[Cousins, 2021]{cousins2021axiomatic}
Cousins, C. (2021).
\newblock An axiomatic theory of provably-fair welfare-centric machine learning.
\newblock {\em Advances in Neural Information Processing Systems}, 34:16610--16621.

\bibitem[Fishburn, 1973]{fishburn1973binary}
Fishburn, P.~C. (1973).
\newblock Binary choice probabilities: on the varieties of stochastic transitivity.
\newblock {\em Journal of Mathematical psychology}, 10(4):327--352.

\bibitem[Ge et~al., 2024]{ge2024axioms}
Ge, L., Halpern, D., Micha, E., Procaccia, A.~D., Shapira, I., Vorobeychik, Y., and Wu, J. (2024).
\newblock Axioms for ai alignment from human feedback.
\newblock {\em arXiv preprint arXiv:2405.14758}.

\bibitem[Guo et~al., 2024]{guo2024direct}
Guo, S., Zhang, B., Liu, T., Liu, T., Khalman, M., Llinares, F., Rame, A., Mesnard, T., Zhao, Y., Piot, B., et~al. (2024).
\newblock Direct language model alignment from online ai feedback.
\newblock {\em arXiv preprint arXiv:2402.04792}.

\bibitem[Jin et~al., 2021]{jin2021bellman}
Jin, C., Liu, Q., and Miryoosefi, S. (2021).
\newblock Bellman eluder dimension: New rich classes of rl problems, and sample-efficient algorithms.
\newblock {\em Advances in neural information processing systems}, 34:13406--13418.

\bibitem[Liu et~al., 2020]{liu2020provably}
Liu, Y., Swaminathan, A., Agarwal, A., and Brunskill, E. (2020).
\newblock Provably good batch off-policy reinforcement learning without great exploration.
\newblock {\em Advances in neural information processing systems}, 33:1264--1274.

\bibitem[Liu et~al., 2024]{liu2024provably}
Liu, Z., Lu, M., Zhang, S., Liu, B., Guo, H., Yang, Y., Blanchet, J., and Wang, Z. (2024).
\newblock Provably mitigating overoptimization in rlhf: Your sft loss is implicitly an adversarial regularizer.
\newblock {\em arXiv preprint arXiv:2405.16436}.

\bibitem[Ouyang et~al., 2022]{ouyang2022training}
Ouyang, L., Wu, J., Jiang, X., Almeida, D., Wainwright, C., Mishkin, P., Zhang, C., Agarwal, S., Slama, K., Ray, A., et~al. (2022).
\newblock Training language models to follow instructions with human feedback.
\newblock {\em Advances in neural information processing systems}, 35:27730--27744.

\bibitem[Pardeshi et~al., 2024]{pardeshi2024learning}
Pardeshi, K.~S., Shapira, I., Procaccia, A.~D., and Singh, A. (2024).
\newblock Learning social welfare functions.
\newblock {\em arXiv preprint arXiv:2405.17700}.

\bibitem[Park et~al., 2024]{park2024rlhf}
Park, C., Liu, M., Kong, D., Zhang, K., and Ozdaglar, A.~E. (2024).
\newblock Rlhf from heterogeneous feedback via personalization and preference aggregation.
\newblock In {\em ICML 2024 Workshop: Aligning Reinforcement Learning Experimentalists and Theorists}.

\bibitem[Rafailov et~al., 2024]{rafailov2024direct}
Rafailov, R., Sharma, A., Mitchell, E., Manning, C.~D., Ermon, S., and Finn, C. (2024).
\newblock Direct preference optimization: Your language model is secretly a reward model.
\newblock {\em Advances in Neural Information Processing Systems}, 36.

\bibitem[Rame et~al., 2024]{rame2024rewarded}
Rame, A., Couairon, G., Dancette, C., Gaya, J.-B., Shukor, M., Soulier, L., and Cord, M. (2024).
\newblock Rewarded soups: towards pareto-optimal alignment by interpolating weights fine-tuned on diverse rewards.
\newblock {\em Advances in Neural Information Processing Systems}, 36.

\bibitem[Ramesh et~al., 2024]{ramesh2024group}
Ramesh, S.~S., Hu, Y., Chaimalas, I., Mehta, V., Sessa, P.~G., Ammar, H.~B., and Bogunovic, I. (2024).
\newblock Group robust preference optimization in reward-free rlhf.
\newblock {\em arXiv preprint arXiv:2405.20304}.

\bibitem[Shi et~al., 2024]{shi2024decoding}
Shi, R., Chen, Y., Hu, Y., Liu, A., Smith, N., Hajishirzi, H., and Du, S. (2024).
\newblock Decoding-time language model alignment with multiple objectives.
\newblock {\em arXiv preprint arXiv:2406.18853}.

\bibitem[Sorensen et~al., 2024]{sorensenposition}
Sorensen, T., Moore, J., Fisher, J., Gordon, M.~L., Mireshghallah, N., Rytting, C.~M., Ye, A., Jiang, L., Lu, X., Dziri, N., et~al. (2024).
\newblock Position: A roadmap to pluralistic alignment.
\newblock In {\em Forty-first International Conference on Machine Learning}.

\bibitem[Touvron et~al., 2023]{touvron2023llama}
Touvron, H., Martin, L., Stone, K., Albert, P., Almahairi, A., Babaei, Y., Bashlykov, N., Batra, S., Bhargava, P., Bhosale, S., et~al. (2023).
\newblock Llama 2: Open foundation and fine-tuned chat models.
\newblock {\em arXiv preprint arXiv:2307.09288}.

\bibitem[Wang et~al., 2023]{wang2023beyond}
Wang, C., Jiang, Y., Yang, C., Liu, H., and Chen, Y. (2023).
\newblock Beyond reverse kl: Generalizing direct preference optimization with diverse divergence constraints.
\newblock {\em arXiv preprint arXiv:2309.16240}.

\bibitem[Williams, 1992]{williams1992simple}
Williams, R.~J. (1992).
\newblock Simple statistical gradient-following algorithms for connectionist reinforcement learning.
\newblock {\em Machine learning}, 8:229--256.

\bibitem[Wu et~al., 2023]{wu2023fine}
Wu, Z., Hu, Y., Shi, W., Dziri, N., Suhr, A., Ammanabrolu, P., Smith, N.~A., Ostendorf, M., and Hajishirzi, H. (2023).
\newblock Fine-grained human feedback gives better rewards for language model training.
\newblock {\em Advances in Neural Information Processing Systems}, 36:59008--59033.

\bibitem[Yang et~al., 2024]{yang2024rewards}
Yang, R., Pan, X., Luo, F., Qiu, S., Zhong, H., Yu, D., and Chen, J. (2024).
\newblock Rewards-in-context: Multi-objective alignment of foundation models with dynamic preference adjustment.
\newblock {\em arXiv preprint arXiv:2402.10207}.

\bibitem[Yu et~al., 2021]{yu2021provably}
Yu, T., Tian, Y., Zhang, J., and Sra, S. (2021).
\newblock Provably efficient algorithms for multi-objective competitive rl.
\newblock In {\em International Conference on Machine Learning}, pages 12167--12176. PMLR.

\bibitem[Zeng et~al., 2024]{zeng2024token}
Zeng, Y., Liu, G., Ma, W., Yang, N., Zhang, H., and Wang, J. (2024).
\newblock Token-level direct preference optimization.
\newblock {\em arXiv preprint arXiv:2404.11999}.

\bibitem[Zhong et~al., 2024]{zhong2024provable}
Zhong, H., Deng, Z., Su, W.~J., Wu, Z.~S., and Zhang, L. (2024).
\newblock Provable multi-party reinforcement learning with diverse human feedback.
\newblock {\em arXiv preprint arXiv:2403.05006}.

\bibitem[Zhou et~al., 2023]{zhou2023beyond}
Zhou, Z., Liu, J., Yang, C., Shao, J., Liu, Y., Yue, X., Ouyang, W., and Qiao, Y. (2023).
\newblock Beyond one-preference-for-all: Multi-objective direct preference optimization.
\newblock {\em arXiv preprint arXiv:2310.03708}.

\bibitem[Zhu et~al., 2023]{zhu2023principled}
Zhu, B., Jordan, M., and Jiao, J. (2023).
\newblock Principled reinforcement learning with human feedback from pairwise or k-wise comparisons.
\newblock In {\em International Conference on Machine Learning}, pages 43037--43067. PMLR.

\end{thebibliography}
\end{document}